\documentclass[twoside]{article}

\usepackage[accepted]{aistats2025}

\usepackage[round]{natbib}

\usepackage{xcolor}         
\usepackage{algorithm}
\usepackage{algorithmic}
\usepackage{amsmath}
\usepackage{amssymb}
\usepackage{mathtools}
\usepackage{amsthm}
\usepackage{thmtools} 
\usepackage{thm-restate}
\usepackage{graphicx}
\usepackage{booktabs}
\graphicspath{{figures/}}
\usepackage{preamble}
\usepackage{tagging}
\usepackage{caption}
\usepackage[T1]{fontenc}

\newcommand{\coderepository}{\url{https://github.com/DavidRBurt/Consistent-Spatial-Validation}}

\usepackage[toc,page,header]{appendix}
\usepackage{minitoc}

\hypersetup{
    colorlinks,
    linkcolor={green!50!black},
    citecolor={blue!50!black},
    urlcolor={blue!80!black}
}

\begin{document}

\doparttoc 
\faketableofcontents 

\twocolumn[

  \aistatstitle{Consistent Validation for Predictive Methods in Spatial Settings}

  \aistatsauthor{ David R.~Burt \And Yunyi Shen \And  Tamara Broderick}

  \aistatsaddress{ Massachusetts Institute of Technology Laboratory for Information and Decision Systems} ]

\begin{abstract}
  Spatial prediction tasks are key to weather forecasting, studying air pollution impacts, and other scientific endeavors. Determining how much to trust predictions made by statistical or physical methods is essential for the credibility of scientific conclusions.
  Unfortunately, classical approaches for validation fail to handle mismatch between locations available for validation and (test) locations where we want to make predictions. This mismatch is often not an instance of covariate shift (as commonly formalized) because the validation and test locations are fixed (e.g., on a grid or at select points) rather than i.i.d.\ from two distributions. In the present work, we formalize a check on validation methods: that they become arbitrarily accurate as validation data becomes arbitrarily dense. We show that classical and covariate-shift methods can fail this check. We propose a method that builds from existing ideas in the covariate-shift literature, but adapts them to the validation data at hand. We prove that our proposal passes our check. And we demonstrate its advantages empirically on simulated and real data.
\end{abstract}

\section{INTRODUCTION}
\label{sec:introduction}
Researchers are often interested in making predictions in a spatial setting. For instance, scientists predict sea surface temperature (SST) for weather forecasting and climate research \citep{minnett2010sst}, predict air pollution at population centers to better understand the effect of pollution on health outcomes such as kidney disease \citep{remigio2022combined}, or predict the prevalence of an invasive species for ecological management \citep{barbet2018can}.
Characterizing the reliability of these predictions is key to understanding their suitability for downstream applications; e.g., \citet{minnett2010sst} describes acceptable SST error tolerances for weather forecasting. Estimates of prediction accuracy can also be used to choose between several predictive methods, as in \citet{shabani2016comparison}.

In the spatial setting, predictive methods need not always arise from a statistical or machine learning approach built using training data. The predictive method is often a complex physical model provided by a third party \citep{remigio2022combined,minnett2010sst,gupta2018validation}. Or it could combine physics and data-driven models \citep{banzon2016longterm,werner2019data, ozkaynak2013air}.
In any of these cases, it is common to estimate the performance of a predictive method by using a set of \emph{validation data}. More precisely, we are ultimately interested in predicting a response at what we call \emph{test sites}; in the SST example above \citep{minnett2010sst}, the test sites are points on a grid (often called a map), or in the air pollution example \citep{remigio2022combined}, the test sites are 28 counties in the US Northeast. We can make predictions at the test sites, but we do not have access to direct observations of the responses there. We have observed responses in the validation data, which we assume were not used in forming the predictive method being evaluated; in the SST example, scientists have SST observations taken by boats and buoys as validation data. So our aim is to estimate the average loss (i.e.\ risk) at the test sites using the validation data. We show that many popular or natural approaches fail at this task.

One widely used approach, called the \emph{\holdout},\footnote{The name originally referred to ``holding out'' data for validation, with the remainder of available data going toward training. While we maintain the naming convention, we emphasize that in our setup there need not be any training data. \citet{minnett2010sst,gupta2018validation, duan2019validation}, among many others, use this approach.} estimates the test risk by taking the empirical average of the validation loss. When the validation and test data are independent and identically distributed (i.i.d.)\ from the same distribution, the \holdout{} has a rigorous justification \citep{devroye1976nonparametric, langford2005tutorial}. But in spatial problems, the validation and test sites need not be similarly dispersed, data may be spatially correlated, and the test sites are often fixed rather than random; recall the grid or point prediction examples above. Indeed, \citet{roberts2017cross} observed problems with the \holdout{} in practice. In the special case where the predictive method is data-driven, some authors \citep[e.g.][]{telford2005secret} have suggested choosing \holdout{} validation sites far from the training data sites. But this proposal still suffers from the problems just mentioned, and in fact simulation studies suggest it does not accurately measure prediction accuracy \citep{wadoux2021spatial,debruin2022dealing}.

Another natural idea is to use covariate-shift approaches to handle potential mismatch between validation and test sites \citep{sarafian2020domain, debruin2022dealing}. However, the covariate-shift literature generally assumes validation sites are drawn i.i.d.\ from one distribution, test sites are i.i.d.\ from another, and the density ratio between these two distributions exists and is bounded. For both the grid and point examples, these last two assumptions are inappropriate.

In what follows, we start by laying out a precise formulation of the prediction validation task in the spatial setting (\cref{sec:formal-problem-statement}). We formalize a desirable property for test-risk estimators: that, if arbitrarily dense validation data accrues in a region including the test points, the test-risk estimate should become arbitrarily accurate (\cref{sec:methods-infill-asymptotics}). We next prove that common estimators fail to satisfy this spatial consistency property (\cref{sec:inconsistent}); namely, we consider each of (1) the \holdout{} estimator, (2) blocked spatial validation \citep{telford2005secret,roberts2017cross}, and (3) an estimator advocated in the covariate-shift literature \citep{loog2012nearest,portier2023scalable}. To construct a spatially consistent estimator, we build on the $k$-nearest neighbor estimator \citep{loog2012nearest}. \citet{loog2012nearest} and \citet{portier2023scalable} advocated fixing $k=1$ for covariate-shift problems. We instead derive an upper bound on the error of the general-$k$ estimator for estimating test risk (\cref{sec:k-selection-procedure}); crucially, our bound is \emph{conditional on the test and validation sites}. We prove that choosing $k$ adaptively by optimizing our upper bound yields a spatially consistent estimator (\cref{sec:general-consistency}). Unlike covariate-shift results \citep[e.g.][]{portier2023scalable}, our results are directly applicable to problems where the test sites are most reasonably thought of as fixed. We illustrate the accuracy and practicality of our proposed method in simulated and real data analyses (\cref{sec:experiments}), with tasks in both grid and point prediction.\footnote{We provide an implementation of our method and the experiments in this paper at: \coderepository.} We discuss further related work in \cref{app:related-work}.

\section{SPATIAL TEST RISK}
\label{sec:formal-problem-statement}
We now formalize
risk estimation at test points in a spatial setting.
We assume each data point occurs at a spatial location $S \in \spatialdomain$, where the spatial domain $(\spatialdomain, \spatialmetric)$ is a metric space. Each data point has observed covariates $X \in \covariatedomain$ and a response $Y \in \responsedomain$. The covariates are a fixed spatial field, $\spatialfield: \spatialdomain \to \covariatedomain$; i.e., the covariates at a point are specified by evaluating $\spatialfield$ at the point's spatial location.

\subsection{Test Risk of a Spatial Predictive Method}
We assume we have access to a predictive method $h: \spatialdomain \to \responsedomain$. The predictive method might use the covariates. We allow for this dependence, but do not make it explicit in our notation. We suppose that practitioners would like to use $h$ to predict the response at a set of test sites where the response is unknown. We collect the $\ntest$ test data points, including true (but unobserved) responses, in $\smash{\Dtest = (\Stest_m, \Xtest_m, \Ytest_m)_{m=1}^{\ntest}}$.

To quantify quality of a predictive method, we need a loss. We assume the loss is bounded, as is common for responses naturally confined to a finite range; cf.\ temperature, pressure, or other physical quantities.\footnote{Loss is also bounded for classification error or certain robust regression cases such as Tukey's biweight loss.}
\begin{assumption}[$\lossbdd$-bounded Loss]\label{assumption:BoundedLoss}
    The loss is a non-negative, bounded function, $\ell\!: \!\responsedomain \times\! \responsedomain \to\! [0, \lossbdd]$.
\end{assumption}
Due to practicalities such as measurement error, the response at a test point is usefully modeled as random. A standard summary of loss over this randomness is the expected loss (a.k.a.\ risk) at the test data. To define this expectation, we need to make assumptions about the data-generating process. It is typical in the non-spatial setting to assume responses are i.i.d.\ conditional on covariates. In the spatial setting, the i.i.d.\ assumption is inappropriate since it ignores spatial location. We instead assume that the response variable may be a function of the spatial location it is observed at, the covariates at that location, and i.i.d noise:
\begin{assumption}[Data Generating Process: Test Data]\label{assumption:DGP}
    Let $j \!\!= \!\test$, $\spatialfield\!: \!\spatialdomain \!\to\! \covariatedomain$. For $1\!\leq \!m\! \leq \! M^j$, $X^j_m = \spatialfield(S^j_m)$ and $Y_m^j = f(S^j_m, X^j_m, \epsilon_m^j)$
    with $f \!:\! \spatialdomain \!\times\! \covariatedomain \to \responsedomain$ and $\epsilon_m^j \smash{\iidsim}P_{\epsilon}$ real-valued random variables.
\end{assumption}
\Cref{assumption:DGP} implies the response is i.i.d.~\emph{given} the spatial location. For example, \cref{assumption:DGP} covers the case where measurement errors on sensors are independent, but the locations of the sensors are not. A widely studied special case of \cref{assumption:DGP} considers additive, homoskedastic noise: namely, $\responsedomain \subset \RR$ and $Y_m^j = f(S^j_m, X^j_m)+\epsilon_m^j$.
\Cref{assumption:DGP} is more general; for example, it allows the noise to be scaled by a continuous, deterministic function of $\spatialdomain$: $Y_m^j = f(S^j_m, X^j_m)+g(S^j_m)\epsilon_m^j$.

Before defining risk, we first define  the average loss of the predictive method at a particular location in space, $S$: $\averageloss(S) := \EE[\ell(f(S, \spatialfield(S), \epsilon), h(S))|S]$.
Finally, we average over all spatial locations of interest, which we assume is a finite set.

\begin{definition}\label{def:test-risk}
Given test points $(\Stest_m)_{m=1}^{\ntest}$, let $\smash{\Qtest \coloneqq (1/\ntest) \sum\nolimits_{m=1}^{\ntest} \smash{\delta_{\Stest_m}}}$, with $\delta_{S}$ a Dirac measure at $S$. For predictive method $h$, let the \emph{test risk} of $h$ be
$
    \conditionalrisk_{\Qtest}(h) :=  (1/\ntest)\sum\nolimits_{m=1}^{\smash{\ntest}} \averageloss(\Stest_m){\vphantom{\scalebox{0.85}{$\frac{a^2}{b}$}}} .
$
\end{definition}
We write out the discrete distribution of test locations, $\Qtest$, to emphasize that test risk averages over the (fixed) test locations. We will find it useful for stating our upper bound on estimation error (\cref{thm:nn-fill-bdd}).

\subsection{Estimating Test Risk}\label{sec:estimating-test-risk}

To estimate test risk, we assume we have access to $\nval$ validation data points, collected in $\smash{\Dval = (\Sval_n, \Xval_n, \Yval_n)_{n=1}^{\nval}}$.
We assume practitioners did not make use of either the validation or test response data when constructing the predictive method.
As an example, the common \holdout{} estimator uses the empirical average of validation loss:
\begin{align}
    \smash{\sampleaveragerisk(h) \coloneqq \smash{\frac{1}{\nval}} \sum\nolimits_{n=1}^{\nval} \ell(\Yval_n, h(\Sval_n)).} \label{eqn:sample-avg-definition}
\end{align}

For validation data to provide information about test risk, we need to make an assumption about how validation data relates to test data. A particularly common (and often very reasonable) assumption is that validation and test data are i.i.d. However, as we argue above, the i.i.d.\ assumption is not appropriate in many spatial applications. To make progress, then, we need to make an alternative assumption relating validation and test data. Essentially we propose that, in spatial settings, smoothness in space is an alternative (and more intuitively plausible) regularity assumption.

More specifically, we first make a standard assumption that validation data follows the same data-generating process as test data, conditional on spatial locations. And second, we assume a form of smoothness in the average loss across the spatial locations. 
\begin{assumption}[Data Generating Process: Validation Data]\label{assumption:DGP_val}
    \Cref{assumption:DGP} remains true when we take $j = \val$, with the same $f, \spatialfield$, and $P_{\epsilon}$ as for $j=\test$.
\end{assumption}

\begin{assumption}[$\lipschitzbdd$-Lipschitz]\label{assumption:Lipschitz}
    For some $\lipschitzbdd \geq 0$, for all  $S,S' \in \spatialdomain, |\averageloss(S)\! - \!\averageloss(S')| \leq \lipschitzbdd \spatialmetric(S,S').$ 
\end{assumption}
\Cref{assumption:Lipschitz} often arises naturally. For example, take $\responsedomain \subset [0,1]$; $\spatialdomain = (\RR^d, \|\cdot\|_2)$; squared loss; and homoskedastic, additive noise. Suppose $f(S, \spatialfield(S))$ is $\lipschitzbdd_Y$-Lipschitz and $h$ is $\lipschitzbdd_h$-Lipschitz. Then \cref{assumption:Lipschitz} holds with $\lipschitzbdd = 2(\lipschitzbdd_Y + \lipschitzbdd_h)$; see \cref{prop:lipschitz-claim}.

\section{WE WANT CONSISTENT ESTIMATORS}\label{sec:methods-infill-asymptotics}
Given an estimator of test risk, we want to check if that estimator performs well. We next formalize one natural check on performance: namely, estimators should become arbitrarily accurate if given validation data that is arbitrarily dense in the spatial domain. This check is analogous to traditional consistency in the i.i.d.\ data setting.

To that end, the \emph{fill distance} is a measure of discrepancy between two sets, $\Psi_1$ and $\Psi_2$.\footnote{The fill distance is not a distance in the mathematical sense since it is asymmetric and can equal $0$ in cases when its two arguments are not exactly equal.
} 
It is the maximum distance from a point in $\Psi_2$ to the nearest point in $\Psi_1$.
\begin{definition}[{\citealt[§5.8]{cressie2015statistics}, \citealt[Definition 1.4]{Wendland_2004}}]
    Let $(\spatialdomain, \spatialmetric)$ be a metric space and $\Psi_1, \Psi_2 \subset \spatialdomain$. The fill distance of $\Psi_1$ in $\Psi_2$ is
    $
        \zeta(\Psi_1; \Psi_2) := \sup\nolimits_{S_2 \in \Psi_2} \inf\nolimits_{S_1 \in \Psi_1} \spatialmetric(S_1, S_2).
    $
\end{definition}

In the spatial statistics literature, \emph{infill asymptotics} describes cases where data are in a compact spatial domain so that the fill distance of the data to its domain tends to $0$. We say an estimator is \emph{consistent for the test risk under infill asymptotics} if---for any $\Qtest$, $\spatialfield$, $f$ and $h$ satisfying our assumptions above---the estimator converges in probability to $\conditionalrisk_{\Qtest}(h)$. 
\begin{definition}[Consistency of Test Risk Estimation Under Infill Asymptotics]
\label{def:consistency}
Fix a predictive method $h$ and a test measure $\Qtest$. Take \cref{assumption:DGP,assumption:BoundedLoss,assumption:DGP_val,assumption:Lipschitz}. Consider an infinite sequence of validation sets of increasing size: $(\Dval_{N})_{N=1}^{\infty}, \Dval_{N} = (\Sval_n, \Xval_n, \Yval_n)_{n=1}^{N}$ such that when $N' < N$,  the first $N'$ points of $\Dval_{N}$ are $\Dval_{N'}$.
Suppose $\smash{\lim_{N \to \infty} \zeta(\Sval_{1:N}, \spatialdomain) = 0}$.
Let $\smash{\hat{R}_N}$ be an estimator constructed from the validation data $\Dval_{N}$.
We say that the estimator $\smash{\hat{R}_N}$ is consistent for the test risk under infill asymptotics if for all $\epsilon > 0, 
   \lim_{N \to \infty} \mathrm{Pr}(|\smash{\hat{R}_N} - \conditionalrisk_{\Qtest}(h)|  \geq \epsilon) = 0.
$
\end{definition}

That is, as validation data fills the spatial domain, the estimator
should converge to the test risk---no matter the composition of test sites.
Our assumption that fill distance tends to zero is generally weaker than an assumption that the validation sites are drawn i.i.d.\ from a distribution with Lebesgue density supported on the spatial domain. \citet[Theorem 2.1]{Reznikov2015TheCR} showed an implication relationship between these assumptions in a much more general setting, and \citet[Lemma 12]{pmlr-v134-vacher21a} discuss the special case for the unit cube. Next, we present a finite-sample version of this implication, with an advantage relative to past work that we keep track of all constants.

\begin{restatable}[Independent and Identically Distributed Data Satisfies an Infill Assumption]{proposition}{iidimpliesinfill}\label{prop:iid-implies-infill}
    Suppose that $\spatialdomain = [0,1]^d$, $\Sval_n \iidsim P$ for $1 \leq n \leq \nval$, and $P$ has Lebesgue density lower bounded by $c >0$ over $\smash{[0,1]^d}$. Let $\spherevolume = \pi^{d/2}/\Gamma(d/2+1)$ be the volume of the $d$-dimensional Euclidean unit ball. For any $\delta \in (0,1)$ there exists an $n_0$ such that for all $\nval \geq n_0$ with probability at least $1-\delta$
    \begin{align}
        \smash{\zeta(\Sval_{1:\nval}, [0,1]^d) \leq \left(\tfrac{4^d}{c\nval \spherevolume}\left(\log \tfrac{6^d\nval}{\spherevolume\delta}\right)\right)^{1/d}.}
    \end{align}
\end{restatable}

We prove \cref{prop:iid-implies-infill} in \cref{app:infill-iid}. The right side of this bound is $\smash{O((\log \nval/\nval)^{1/d})}$, and so the fill distance converges to zero in probability under these assumptions.

\textbf{Consistency under infill asymptotics is a minimal desirable property.}
Like traditional consistency, we emphasize that \cref{def:consistency} is just a single check among many. For instance, often practitioners will be interested in extrapolation far from observed data, which is not modelled by infill asymptotics and will need to be considered separately. Our only supposition here is that we will generally prefer test-risk estimators that satisfy consistency under infill asymptotics to those that do not.

\section{CURRENT ESTIMATORS ARE INCONSISTENT}
\label{sec:inconsistent}
Even though consistency under infill asymptotics is a minimal desirable property, we next prove that principle existing test-risk estimators fail to satisfy it in realistic problems.


\textbf{Inconsistency of the \Holdout{}.} We state our result and then discuss the realism of the example.
\begin{restatable}[Inconsistency of \holdout]{proposition}{counterexampleholdout}\label{prop:holdout-inconsistent}
    There exists a set of test points and a data-generating process satisfying infill asymptotics such that $\smash{\sampleaveragerisk}$ is not a consistent estimator of the test risk. 
\end{restatable}

While the \holdout{} estimator of test risk is consistent for i.i.d.\ test and validation data \citep{devroye1976nonparametric, langford2005tutorial}, we can construct examples showing \cref{prop:holdout-inconsistent} by observing that $\smash{\sampleaveragerisk}$ has no dependence on the test task. So unless all test tasks have the same risk (which will be true only in unusually simplistic spatial settings), it cannot estimate them all consistently. The \holdout{} estimator will generally exhibit non-trivial bias since it averages loss across the validation sites when we really care about loss at the test sites. See \cref{app:holdout-inconsistent-lipschitz} for a formal proof and also an example where the \holdout{} converges to $\Delta$, the maximum possible error for a bounded loss. 

\textbf{Inconsistency of Blocked Spatial Validation.}
We first define blocked spatial validation and then establish its spatial inconsistency.
Imagine we have access to a set of data that may be used for training or validation. In blocked spatial validation, all of the data within a particular contiguous area (typically a block, or square) is reserved for validation, and the rest is used for training; see \cref{fig:spatial-blocking-realistic} (left) in \cref{app:spatial-blocking} for an illustration. Finally, the blocked spatial validation estimator of test risk is \cref{eqn:sample-avg-definition}; that is, this estimator is a special case of the \holdout{}. While we focus on validation in the present work, blocked spatial validation often represents one fold within a cross-validation procedure (e.g.~\citealp{mila2022nearest, Linnenbrink2023kNNDM,Rest2014Spatial,wang2023spatial}); see \cref{app:spatial-blocking} for a discussion. Prior work has shown via simulation that blocked spatial cross-validation can provide biased estimates of the test performance of a method (see \citealp[Figure 1]{wadoux2021spatial} and \citealp[Figure 9]{debruin2022dealing}). Since our counterexample in \cref{prop:holdout-inconsistent} is not a blocked spatial validation estimator, it remains to show whether blocked spatial validation is spatially inconsistent.

\begin{restatable}[Inconsistency of blocked spatial validation]{proposition}{counterexampleblocked}\label{prop:blocking-inconsistent}
    There exists a set of test points and a data-generating process satisfying infill asymptotics such that $\smash{\sampleaveragerisk}$ applied in blocked spatial validation is not a consistent estimator of the test risk. 
\end{restatable}

Spatial inconsistency arises for the same reason as in \cref{prop:holdout-inconsistent}: the validation locations may be systematically different from the test locations, and adding more validation locations need not resolve these differences. See \cref{app:spatial-blocking} for a proof.

\textbf{Nearest Neighbor Estimator.} Because of regularity in the error function (\cref{assumption:Lipschitz}), validation data are informative about the error at a nearby test site.
\citet{loog2012nearest} proposed risk estimators using $k$-nearest neighbors in the context of covariate shift. Both \citet{loog2012nearest} and \citet{portier2023scalable} advocated for the use of 1-nearest neighbor (1NN) in the covariate-shift setting, with the latter providing theoretical justifications under standard covariate-shift assumptions. However, we show the 1NN estimator exhibits inconsistency in our spatial setting.

We first review a general $k$-nearest neighbor estimator, which we revisit later. 
Define the $k$-nearest neighbor radius of a point $S \!\in\! \spatialdomain$ as
$
    \!\tau^k(S) \!:=\! \inf \{a \!\in\! \RR \!:\! \! |\Sval_{1:\nval} \!\cap \! \ball{S}{a}| \!\geq\! k\},
$
where $\ball{S}{a}$ is the ball of radius $a$ centered at $S$. The $k$-nearest neighbor set\footnote{We state our results for nearest neighbors with ties resolved by including all equidistant points. However, our analysis holds for arbitrary tie-breaking methods.} of a point $S \in \spatialdomain$ is
   $ A^k(S) \!:= \!\{1 \!\leq n\! \leq \!\nval\!: \! \Sval_n \in\! \ball{S}{\!\tau^k(S)}\}.$
As long as $1 \leq k \leq \nval$, $A^k(S)$ contains at least $k$ points: the $k$ nearest neighbors to $S$ in the validation set. It may be larger than $k$ if multiple points are equidistant from $S$.

\begin{definition}\label{def:nn-estimator}
    The $k$-nearest neighbor (kNN) test-risk estimator is 
        $\smash{\nearestneighborrisk}(h)  \!:=\!\sum\nolimits_{n=1}^{\nval}\! \nnweights_n \ell(\Yval_n, h(\Sval_n))$, with
        $\nnweights_n \!:= \frac{1}{\ntest}\sum\nolimits_{m=1}^{\ntest} \mathrm{1}\{\Sval_n  \!\in\! A^k(\Stest_m)\}/|A^k(\Stest_m)|$.
\end{definition}
\citet{loog2012nearest} proposed weighting the loss function in this way when training a model under covariate shift. 
\citet{portier2023scalable} analyzed a similar approach, in which validation points are sampled with probabilities corresponding to the weights in \cref{def:nn-estimator}, for mean estimation. \citet{portier2023scalable} made standard covariate shift assumptions of i.i.d.\ validation sites, i.i.d.\ test sites, and a bounded density ratio between the validation and test distributions. 

\textbf{Inconsistency of 1NN.} We again state our result and then develop intuition.

\begin{restatable}[Inconsistency of 1NN]{proposition}{counterexampleonenn}\label{prop:counterexample-1nn}
    There exist a set of test points and a data-generating process satisfying infill asymptotics such that $\onennestimator(h)$ is not a consistent estimator of the test risk. 
\end{restatable}

For intuition, recall that---unlike in the covariate-shift setting---test points in the spatial setting are commonly fixed rather than arising i.i.d.\ from a distribution. 
Consider the case where $\Qtest = \delta_S$ for some $S \in \spatialdomain$. Using $k=1$ leads us to estimate the error using a single validation point, which is inconsistent due to observation noise at that point. Where the problem with the \holdout{} estimator was bias, the problem with 1NN is variance. In \cref{app:inconsistency-1nn} we prove \cref{prop:counterexample-1nn} and show 1NN has large error when applied to classification point prediction tasks.

\textbf{Inconsistency of Nearest Neighbors When $k$ Is a Function of the Number of Validation Points.}
In fact, we can prove a more general result: that any nearest-neighbor test-risk estimator where the number of neighbors depends only on the number of validation points is inconsistent under infill asymptotics, regardless of type of dependence. 

\begin{restatable}[Inconsistency of kNN depending on number of validation points]{proposition}{counterexamplekenn}\label{prop:counterexample-k-sequence}
    Let $(k_n)_{n=1}^{\infty}$ be any sequence of natural numbers. Define the sequence of estimators $\smash{\kseq}$ to be the nearest neighbor risk estimators using $\nval$ validation points and $k_{\nval}$ neighbors. Then there exists a data-generating process satisfying infill asymptotics, a test set containing a single point, a predictive method $h$ resulting in an error function satisfying the Lipschitz assumption, and some $\!\epsilon, \delta>\!0$ such that with probability at least $1\!-\!\delta$, $\forall \nval$, $|\smash{\kseq}(h) - \conditionalrisk_{\Qtest}(h)| \geq \epsilon$. 
\end{restatable}

See \cref{app:inconsistency-k-sequence} for a proof. There are two cases. (1) If the number of neighbors is bounded, the estimator can suffer from non-vanishing variance as in the 1NN case. Or (2) the number is unbounded, so there exists a sequence of validation sites that accumulates slowly around each test site leading to non-vanishing bias. Inconsistency of both 1NN and the \holdout{} are corollaries of \cref{prop:counterexample-k-sequence}; for 1NN, choose:  $\forall n, k_n \!=\! 1$. For the \holdout{}, choose: $\forall n, k_n \! =\! n$.

\section{A CONSISTENT ESTIMATOR}
\label{sec:new_method}
We next provide a novel bound on the test risk estimation error of kNN. We propose using a kNN estimator with $k$ chosen by optimizing our bound. We show that our proposed estimator is consistent for test risk under infill asymptotics. We here focus on error estimation; in \cref{app:model-selection} we provide promising results for model selection and discuss open challenges.

\subsection{Our Bound and Estimator}\label{sec:k-selection-procedure}

In light of the examples in \cref{sec:inconsistent}, we propose to trade off the larger variance of small $k$ and larger bias of large $k$ by optimizing a bound depending on the validation set. Crucially, we adapt $k$ using the actual locations of the test and validation sites, as \cref{prop:counterexample-k-sequence} suggests such adaptivity is necessary. 
To that end, we first derive a bound on the test-risk estimation error as a function of $k$ and the locations of test and validation sites.
To state our bound, it will be useful to define the \emph{$k\textsuperscript{th}$-order fill distance}\footnote{We assume in this definition that all spatial locations are distinct. If not, $\Psi_1$ should be treated as a multi-set.
} of a set $\Psi_1$ in a set $\Psi_2$ as the maximum distance from a point in $\Psi_2$ to its $k\textsuperscript{th}$ nearest neighbor in $\Psi_1$:
\begin{align}
    \zeta^k(\Psi_1; \Psi_2) = \sup_{S_2 \in \Psi_2}\,    \inf_{\substack{A \subset \Psi_1,\\   |A| = k}} \,\,\sup_{S_1 \in A} \spatialmetric(S_1, S_2).
\end{align}

\begin{restatable}[Bound on Estimation Error in Terms of Fill Distance]{theorem}{nnbound}\label{thm:nn-fill-bdd}
Consider a validation set $\Dval$ of size $\nval$ and a test set $\Dtest$ of size $\ntest$.
Take the $k$-nearest neighbors test-risk estimator from \cref{def:nn-estimator}. Choose $\delta \in (0,1)$ and $k$ such that $1 \leq k \leq \nval$. Let $\rho_k\coloneqq \zeta^k(\Sval_{1:\nval}, \Stest_{1:\ntest})$ and $\seconstant{\delta}:=\lossbdd\sqrt{\frac{1}{2}\log\frac{2}{\delta}}$. Take \cref{assumption:DGP,assumption:BoundedLoss,assumption:DGP_val,assumption:Lipschitz}. Then, with probability at least $1-\delta$,
    \begin{align}
          |\conditionalrisk_{\Qtest}(h) - &\nearestneighborrisk(h)| \leq   \lipschitzbdd \rho_k + \seconstant{\delta}\|\nnweights \|_2 \nonumber \\
          \leq \lipschitzbdd\rho_k  &+ \seconstant{\delta}\sqrt{\max_{1 \leq n \leq \nval}\Qtest(\ball{\Sval_n}{ \rho_k})/k}, 
          \label{eqn:both-ub} 
    \end{align}
    where $\ball{S}{r}$ denotes the ball of radius $r$ centered at $S$. See 
\cref{assumption:BoundedLoss,assumption:Lipschitz,def:test-risk} for $\lossbdd$, $\lipschitzbdd$, $\Qtest$ respectively.
\end{restatable}
We prove \cref{thm:nn-fill-bdd} in \cref{app:proof-nn-bound}. We use the right-hand side of the first inequality algorithmically, and the right-hand side of the second inequality in proofs and to gain intuition for cases when the bound is small. The first term on the far-right-hand side in 
\cref{eqn:both-ub} is a worst-case upper bound on the bias of our estimator; it is large if the average loss varies quickly in space or if validation data is not available near test data. Larger $k$ may increase the first term. The second term comes from applying a tail bound; if most of the weight is put on a few validation points, the resulting estimator has high variance and this term is large. Sufficiently large $k$ will decrease the second term. If we can find a $k$ such that both (a) the distance from each test point to its $k$ nearest neighbors is small and (b) no validation point has too much impact on our estimator, $\smash{\nearestneighborrisk(h)}$ provides a good estimate for $\smash{\conditionalrisk_{\Qtest}(h)}$.  

\Cref{thm:nn-fill-bdd} is closely related to \citet[Prop.~4]{portier2023scalable}. While there are technical differences in the proof and algorithm (and the risk that is bounded), the substantive distinction is that we state our bound directly in terms of the fill distance, instead of upper bounding this distance again as done in \citet{portier2023scalable}. We can therefore avoid making assumptions about the distributions of the sites; we instead highlight the fill distance of the validation set as an essential quantity in controlling the accuracy of nearest neighbor risk estimation. 

\textbf{Selection Procedure with Unknown Lipschitz Constant.}
If the Lipschitz constant of the average loss, $\lipschitzbdd$, can be upper bounded, for example by knowledge about how quickly varying the spatial processes involved in the analysis are, then $k$ can be selected by minimizing the first upper bound in \cref{eqn:both-ub}. Since the bound is conditional on the validation and test sites, the bound still holds with the same probability for $k$ selected by this minimization. However, it will generally be the case that the Lipschitz constant is unknown. We therefore suggest choosing the number of neighbors by minimizing the upper bound from \cref{thm:nn-fill-bdd} with $1$ in place of the Lipschitz constant:
\begin{align}
    \smash{\kstar[T] \in \arg \min\nolimits_{k \in T} \rho_k + \seconstant{\delta}\|\nnweights\|_2}. \label{eqn:num-neighbors-minimization}
\end{align}
For computational efficiency, we focus on choosing $k$ as a power of 2: $T = T_2 := \{2^{i}\}_{i=1}^{\smash{\lfloor \log_2 \nval \rfloor}}$. We call the resulting estimator \emph{spatial nearest neighbors} (SNN).

\subsection{Our Method is Consistent}\label{sec:general-consistency}

We show that SNN is spatially consistent.
\begin{restatable}[Our Method is Consistent under Infill Asymptotics]{corollary}{predictionconsistent}\label{cor:prediction-consistent} Let $\spatialdomain = [0,1]^d$. Take \cref{assumption:DGP,assumption:BoundedLoss,assumption:DGP_val,assumption:Lipschitz}. Let $\tilde{\rho} := \zeta(\Sval_{1:N}, \spatialdomain)$. Let $\kstar \in \arg \min_{k \in T_2} \rho_k + \seconstant{\delta}\|\nnweights\|_2$ with $\delta = \min(1, r)$ and $r \in [\smash{c\tilde{\rho}, C\tilde{\rho}}]$ for some constants (possibly depending on dimension) $c, C >0$. Then the $\kstar[T_2]$-nearest neighbor risk estimator is consistent under infill asymptotics.
\end{restatable}

We prove \cref{cor:prediction-consistent} in \cref{app:consistency-nn}. \Cref{cor:prediction-consistent}  states that selecting the number of neighbors by minimizing an upper bound on estimation error leads to an estimator that is consistent regardless of the test data, as long as the validation data are dense on the unit cube. In \cref{app:computation-approx-fill}, we provide a computationally efficient algorithm for calculating an $r$ satisfying the condition $c\tilde{\rho} \leq r \leq C\tilde{\rho}$, and we prove the correctness of this algorithm. 
We show that the computational complexity of SNN is $O(d\ntest\nval+ \ntest\nval\log \nval)$ in \cref{prop:comp-complexity-snn} and state the computational cost in experiments in \cref{app:computational-complexity-1nn}.


\section{EXPERIMENTS}
\label{sec:experiments}
\begin{figure*}[ht]
    \centering    \includegraphics[width=0.495\textwidth]{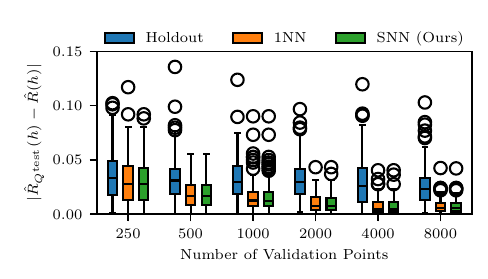}
    \includegraphics[width=0.495\textwidth]{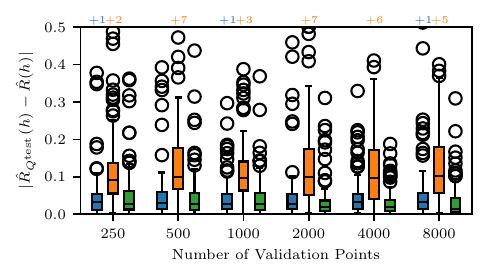}
    \caption{Error for test risk estimation in the \emph{grid prediction task (left)} and \emph{point prediction task (right)}  across methods (\holdout{} in blue, 1NN in orange, our SNN in green); lower values correspond to better performance. The vertical axis shows the absolute difference between the estimated test risk and empirical test risk. Each box plot shows the median, inter-quartile range, and outliers based on 100 synthetic datasets. The horizontal axis tracks increasing validation set sizes. Numbers above the upper box indicate the number of outliers falling above the vertical limit.}
    \label{fig:consistency-grid-plot}
\end{figure*}
Our theory suggests the \holdout{} exhibits substantial bias in many tasks. And we expect 1NN to exhibit high variance in point prediction tasks. Our experiments confirm these observations. While there exist tasks where either the holdout or 1NN performs similarly to SNN, there are also many tasks where each performs much worse than SNN. Since SNN performs well across all experiments, we prefer SNN when a new task arises.

\textbf{Ground Truth.} In a traditional machine learning prediction task, analysts ask how well a predictive method predicts the observed response at a set of covariates, so the observed response would form the ground truth. Since here we instead judge evaluation methods, we must ask how well the evaluation method estimates test risk (\cref{def:test-risk}); that is, the true test risk now forms ground truth. From \cref{assumption:DGP_val}, the true test risk requires an integral over the (unknown) noise $P_{\epsilon}$. Accessing ground-truth responses is often easy; by contrast, it is highly unusual to access even a high-quality approximation of the integral for test risk (much less the exact test risk) in a real task. Therefore we devise a series of workarounds. First, we consider realistic tasks with simulated data. Second, we consider a realistic task with a semi-simulated data set, where we control the noise distribution by constructing it from bootstrapped residuals that arise from real data. Third, we use fully real data to construct a ground truth by considering an unrealistic task. Finally, we consider fully real data and a realistic task by forfeiting access to ground truth.

\subsection{Test Risk Estimation on Synthetic Data}\label{sec:risk-estimation-synthetic}
We set up two fully synthetic experiments: a grid prediction task and a point prediction task. Based on our analyses above, we expect the \holdout{} to struggle with the former and 1NN to struggle with the latter; our experiments confirm this intuition.
See \cref{app:risk-estimation-synthetic} for full experiment details.

\textbf{Data and Ground Truth.}
In both experiments, we vary $\nval \in \{250, 500, 1000, 2000, 4000, 8000\}$. We use a truncated squared loss: $\ell(a,b) = \max((a-b)^2, 1.0)$.
For the \emph{grid task}, the test sites are a $50 \times 50$ grid of equally spaced points in $[-0.5, 0.5]^2$ (orange points in \cref{fig:synthetic-sites}). We generate the validation sites via a sequential process that leads to clustering (blue points in \cref{fig:synthetic-sites}). For the \emph{point task}, there is a single test site at $(0,0)$. Validation sites are i.i.d.\ uniform in $[-0.5,0.5]^2$.
For both tasks, we generate covariates and responses conditional on the sites:
$Y_i^{\smash{j}}         = f(\spatialfield^1(S^{\smash{j}} _i), \spatialfield^2(S^{\smash{j}}_i)) + \eta(S^j_i) + \epsilon_i^{\smash{j}}, \,  \epsilon^j_i  \smash{\iidsim} N(0, \sigma^2)$. We generate $\eta, \spatialfield^1,\spatialfield^2, f$ according to independent Gaussian processes (GPs); we describe our kernel and parameter choices in \cref{app:simulation-dgp}. We plot examples of the generated data in \cref{fig:grid-prediction-data-0,fig:point-prediction-data-0}.
We make draws from the data-generating process to form an unbiased Monte Carlo estimate of the test risk, $\smash{\empconditionalrisk_{\Qtest}}(h)$, and use this estimate as ground truth; see \cref{app:mc-estimate-test-risk-synthetic} for details.

\textbf{Spatial Predictive Method.}
To arrive at our spatial predictive method, we generate training data according to the same distribution as the validation data.
Since real-world analyses are often missing potentially relevant covariates, we use only the first covariate (and not the second) as a realistic form of misspecification.
We fit a GP regression, with zero prior mean and the same kernel used in data generation; we predict using the posterior mean.

\textbf{Results.}
We expect our SNN estimator to be consistent in all tasks. In the grid task, we expect the variance of 1NN to be low since there are many test points spread across space. And we expect the bias of the \holdout{} to be high since the test and validation points have different spatial arrangements.
Our results in the grid task (\cref{fig:consistency-grid-plot}, left) agree with our intuitions; the errors of 1NN and our SNN decrease much more rapidly across the values of $\nval$ than the \holdout{}.

Given the single test point in the point task, we expect the high variance of 1NN to be an issue with substantial probability. \Cref{fig:consistency-grid-plot} (right) agrees with our intuition; the error of our SNN estimator decreases much more rapidly across the values of $\nval$ than 1NN. In this case, we find that the \holdout{} errors decrease rapidly as well. We also plot the (signed) relative difference of each estimator to the empirical test risk in \cref{fig:grid-rel-risk,fig:point-rel-risk} (\cref{app:risk-estimation-synthetic}).

In \cref{tab:grid-k-chosen,tab:point-k-chosen} (\cref{app:risk-estimation-synthetic}), we show $k$, the number of neighbors selected by SNN for each of the two tasks.
In the grid task, the $k$ selected was at most 4 in all cases we considered. In the point task, the value of $k$ selected increased with $\nval$, though it always remained over an order of magnitude smaller than $\nval$.

\subsection{Temperature, Bootstrapped Residuals}\label{sec:air-temp-bootstrapped}

We next consider a real task on a semi-synthetic dataset. We find that 1NN performs poorly; while the \holdout{} performs best, SNN performs well. Full details can be found in \cref{app:airTemp}.

\textbf{Data and Ground Truth.}
Our test task is prediction of monthly average air temperature in January 2023 at the 5 largest urban areas in the United States (New York City, Los Angeles, Chicago, Miami, and Houston), based on available weather station data in the same month  \citep{ghcnm}.
Loss is truncated absolute error (in °C).
To access ground truth test risk, we create a partially synthetic response variable. We first fit Gaussian process regression (GPR) to all the available weather station data. We build 100 synthetic datasets by calculating the residuals of the posterior mean of this model, sampling a residual value for each weather station and point we want to predict at and adding these to the mean prediction of the model. Because we then have access to samples from the distribution of response values at the test sites under this data-generating process, we can obtain (an accurate estimate of) ground truth test risk of a predictive method on this partially synthetic response (\cref{sec:estimation-ground-truth-airtemp}).

\textbf{Spatial Predictive Methods.} We train two predictive methods on this data: GPR and a geographically weighted regression (GWR) based on MODIS-Aqua \citep{modisLST} land surface temperature measurements, inspired by \citet{hooker2018global}. We use 50\% of the weather station locations for training the predictive methods and the remaining 50\% for validation (3211 observations in each).

\textbf{Results.} The error in estimating the predictive performance of both methods is shown in \cref{fig:realdata-boxplots} (far and mid left) for 100 different datasets with different samples of the residuals (but the same training and validation split). Given the point prediction task, we expect 1NN (orange, middle) to have a high variance; the figure confirms our intuition. The estimates given by SNN (green, right) and the \holdout{} (blue, left) are much closer to the ground truth. In this case, the \holdout{} has a small bias, and its variance is substantially lower than SNN. So, in this case, the \holdout{} typically returns slightly better estimates of the error than SNN. If many more weather stations were used for validation, we expect that the SNN would eventually outperform the holdout estimate.
\begin{figure}[t]
    \includegraphics[width=0.98\columnwidth]{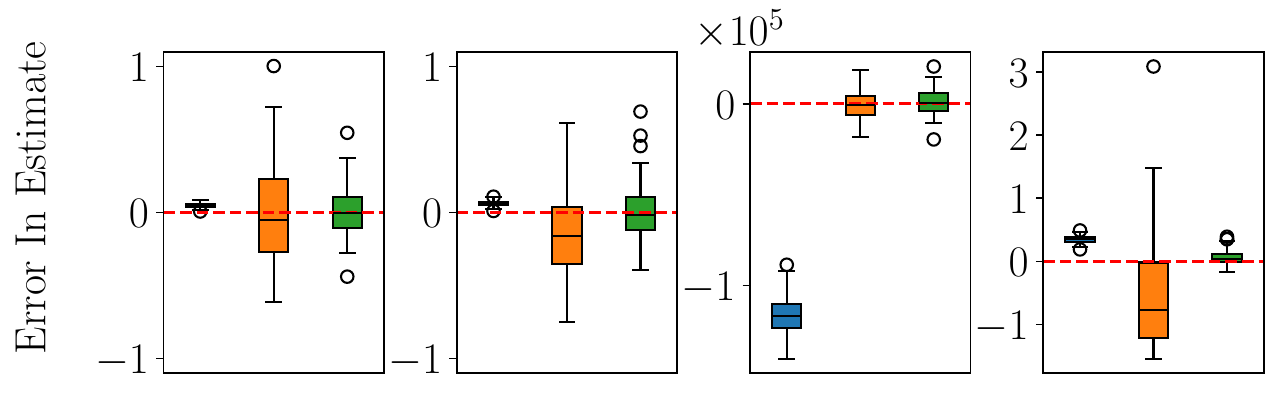}
    \caption{Signed errors in estimating the test risk for (left to right) the bootstrapped air temperature task with GWR (°C); the same task with GPR; the flat price task (£); and the wind speed task (m/s). The \holdout{} (blue), 1NN (orange), and SNN (green) appear left to right in each plot.}\label{fig:realdata-boxplots}
\end{figure}

\subsection{Property Sales in England and Wales}\label{sec:uk-ppd}
Here and in \cref{sec:wind_speed}, we define somewhat unrealistic tasks to access ground truth on real data. Here we find that 1NN and SNN perform similarly well while the \holdout{} exhibits a large bias. \Cref{app:UK-housing} contains additional details and figures for this experiment.

\textbf{Data and Ground Truth.}
We predict the price of a flat in England and Wales based on location, loosely following \citet{hensman2013gaussian} but using data from 2023 \citep{ppduk}.
We make 100 datasets by sampling a training dataset of 40,000 points from flat sales outside London, a test set consisting of 1,000 flat sales within London, and a validation set of the remaining sales (31,484 outside London, 21,179 in London). The loss is truncated mean absolute error.
In \cref{app:uk-estimation-truth}, we justify how we can form a high-quality estimate of ground-truth test risk by: assuming a form of independence, using a bounded loss, and applying Hoeffding's inequality. 

\textbf{Spatial Predictive Method.} We fit a Gaussian process regression model with variational inference as in \citet{hensman2013gaussian} with minor modifications; we use a sum of two Mat\'ern \nicefrac{3}{2} kernels. We use the non-stochastic version of variational inference in GPR \citep{titsias2009variational} to avoid known difficulties with tuning hyperparameters in the stochastic version \citep{Ober2024recommendation}, and 2000 inducing points.

\textbf{Results.} Because the model is trained using data only outside London, we expect it to make larger errors predicting flat sales in London than outside London. As a result, we expect the \holdout{} to have a large bias. We expect both 1NN and SNN to perform similarly: since the test sites are sampled randomly within London, we expect them to have different nearest neighbors and so variance of 1NN should be reasonably small. The mean absolute error of each method, relative to the estimate of the ground truth, is shown in \cref{fig:realdata-boxplots} (mid right). As expected, the \holdout{} substantially underestimates the test risk, while the other estimators perform reasonably well.

\subsection{Wind Speed Prediction} \label{sec:wind_speed}

In this next experiment with an unrealistic task but fully real data, we find that both the \holdout{} and 1NN perform poorly while SNN performs well. Full details are in \cref{app:windSpeed}.

\textbf{Data and ground truth.} Our test task is predicting the average wind speed on a typical day in January at Chicago O'Hare airport, from daily historical weather station data \citep{ghcnd}. There are 775 historical weather station observations at the Chicago O'Hare site in January months. We split the remaining weather stations (at a station level) into training and validation sets consisting of 962 training stations and 241 validation stations, respectively. Each station has a different number of observations depending on how complete the historical record of average daily wind speed data is at this site. Typically the training set consists of on the order of $580{,}000$ measurements and the validation set around $126{,}000$ measurements. We perturb the location labels of each measurement at each weather station by a tiny amount to avoid ties when running nearest neighbors. The test set contains all of January rather than just a particular date so that we can form a high-quality estimate of ground truth. To form ground truth, we also assume that average wind speed decorrelates in time quickly. The loss is truncated (root) mean square error. See \cref{app:estimation-ground-truth-windspeed} for more details.

\textbf{Spatial Predictive Method.} We use LightGBM \citep{ke2017lightgbm} to make predictions.

\textbf{Results.} We expect 1NN to have a high variance because the task is point prediction. \Cref{fig:realdata-boxplots} (right) confirms this. The \holdout{} has large bias for this task while SNN exhibits low bias and low variance.

\begin{table}[t]
    \centering
    \caption{Test risk estimates for the 5-metros task. Rows correspond to predictive methods and columns to estimators. We report two-standard-error intervals for the \holdout{}. (Recall that these intervals are justified under the common, but here-incorrect, assumption that data are i.i.d.) For each estimator, we bold the predictive method with lower estimated test risk. }\label{table:airtemp-metro}
    \begin{tabular}{c|ll} 
        & GWR                      & Spatial GP \\ \hline
Hold. & $\mathbf{0.83 \pm 0.03}$ & $0.90 \pm 0.04$  \\
1NN     & $0.61$                   & $\mathbf{0.44}$ \\
SNN     & $\mathbf{0.53}$          & $0.61$
\end{tabular}

\end{table}
\subsection{Temperature, Real Response}\label{sec:airtemp}

We finally consider a case with real data and a real task. Although we cannot access ground truth, we show that the \holdout{}, 1NN, and our SNN give very different estimates of test risk and can differ in model selection. Given all the previous results, we suggest using SNN. See \cref{app:airTemp} for full details and a grid-prediction task where all estimators are in agreement.

\textbf{Data and Models.}
The data, models and test task are the same as in \cref{sec:air-temp-bootstrapped}, but the actual response values are used for training and validation.

\textbf{Results.} 
\Cref{table:airtemp-metro} shows a large discrepancy in test risk estimates across estimators. We see that 1NN chooses a different predictive method (spatial GP) than the \holdout{} or our SNN does.

\section{DISCUSSION}
\label{sec:discussion}
In this paper, we precisely formulate the problem of validation in the spatial setting; we provide a new check for validating spatial predictions (consistency under infill asymptotics); we demonstrate empirically and theoretically that common existing methods don’t pass this check; and we show by construction that it is possible to pass the check with a practical algorithm (our SNN algorithm).

\textbf{Validation of Spatio-temporal Forecasts.}
Many prediction tasks have a temporal as well as a spatial component. Our method can be applied when treating time as part of a ``location,'' but infill asymptotics would then imply we expect access to arbitrarily dense observations in time. We believe it is more typical in temporal prediction tasks to accrue data forward in time and to be interested in prediction forward in time. In the spirit of the present work, one might formalize the intuition that past forecasting performance is informative about future performance as a stationarity assumption.


\textbf{Validation for Extrapolation.}
Even in spatial tasks, practitioners often want to extrapolate to unobserved parts of the space. Indeed, blocked spatial validation is often motivated by this goal. We expect consistent risk estimation in this setting to require additional assumptions about the data-generating process (beyond our smoothness assumption, \cref{assumption:Lipschitz}). We also expect that blocked spatial validation may still be practically useful in extrapolation settings, as discussed in \citet{roberts2017cross}. Consistent estimation of predictive error in extrapolation settings (either forward in time or across space) seems a challenging, but important, area for future research.  

\textbf{Cross-validation.}
We have focused on validation of a fixed predictive method. When statistical or machine learning methods are used for prediction, and data is scarce, cross-validation is frequently used.
However, recent research has recognized that cross-validation actually incorporates an interesting, non-trivial additional layer of complexity beyond validation \citep{bates2023cv}. Indeed, this research has shown that even in the classic independent and identically distributed setting, cross-validation is not estimating the same quantity as validation. So we expect a full understanding of cross-validation in the spatial setting to similarly require substantial additional work beyond understanding validation.

\subsection*{Acknowledgments}
This work was supported in part by an NSF CAREER Award, the Office of Naval Research under grant N00014-20-1-2023 (MURI ML-SCOPE), Generali, and a Microsoft
Trustworthy AI Grant.

\bibliographystyle{abbrvnat}
\bibliography{main}

\onecolumn
\newpage
\twocolumn
\section*{Checklist}

\begin{enumerate}

  \item For all models and algorithms presented, check if you include:
        \begin{enumerate}
          \item A clear description of the mathematical setting, assumptions, algorithm, and/or model. \textbf{Yes. Assumptions needed for our algorithm to be consistent are stated in the main text and are linked in the following result statement: \cref{cor:prediction-consistent}.}
          \item An analysis of the properties and complexity (time, space, sample size) of any algorithm. \textbf{We derive the computational complexity of our algorithm in \cref{app:computational-complexity-1nn}, see \cref{prop:comp-complexity-snn}.}
          \item (Optional) Anonymized source code, with specification of all dependencies, including external libraries. \textbf{Yes, code is available at: \coderepository. Dependencies can be found in the setup file, and instructions to run the code are in the readme file.}
        \end{enumerate}

  \item For any theoretical claim, check if you include:
        \begin{enumerate}
          \item Statements of the full set of assumptions of all theoretical results. \textbf{Yes, we state all  assumptions in theorem statements.} 
          \item Complete proofs of all theoretical results. \textbf{Yes, proofs are in \cref{app:proofs}.}
          \item Clear explanations of any assumptions. \textbf{Yes, we have provided intuition for, and descriptions of, \cref{assumption:DGP,assumption:BoundedLoss,assumption:DGP_val,assumption:Lipschitz} around the statement of these assumptions.}
        \end{enumerate}

  \item For all figures and tables that present empirical results, check if you include:
        \begin{enumerate}
          \item The code, data, and instructions needed to reproduce the main experimental results (either in the supplemental material or as a URL). \textbf{Yes, see prior answer regarding code. Experimental details are included in \cref{app:experimental-details}.}
          \item All the training details (e.g., data splits, hyperparameters, how they were chosen). \textbf{Yes, all training details are provided in \cref{app:experimental-details}.}
          \item A clear definition of the specific measure or statistics and error bars (e.g., with respect to the random seed after running experiments multiple times). \textbf{Yes, we provide detailed descriptions of the types of randomness considered in each experiment (e.g.\ bootstrapping versus generating completely independent synthetic experiments) in \cref{app:experimental-details}, as well as the method and limitations for estimating the ``ground truth'' risk.}
          \item A description of the computing infrastructure used. (e.g., type of GPUs, internal cluster, or cloud provider). \textbf{Yes, we describe computing infrastructure in \cref{app:compuational-setup}.}
        \end{enumerate}

  \item If you are using existing assets (e.g., code, data, models) or curating/releasing new assets, check if you include:
        \begin{enumerate}
          \item Citations of the creator If your work uses existing assets. \textbf{Yes, we provide citations for datasets in \cref{app:airTemp-datasources,app:uk-data,app:windSpeed-data}. We include citations for libraries used as well.}
          \item The license information of the assets, if applicable. \textbf{Yes, we describe license information of datasets in \cref{app:airTemp-datasources,app:uk-data,app:windSpeed-data}.}
          \item New assets either in the supplemental material or as a URL, if applicable. \textbf{The only new asset is code. See earlier responses.}
          \item Information about consent from data providers/curators. \textbf{Not applicable. All data used is permitted by terms of licensing. Most is open data (product of government agency).}
          \item Discussion of sensible content if applicable, e.g., personally identifiable information or offensive content. \textbf{Not Applicable}
        \end{enumerate}

  \item If you used crowdsourcing or conducted research with human subjects, check if you include:
        \begin{enumerate}
          \item The full text of instructions given to participants and screenshots. \textbf{Not Applicable}
          \item Descriptions of potential participant risks, with links to Institutional Review Board (IRB) approvals if applicable. \textbf{Not Applicable}
          \item The estimated hourly wage paid to participants and the total amount spent on participant compensation. \textbf{Not Applicable}
        \end{enumerate}

\end{enumerate}

\newpage
\onecolumn
\appendix
\addcontentsline{toc}{section}{Appendix} 
\part{Appendix} 
{
  \hypersetup{linkcolor=black}
  \parttoc
}

\section{EXTENDED RELATED WORK}\label{app:related-work}
\droptag{short}
\usetag{long}
\subsection{Validation versus Cross-Validation}\label{app:val-vs-cv}
While we expect our work to have implications for cross validation (CV), we focus on validation here since (1) we do not assume training data exists or is easy to change (for example in the case of a large physical model) and (2) CV presents additional subtle challenges. The \holdout{} is broadly applicable to any predictive method (whether data-driven, physical, or a combination thereof). However, in the cases when data is scarce and a data-driven predictive method is used, cross-validation is commonly believed to make more efficient use of the data. CV is widely used in spatial analyses; among many examples are \citet{WANG2022101506, valavi2021modelling, Kianian2021imputing}. However, the interpretation of the error estimates given by cross-validation is subtle even in the classical i.i.d.\ setting \citep{bates2023cv}. We focus on the validation setting in this work due to its broad applicability and clear interpretation, though extensions to cross-validation, as well as a theoretical understanding of how the resulting error estimate relate to predictive performance is a promising direction for future work.

\subsection{Covariate Shift}
In covariate shift, it is generally assumed that $\Sval_n  \iidsim P^{\text{val}}$, $\Stest_m \iidsim Q$ and $U^{\ell}_i| S^{\ell}_i \iidsim \mu$ for $\ell \in \{\text{val},\text{test}\}$. Moreover, it is typically assumed the density ratio $\frac{d Q}{d P}$ exists and is bounded, although there is recent work relaxing this assumption in the context of non-parametric regression \citep{kpotufe2021marginal, pathak2022new}. Mean estimation seeks to estimate $\EE[U^{\text{test}}_1]$. Taking $(U^{\text{test}}_m|\Stest_m) = \ell(\Ytest_m, h^{\spatialfield}(\Stest_m))$, this is the same task we consider, but with the (stronger) assumption that covariates are independent and identically distributed with a bounded density ratio between the test and validation distributions. This assumptions is not appropriate for validation with many spatial datasets, with particularly simple examples being when the task of interest requires prediction at a single spatial location, or on a regular grid. Our assumptions are essentially a conditional formulation of mean estimation under covariate shift. Many methods proposed for addressing mean estimation under covariate shift, including all the approaches we describe below, are based on re-weighting validation points, then applying the holdout approach described earlier with these weights.

The kernel mean matching algorithm provides a solution under standard covariate shift assumptions with a bounded density ratio and assuming the average loss as a function of space $\averageloss$ lives in a reproducing kernel Hilbert space (RKHS) and has small norm \citep{gretton2009covariate}. \citet{yu2012analysis} provide a finite sample bound for this method\iftagged{short}{.}{, showing that under the assumptions outlined above, with high probability kernel mean matching can estimate the $\conditionalrisk_Q(h)$ with error $O_p(\frac{1}{\sqrt{\ntest}}+ \frac{1}{\sqrt{\nval}})$.} Recently, \citet{portier2023scalable} considered the mean estimation problem under covariate shift, but relaxed the RKHS assumption to instead assume that the average loss is Lipschitz, the same assumption we take. Their estimator is built on nearest neighbor regression, and they advocate the use of $1$-nearest neighbor. Compared to this work, the primary advantage of our analysis is that it removes the assumption that the sites are independent and identically distributed, making it directly relevant to tasks like grid prediction. Moreover, in the more general setting we consider, $1$-nearest neighbor is not always consistent (\cref{prop:counterexample-1nn}), and using more neighbors can be beneficial. We provide a method for selecting the number of neighbors that has similar statistical properties as their approach for grid prediction, but retains consistency for a wider class of problems where using a single neighbor is no longer consistent.

\subsubsection{Covariate Shift in the Context of Spatial Validation}
Several recent works have applied the covariate shift framework to spatial problems. \citet{sarafian2020domain} considered a weighted estimator motivated by importance weighting to address covariate shift. This suggests taking the weights to equal the density ratio of the test sites to the validation sites (which is assumed to exist) \citep{shimodaira2000improving}. The density of the test sites was assumed to be known and a kernel density estimate was used to estimate the density of the validation sites. \tagged{long}{ While \citet{sarafian2020domain} observe this estimate is unbiased if the density ratio is known, in practice the estimator will be biased due to error in estimating the density ratio.} While this estimator might be consistent with regularity assumptions on the densities, the assumption that the densities exist and have a bounded ratio is restrictive for many spatial tasks -- for example, if we care about the quality of predictions at a few specific locations. \tagged{long}{\citet{debruin2022dealing} consider a similar estimator to \citet{sarafian2020domain}, but normalize the weights to sum to $1$ (which may not be the case for the weights given directly by density estimation).}

\subsection{Predictive Validation for Spatial Data}
While our analysis focuses on variants of the holdout we also discuss methods for spatial cross-validation, as they can often be adapted to cases with held-out data. 

\subsubsection{Limitations of the Holdout Approach} The holdout has been empirically shown to under-estimate error for models trained with data more similar to the validation data than to the test data in several works; the review paper of \citet{roberts2017cross} provides a detailed description of this phenomenon in the context of ecological statistics. Despite concerns raised in previous works, the holdout is widely used in spatial application areas for comparing methods and indicating the reliability of a given method. 

\subsubsection{Spatial Stratification Approaches}
Concerns over the quality of the holdout estimate have led to the development of validation approaches based on evaluating the loss of a model on held-out data far from the data used to train the model \citep{telford2005secret}. \iftagged{long}{Cross-validation strategies for spatial datasets often also focus on evaluating a model on data that are far (in the spatial domain) from the data used to train the model. Spatially stratified blocking approaches are described in \citet{lieske2011robust} and \citet{roberts2017cross}, and several software packages make these cross-validation methods readily accessible in common statistical programming languages, especially \texttt{R} \citep{valavi2019blockcv, mahoney2023assessing}. Similarly, v}{V}ariants of leave-one-out cross-validation in which points close to the point on which error will be assessed are also held-out during training have been developed for sequential data \citep{burman1994cross} and adapted to the spatial setting \citep{telford2009evaluation, Rest2014Spatial}. Several simulation studies support claims that spatial buffering provides more realistic estimates of model risk than the standard cross-validation \citep{roberts2017cross, mahoney2023assessing}. However, other simulation studies have shown that using spatially disjoint regions to train the model and to validate the model can lead to over-estimation of generalization error when the available data covers most of the space in which we are interested in making predictions \citep{ploton2020spatial, debruin2022dealing}.
While simulation studies show the strengths, and some of the limitations, of ensuring validation data are far in space from training data as a method for validating a predictive method, there is not clear theory establishing under what assumptions it allows for accurate evaluation of the risk, or consistent model selection. \tagged{long}{\citet{roberts2017cross} offers some useful heuristics for when spatial stratification is preferable to the holdout. \citet{racine2000consistent} provides a sketch for the consistency of the leave-one-out method described above for model selection in linear models with stationary sequential data, but we are not aware of a detailed proof clarifying the underlying assumptions about mixing of the process, extensions to non-linear models or extensions to the spatial setting.}

\tagged{long}{
\subsubsection{Other Approaches for Spatial Validation}
Other heuristics for estimating the error have emerged in the ecological statistics literature. \Citet{debruin2022dealing} also consider model based approaches based on non-parametric regression of the residual with a Gaussian process (kriging) to estimate the error of the model. The square of this regressor is then used as a plug-in estimator for estimating the squared loss. This approach is proposed as a heuristic, and it is not clear whether it is consistent, especially if the likelihood of the model fit to the residuals is misspecified, which will certainly be the case in practice. \Citet{mila2022nearest, Linnenbrink2023kNNDM} considered the distance between each validation site and its $k$-nearest neighbors in the training sites, and attempted to make the distribution of these distances similar to the distribution between the test sites and the training sites. However, it is not clear what assumptions on the data are needed for such a method to reliably estimate the generalization error of a method. \citet{Meyer2021Predicting} emphasized that the validity of estimates of the generalization error of a spatial model depends on how similar the validation data are to the training data, relative to how similar test data are to the training data, and suggested only providing error estimates over an area that is judged to not be significantly more different from the training data than the available validation data are from the training data. The infill assumption we make in this work provides a particularly simple formalization of this idea, since it means that we have validation data close to every test point, and so we are able to reasonably estimate the error at each test point. \citet{meyer2022machine} provide a recent discussion of challenges of evaluating predictions made on a regular grid (map prediction) as well as other recent references for proposed spatial validations approaches. The point prediction problems we consider are analogous to the local error estimates they advocate for, while the grid prediction problems we consider are a form of global estimate. We provide additional discussion of blocked spatial validation in \cref{app:spatial-blocking}.
}


\subsubsection{Aspects of the Problem we do not Consider}
If a statistical prediction method is used, there is a question of how to partition data between data used for training and validation, which is a central consideration in, for example \citet{mila2022nearest, Linnenbrink2023kNNDM}. In contrast, we focus on the case when a validation set is already decided upon. This allows our approach to be applied to physically-driven prediction methods, statistical methods and combinations thereof. Moreover, this allows our approach to be applied to models that have already been built when validation data becomes available and rebuilding the predictive method with a new training set would be expensive.

\subsection{Non-exchangeable Conformal Prediction}
Tools have been developed for providing confidence intervals for prediction with data that are not exchangeable using variants of conformal prediction \citep{tibshirani2019conformal, mao2022valid, barber2023conformal}. Particularly relevant to our work is \citet{mao2022valid}, who construct confidence intervals at a specific location based on the error at its $k$-nearest neighbors in the validation set. This is conceptually the same as the approach we take in mean estimation, but they focus on confidence intervals instead of risk estimation. They derive consistency results for the coverage of the intervals under an infill asymptotic setting, but do not show finite sample bounds which we prove in this work, making our results more quantitative. Finally, we give a theoretically-grounded method for choosing the number of neighbors by minimizing an upper bound on the error of the estimator, whereas it is unclear how to choose the number of neighbors in their approach.
\subsubsection{Conformal Prediction for Graph Learning.}
\citet{clarkson2023distribution} considered prediction sets in the context of graph learning. The infill setting we study is more natural in continuous metric spaces, as opposed to discrete metric spaces such as graphs. Additionally, there are several other important differences between our work and \citet{clarkson2023distribution}.

\textbf{Target of Estimation.} As with other conformal works, \citet{clarkson2023distribution} is interested in building confidence sets based on a predictor that contain the class of the true node with high probability. We are interested in estimating the average error of a predictive method. While the size of confidence sets and the magnitude of the error of a predictive method can both be thought of as quantifying the quality of predictions made by the method, they measure different things. For example, the average error provides a single summary statistic of the quality of a model, which might be useful for model selection. Whereas confidence sets might be useful for understanding what possible outcomes are plausible or probable.

\textbf{Framework for checking if a validation method is useful} We see our primary contribution as providing a framework for checking if a validation method has reasonable, minimal desirable properties in spatial settings. \citet{clarkson2023distribution} is focused on the construction of a particular form of confidence set and verifying coverage properties of this construction.

\textbf{Selection of number of neighbors} Both our estimator (Spatial Nearest Neighbors) and the method proposed in \citet{clarkson2023distribution} are based on nearest neighbors. Within the risk estimation literature, there is also prior work using kNN. The primary novelty of our estimator is giving an adaptive way to select the number of neighbors. \citet{clarkson2023distribution}does not suggest a method for selecting $k$, and focuses on cases where $k=1$ or $k=2$ neighbors suffice (we give counterexamples to this in our setting). \citet{clarkson2023distribution} also suggests estimators that use continuous weights instead of neighbors, as in \citet{mao2022valid}.

\section{ERROR ESTIMATION FOR MODEL SELECTION}\label{app:model-selection}
A practitioner will commonly select a predictive method within some collection by choosing the method with lowest estimated test risk. We consider two cases.

\textbf{(1) Fixed predictive methods.} So far we have focused on the setting where predictive methods are fixed in advance. As we accrue  validation data, the consistency of SNN ensures it will eventually choose the predictive method with lowest test risk. Without consistency, \holdout{} and 1NN cannot be trusted to choose the best predictive method.

\textbf{(2) Proportional training and validation data.} If we consider the special case where the predictive method is fit using training data (vs., say, a physical model), it is common for a practitioner to have a single set of available data that they then partition into training data and validation data; for instance, a fixed percentage of the total data may go to training. In this case, the predictive method changes as the validation set grows. Nonetheless, we can still be sure to eventually choose the predictive method with the lowest risk if the test-risk estimate has a faster rate of convergence in the number of validation data points than the convergence rate of the predictive method in the number of training points.

We are not able to formally characterize how SNN performs in model selection. But we provide rigorous results on rates of convergence of SNN that give suggestive guidance. In particular, (a) we first consider the case of a \emph{grid prediction task}, where test sites are arranged on a grid. In this case, we are able to give a rate of convergence for both $1$NN and SNN (\cref{cor:map-fill-bdd,cor:map-iid-bdd}). If the validation data are i.i.d.\ our rate is faster than the minimax optimal convergence rate for predicting a Lipschitz function in the presence of additive, homoskedastic noise.  Our result shows the same rate of convergence as \citet[Prop.~4]{portier2023scalable}, who considered test data that was i.i.d.\ instead of on a grid. Our result is suggestive that both nearest neighbor methods may perform well at model selection for validating maps (grid prediction). 

(b) Second, we provide finite-sample bounds and asymptotic characterizations for general (non-grid) $\Qtest$ (\cref{cor:general-dense-bdd}). In this case, when validation data is i.i.d., up to logarithmic factors, we show that SNN converges \emph{at} the optimal rate of convergence for Lipschitz functions (\cref{cor:dense-bdd-iid}). Since in this case, the SNN convergence rate is not \emph{strictly} faster than the training convergence rate, it might be difficult to select between statistical methods that converge at the optimal regression rate. But we would still expect to be able to select between (i) a statistical method that converges at the optimal rate and (ii) one that does not (for example a misspecified parametric model). Since the \holdout{} and 1NN may not even be consistent, we could not rely on them to select the better model even in this latter, easier case.

\section{PROOFS OF CLAIMS}\label{app:proofs}
We now present results and proofs not included in the main text. We essentially follow the order of results in the main text. Section \cref{app:lipschitz} gives sufficient conditions for \cref{assumption:Lipschitz} to hold for a homoskedastic, additive noise model and squared loss for responses taking values in $[0,1]$, mentioned in \cref{sec:estimating-test-risk}. \Cref{app:infill-iid} focuses on the proof of \cref{prop:iid-implies-infill}, first recalling several properties of covering numbers that will be used in the result. In \cref{app:inconsistency-claims} we restate and prove our \cref{prop:holdout-inconsistent,prop:counterexample-1nn,prop:counterexample-k-sequence}, which show limitations of existing validation methods in the spatial setting we study. In \cref{app:proof-nn-bound} we prove a general result upper bounding the error in estimating the risk using $k$-nearest neighbors, as well as an upper bound on the error for $\kstar$.  \Cref{app:consistency-nn} establishes the consistency of spatial nearest neighbors. \Cref{app:model-selection-theory} discusses issues related to model selection and proves rates of convergence for spatial nearest neighbors.

\subsection{Lipschitz Constant for Lipschitz Response and Predictive Method}\label{app:lipschitz}
While assuming the average loss is Lipschitz continuous as a function of space is mathematically convenient, it is perhaps more natural to make assumptions about the spatial field we are trying to make predictions about, as well as the predictive method we are using to make prediction. The following proposition gives an example of how Lipschitz continuity of the processes involved can imply \cref{assumption:Lipschitz}.

\begin{proposition}\label{prop:lipschitz-claim}
Consider $\responsedomain \subset [0,1]$ and squared loss. Suppose $f^{\spatialfield}(S):= f(S, \spatialfield(S))$ is $\lipschitzbdd_Y$-Lipschitz and $h^{\spatialfield}$ is $\lipschitzbdd_h$-Lipschitz. Let $(S,X,Y)$ be generated as in \cref{assumption:DGP}, with $Y=f(S, X) + \epsilon$. Then $\averageloss(S) \coloneqq \EE[(Y - h^{\spatialfield}(S))^2|S]$ is $2(\lipschitzbdd_Y+\lipschitzbdd_h)$-Lipschitz.
\end{proposition}
\begin{proof}
    Let $S, S' \in \spatialdomain$ and $\epsilon, \epsilon'$ be the associated noise random variables. Then,
    \begin{align}
        |\averageloss(S) - \averageloss(S')| &= \EE[(f^{\spatialfield}(S) + \epsilon) - h^{\spatialfield}(S))^2 - (f^{\spatialfield}(S')+ \epsilon') - h^{\spatialfield}(S'))^2 ] \\
        &=f^{\spatialfield}(S)^2 - f^{\spatialfield}(S)^2 + h^{\spatialfield}(S)^2 - h^{\spatialfield}(S')^2,
    \end{align}
    where we have used that $\EE[\epsilon] = \EE[\epsilon'] = 0$, $\epsilon, \epsilon'$ are independent from $h$ and $\EE[\epsilon^2] = \EE[(\epsilon')^2]$. Then,
    \begin{align}
        f^{\spatialfield}(S)^2- f^{\spatialfield}(S')^2 + h^{\spatialfield}(S)^2 - h^{\spatialfield}(S')^2 
        &= (f^{\spatialfield}(S) + f^{\spatialfield}(S'))(f^{\spatialfield}(S) - f^{\spatialfield}(S')) \\  & \phantom{\qquad} + (h^{\spatialfield}(S) + h^{\spatialfield}(S'))(h^{\spatialfield}(S) - h^{\spatialfield}(S')) \nonumber \\
        & \leq 2|f^{\spatialfield}(S) - f^{\spatialfield}(S')| + 2|h^{\spatialfield}(S) - h^{\spatialfield}(S')| \\
        & \leq 2 (\lipschitzbdd_Y + \lipschitzbdd_h)\spatialmetric(S,S').
    \end{align}
    The first inequality uses that $\responsedomain \subset [0,1]$ and the second the Lipschitz assumptions on $f$ and $h$.
\end{proof}

Therefore, at least in the case of squared loss with bounded response and predictive method values and a homoskedastic additive noise model, smoothness of the average response surface (as a function of space) together with smoothness of the predictive method imply \cref{assumption:Lipschitz}. We expect  to hold for other losses that are Lipschitz functions of the response and prediction.

\subsection{Proof that Independent and Identically Distributed Data Implies Infill Asymptotic with High Probability}\label{app:infill-iid}

The purpose of this section is to prove \cref{prop:iid-implies-infill}. We begin by recalling this proposition:

\iidimpliesinfill*

Our proof is similar to earlier proofs in \citet[Section 5.2]{Reznikov2015TheCR} and essentially follows the stackoverflow response \citet{pinelis2021fill} keeping track of numerical constants. Essentially, the idea is that were the fill distance to be large, there must be a ball of large radius that doesn't contain any points in the sample of location. But because the probability of a point in the sample falling in any ball is bounded below in terms of the volume of the ball, this must be improbable as the sample size grows. 

\subsubsection{Preliminary Definitions and Lemmas Related to Nets and Covering Number}
In order to formalize the proof of \cref{prop:iid-implies-infill} sketched above idea, we recall the definition of a net and covering number, as well as some standard properties of covering numbers. 

\begin{definition}[Net, Covering Number]\label{def:covering-number}
    Let $A \subset \RR^d$ be a compact set. Any finite set $C \subset \RR^{d}$ such that
    \begin{align}
        A \subset \bigcup_{S \in C} \ball{S}{\epsilon}
    \end{align}
    where $\ball{S}{\epsilon}$ denotes the $d$-dimensional closed Euclidean ball of radius $\epsilon \geq 0$ centered at $S$ is referred to as an $\epsilon$-net of $A$. The $\epsilon$-covering number of $A$, denoted by $\coveringnumber{\epsilon}{A}$ is the minimum cardinality of an $\epsilon$-net of $A$.
\end{definition}
Because $A$ is compact for any $\epsilon$ the $\epsilon$-covering number of $A$ is finite, and will generally increase as $\epsilon \to 0$.

We now recall how covering number behaves under affine transformations of the underlying space. For sets $A, B$ and a scalar $\alpha$, we define $A+B = \{a + b: a \in A, b \in B\}$ and $\alpha A = \{\alpha a: a \in A\}$.
\begin{proposition}\label{prop:scale-translate-cov-num}
    For any compact $A \subset \RR^d$ and $c \in \RR^{d}$, $\coveringnumber{\epsilon}{A} = \coveringnumber{\epsilon}{A + \{c\}}$. For any $\alpha > 0$, $\coveringnumber{\epsilon}{A} = \coveringnumber{\alpha \epsilon}{\alpha A}$.
\end{proposition}
\begin{proof}
    The first claim is shown by noting that if $C$ is an $\epsilon$-net of $A$ then $C+\{c\}$ is an $\epsilon$-net of $A+\{c\}$ so that $\coveringnumber{\epsilon}{A + \{c\}} \leq \coveringnumber{\epsilon}{A}$. Applying a symmetric argument implies the reverse inequality.

    For the second claim, let $C$ be an $\epsilon$-net for $A$. Then,
    \begin{align}
        \alpha A \subset \alpha \left(\bigcup_{S \in C} \ball{S}{\epsilon}\right) = \bigcup_{S \in C} \ball{\alpha S}{ \alpha \epsilon} = \bigcup_{S' \in \alpha C} \ball{S'}{\alpha \epsilon},
    \end{align}
    and so $\alpha C$ is an $\alpha \epsilon$-net of $\alpha A$. It follows that $\coveringnumber{\alpha \epsilon}{\alpha A} \leq \coveringnumber{\epsilon}{A}$. The opposite inequality is obtained via the same argument applied with $\alpha'=\frac{1}{\alpha},$ $A' = \alpha A$ and $\epsilon' = \alpha \epsilon$.
\end{proof}

In the proof of \cref{prop:iid-implies-infill} we apply a union bound over all elements in a net for the unit cube. We therefore need a result telling us that this net does not contain too many elements.
\begin{lemma}[{Covering number of Unit Cube}]\label{lem:covering-number-cube-ub}
    Let $\epsilon \in (0,1]$. Then $\coveringnumber{\epsilon}{[-1,1]^d} \leq \frac{1}{\spherevolume}\left(\frac{6}{\epsilon}\right)^d$
\end{lemma}
where $\spherevolume = \frac{\pi^{d/2}}{\Gamma(\frac{d}{2}+1)}$ is the volume of the $d$-dimensional unit sphere.
\begin{proof}
    By the upper bound in \citet[Lemma 5.7]{Wainwright2019High},
    \begin{align}
        \coveringnumber{\epsilon}{[-1,1]^d} \leq \frac{1}{\spherevolume}\text{vol}\left(\left[-\frac{2}{\epsilon}, \frac{2}{\epsilon}\right]^d + \ball{0}{1}\right).
    \end{align}
    For any $S = S_1 + S_2 \in [-\frac{2}{\epsilon}, \frac{2}{\epsilon}]^d + \ball{0}{1}$ with $S_1 \in [-\frac{2}{\epsilon}, \frac{2}{\epsilon}]^d$ and $S_2 \in \ball{0}{1}$, by the triangle inequality for infinity norm,
    \begin{align}
        \|S\|_{\infty} \leq \frac{2}{\epsilon} + 1 \leq \frac{3}{\epsilon},
    \end{align}
    where we first used that points in the unit ball have infinity norm not more than $1$ and then used that $\epsilon \leq  1$. Therefore,
    \begin{align}
        \text{vol}\left(\left[-\frac{2}{\epsilon}, \frac{2}{\epsilon}\right]^d + \ball{0}{1}\right) \leq       \text{vol}\left(\left[-\frac{3}{\epsilon}, \frac{3}{\epsilon}\right]^d\right) = \left(\frac{6}{\epsilon}\right)^d.
    \end{align}
\end{proof}

\subsubsection{Main Proof}

We again recall \cref{prop:iid-implies-infill} for convenience 
when reading the proof.
\iidimpliesinfill*

\begin{proof}[Proof of \cref{prop:iid-implies-infill}]
    For some $\tau \in (0,1)$ (to be selected later) let $C$ be a minimal cardinality $\tau/2$-net for $[0,1]^d$. If $\ball{S}{\tau/2}$ contains a validation point for all $S \in C$, then for any $S' \in [0,1]^d$, by the triangle inequality,
    \begin{align}
        \min_{1 \leq n \leq \nval} \spatialmetric(S', \Sval_n) \leq \min_{1 \leq n \leq \nval} \left(\min_{S' \in C} \spatialmetric(S', S) + \spatialmetric(\Sval_n, S) \right)\leq \tau.
    \end{align}
    Therefore the probability that the fill distance is large $(> \tau)$ is less than the probability that there exists an element of the net such that no validation point is close to it (within radius $\tau/2$):
    \begin{align}
        \Pr(\zeta(\Sval_{1:\nval}, [0,1]^d) > \tau) \leq \Pr(\exists S \in C : \Sval_n \not\in \ball{S}{\tau/2}\, \forall 1 \leq n \leq \nval). \label{eqn:prob-element-no-val-close}
    \end{align}
    The probability of any particular validation point falling in a ball centered at any point of radius contained in the $\tau /2$ can't be too small since $P$ has density that is bounded below: For any $S \in [0,1]^d$
    \begin{align}
        P(\ball{S}{\tau/2})  = \int_{S' \in \ball{S}{\tau/2}} \frac{dP}{d\lambda}(S') d\lambda(S')
        \geq  \int_{S' \in \ball{S}{\tau/2} \cap [0,1]^d}  c d\lambda(S')
        \geq c\frac{\text{vol}(\ball{S}{\tau/2})}{2^d}, \label{eqn:lower-bound-prob-ball}
    \end{align}
    where $\lambda$ denotes Lebesgue measure and in the final inequality we have used that since $S \in [0,1]^d$, at least one quadrant of $\ball{S}{\tau/2}$ is contained in $[0,1]^d$. 
    
    Returning to \cref{eqn:prob-element-no-val-close} and taking a union over all elements in the net:
    \begin{align}
        \Pr(\zeta(\Sval_{1:\nval}, [0,1]^d) > \tau) & \leq \coveringnumber{\tau/2}{[0,1]^d} \max_{a \in A} \Pr(\Sval_n \not\in \ball{a}{\tau/2}\, \forall 1 \leq n \leq \nval) \\
                                                    & \leq \coveringnumber{\tau/2}{[0,1]^d}  \max_{a \in A} (1-P(\ball{a}{\tau/2}))^{\nval} \label{eqn:sphere-prob-bound}       \\
                                                    & \leq \coveringnumber{\tau/2}{[0,1]^d}  \left(1-\frac{c}{2^d} \spherevolume\left(\frac{\tau}{2}\right)^{d}\right)^{\nval},\label{eqn:prob-bdd}
    \end{align}
    where $\spherevolume = \frac{\pi^{d/2}}{\Gamma(\frac{d}{2}+1)}$ is the volume of the $d$-dimensional unit sphere. \Cref{eqn:sphere-prob-bound} uses \cref{eqn:lower-bound-prob-ball}. We now upper bound the terms in \cref{eqn:prob-bdd}.
    
Applying, \cref{prop:scale-translate-cov-num,lem:covering-number-cube-ub}
    \begin{align}
       \coveringnumber{\tau/2}{[0,1]^d}   \leq \frac{1}{\spherevolume}\left(\frac{6}{\tau}\right)^d. \label{eqn:cov-num-bdd}
    \end{align}
    Using the inequality $(1-x) \leq e^{-x}$,
    \begin{align}
         \left(1-\frac{c}{2^d} \spherevolume\left(\frac{\tau}{2}\right)^{d}\right)^{\nval} \leq \exp\left(-\frac{c}{2^d}{\nval} \spherevolume\left(\frac{\tau}{2}\right)^{d}\right).\label{eqn:exp-inequality}
    \end{align}
    Combining \cref{eqn:prob-bdd,eqn:cov-num-bdd,eqn:exp-inequality},
    \begin{align}
        \Pr(\zeta(\Sval_{1:\nval}, [0,1]^d) > \tau) \leq \frac{1}{\spherevolume}\left(\frac{6}{\tau}\right)^d\exp\left(-\frac{c}{2^d}{\nval} \spherevolume\left(\frac{\tau}{2}\right)^{d}\right)
    \end{align}
    Now choose $\tau^d = \frac{4^d}{c\nval \spherevolume}\log \frac{6^d\nval}{\spherevolume\delta}$.
    For all $\nval$ larger than some $n_0$, $\tau < 1$ because $\lim_{\nval \to \infty} \frac{4^d}{c\nval \spherevolume}\log \frac{6^d\nval}{\spherevolume\delta} = 0$, and so this choice satisfies our earlier assumption.
    For this choice of $\tau$
    \begin{align}
    \frac{1}{\spherevolume}\left(\frac{6}{\tau}\right)^d\exp\left(-\frac{c}{2^d}{\nval} \spherevolume\left(\frac{\tau}{2}\right)^{d}\right) = \frac{1}{\tau^d\nval}\delta & = \frac{c\spherevolume\delta}{4^d\log
            \frac{6^d \nval}{\spherevolume\delta}} \leq \delta \frac{c}{2^d\log
            \frac{3^d \nval}{\delta}},
    \end{align}
    where in the final inequality we use that $\spherevolume \leq 2^d$ since the unit ball is contained in the unit cube. For all $\nval \geq \frac{\delta e^{c/2^d}}{3^d}$ the right-hand side is less than $\delta$; but $\delta < 1, c \leq 1$ and so this holds for all $\nval \geq 1$.
\end{proof}

\subsection{Proof of Inconsistency of Existing Methods}\label{app:inconsistency-claims}
In this section, we restate and prove our results on limitations of existing methods for validation. These results were stated in \cref{sec:inconsistent} in the main text.

\subsubsection{Inconsistency of the \Holdout{}}\label{app:holdout-inconsistent-lipschitz}
We begin by focusing on the \holdout{}. As sketched in the main text, the \holdout{} does not depend on the particular test set, and therefore cannot approximate the test risk well for point prediction tasks unless the average loss at both points is the same. 

\counterexampleholdout*
We prove the strong result.
\begin{proposition}[Counterexample to the consistency of \holdout]\label{prop:holdout-inconsistent-app}
Consider a single test point $(S,X,Y)$ satisfying \cref{assumption:DGP}. Suppose the test risk at such a point (\cref{def:test-risk} with $\ntest = 1$) is not constant as a function of $S$. Then there exists a test set (of size one) such that $\smash{\sampleaveragerisk}$ is not a consistent estimator of the test risk on that test set. 
\end{proposition}
\begin{proof}
    Because $\averageloss(S)$ is not a constant $\spatialdomain$ contains at least two elements $S, S'$ such that $\averageloss(S) \neq \averageloss(S')$. Let $Q = \delta_S$ and $Q' = \delta_{S'}$. There there exists some $\gamma > 0$ such that 
    \begin{align}
        |\conditionalrisk_Q(h) -\conditionalrisk_{Q'}(h)|  = |\averageloss(S)- \averageloss(S')| > 2\gamma. 
    \end{align}
    If  $\sampleaveragerisk$ is not consistent for $\conditionalrisk_Q(h)$, we are done. Otherwise, by the definition of consistency, for all $N \geq N_0$ with probability at least $1/2$, 
    \begin{align}
        |(\sampleaveragerisk)^{(N)}(h) -\conditionalrisk_{Q}(h)| < \gamma . \label{eqn:holdout-close-toQ}
    \end{align}
    where we use $(\sampleaveragerisk)^{(N)}$ to denote the estimator constructed using the first $N$ validation points, $\Dval_N$. By the reverse triangle inequality,
    \begin{align}
        |(\sampleaveragerisk)^{(N)}(h) - \conditionalrisk_{Q'}(h)|  \geq |\conditionalrisk_Q(h) -\conditionalrisk_{Q'}(h)|  - |(\sampleaveragerisk)^{(N)}(h) -\conditionalrisk_{Q'}(h)|.\label{eqn:reverse-triangle}
    \end{align}
    Combining \cref{eqn:holdout-close-toQ} and \cref{eqn:reverse-triangle}, for all $N \geq N_0$  with probability at least $1/2$ 
    \begin{align}
        |(\sampleaveragerisk)^{(N)}(h) - \conditionalrisk_{Q'}(h)| > 2\gamma  - \gamma  = \gamma ,
    \end{align}
    which implies the \holdout{} is not consistent for $\conditionalrisk_{Q'}(h)$.
\end{proof}

\begin{proposition}\label{prop:holdout-max-error}
    Let $\ell(a,b) = \max(1, |a-b|)$. There exists a data-generating process, test set containing a single site and predictive method satisfying infill asymptotics such that the \holdout{} converges in probability to $0$, while $\conditionalrisk_{\Qtest} = 1$.
\end{proposition}
\begin{proof}
    Consider $\spatialdomain=[0,1]$, $h\equiv 0$, $Y = S$, $\Stest = 1$ and 
    \begin{align}
        \Sval_m &= \begin{cases}
                    U_m & m \text{\,is prime}, \\
                    0 & m \text{\,otherwise}.
                    \end{cases}
    \end{align}
    with $U_m$ independent and identically distributed uniform variables. By the infinitude of primes, there are infinitely many $m$ such that $\Sval_m$ is uniformly distributed and, for example by \cref{prop:iid-implies-infill}, this implies that $\Sval_m$ satisfies the infill assumption. On the other hand, by the prime number theorem, as $\ntest \to \infty$, the density of primes in the natural numbers tends to $0$, and so
    \begin{align}
        \sampleaveragerisk(h) = \frac{1}{\ntest}\sum_{m=1}^{\ntest} \Sval_m = \frac{1}{\ntest}\sum_{\substack{m=1 \\ \text{prime}}}^{\ntest} U_m \leq \frac{|\{ m : 1\leq m \leq \ntest, m\text{\, prime}\}|}{\ntest} \to 0,
    \end{align}
    and $\conditionalrisk_{\Qtest}(h) = 1$.
\end{proof}

\subsubsection{Inconsistency of $1$-Nearest Neighbor Estimator}\label{app:inconsistency-1nn}
We now turn to $1$-nearest neighbor risk estimation and restate and prove \cref{prop:counterexample-1nn}:
\counterexampleonenn*
We again actually prove a strong result
\begin{proposition}\label{prop:counterexample-1nn-app}
    Assume any test point satisfies \cref{assumption:DGP}.
Assume there exists a constant $c > 0$ such that for any test point $(S, X, Y)$,  $\mathrm{Var}[\ell(Y, h^{\spatialfield}(S)) | S] \geq c$. Next, consider an infinite sequence of validation sets as in \cref{def:consistency}.
Suppose there exists an $S' \in \spatialdomain$ such that for any $r > 0$ and $\nval > 0$, $|\{1 \leq j \leq \nval : \spatialmetric(\Sval_j, S') = r\}| \leq 1$. Choose any $\Qtest$ such that $\Qtest(\{S'\}) > 0$.
Then there exists a $\delta \in (3/4,1)$ and a $C^{(\delta)}> 0$ such that, for each $\nval$, with probability at least $1-\delta$,
$|\conditionalrisk_{\Qtest}(h) \!-\!(\onennestimator)^{(\nval)}(h)| \geq C^{(\delta)}.$
Here $(\onennestimator)^{(\nval)}$ denotes the 1NN estimator associated to the first $\nval$ data points and $\delta$. $C^{(\delta)}$ and $\delta$ do not depend on $\nval$ or other properties of the sequence of validation data.
\end{proposition}
The technical condition $|\{1 \leq n \leq \nval : \spatialmetric(\Sval_j, S') = r\}| \leq 1$ ensures that the 1NN set for each point contains exactly $1$ point. That is, there are no ties. This condition would be satisfied for the infinite sequence of validation sets with probability one if, for instance, the validation points were chosen i.i.d.\ from a uniform measure on a compact set; cf.\ \cref{prop:iid-implies-infill}. Alternatively, it can be removed if any form of tie-breaking that selects a single nearest neighbor is used in defining the estimator.

The idea of the proof is that if the loss has a non-vanishing variance, then the one nearest neighbor procedure results in an estimator with a non-zero variance for point prediction. Therefore, it cannot converge in probability to the actual risk, which is deterministic.

\paragraph{Preliminary Result}
We begin by proving a result that says that if a bounded random variable has a second moment bounded below by $C$, then it cannot be close to zero most of the time. We will apply this inequality to the second moment of the difference between the $1$-nearest neighbor estimator and the test risk to in our proof of \cref{prop:counterexample-1nn}.

\begin{proposition}\label{prop:2nd-moment-bdd}
    Let $U$ be a random variable with $U \in [-A, A]$ almost surely and $\EE [U^2] \geq C > 0$. Then for any $\delta \in (1-\frac{C}{A^2}, 1)$ with probability $1-\delta$
    \begin{align}
        |U| \geq A\sqrt{1-\frac{1-\frac{C}{A^2}}{\delta}}.
    \end{align}
\end{proposition}
\begin{proof}
    Because $U \in [-A, A]$, $A^2 -U^2$ is a non-negative random variable. Applying Markov's inequality, for any $t > 0$, 
    \begin{align}
        \mathrm{Pr}(A^2 -U^2 \geq t) \leq \frac{A^2 - \EE[U^2]}{t} \leq  \frac{A^2 - C}{t}. \label{eqn:init-markov}
    \end{align}
    Take $t= \frac{A^2- C}{\delta}$ which is greater than $0$ because $\delta \in (1-\frac{C}{A^2}, 1)$. Then \cref{eqn:init-markov} becomes,
    \begin{align}
        \mathrm{Pr}\left(A^2 -U^2 \geq \frac{A^2- C}{\delta}\right) \leq \delta.
    \end{align}
    Taking complements, with probability at least $1-\delta$
    \begin{align}
        A^2 -U^2 < \frac{A^2- C}{\delta}.
    \end{align}
    Rearranging implies that with probability at least $1-\delta$
    \begin{align}
        U^2 > A^2 - \frac{A^2- C}{\delta}.
    \end{align}
    For $\delta \in (1-\frac{C}{A^2}, 1)$ this bound is non-vacuous (strictly greater than zero). Taking square roots, which is monotone, for any such $\delta$ with probability at least $1-\delta$ 
    \begin{align}
        |U| > A\sqrt{1 - \frac{1-\frac{C}{A^2}}{\delta}}.
    \end{align}
\end{proof}

\paragraph{Proof of Inconsistency of $1$-nearest neighbor}
We return to our proof of \cref{prop:counterexample-1nn}. The idea will be to consider a point prediction task and then apply \cref{prop:2nd-moment-bdd} to show that with some fixed probability, the $1$-nearest neighbor estimator is a fixed distance away from the test risk, even as $\nval$ increases.

\begin{proof}[Proof of \cref{prop:counterexample-1nn}]
Because expectation minimizes the squared error to a random variable over all constant functions
    \begin{align}
        \EE|\conditionalrisk_{\Qtest}(h) - (\onennestimator)^{(\nval)}(h)|^2 
           \geq \EE| (\onennestimator)^{(\nval)}(h) - \EE[  (\onennestimator)^{(\nval)}(h)]|^2.\label{eqn:sq-error-variance}
    \end{align}
Because the $\epsilon^{\textup{val}}_n$ are independent the variance of $(\onennestimator)^{(\nval)}$ is additive
\begin{align}
    \EE| (\onennestimator)^{(\nval)}(h) - \EE  (\onennestimator)^{(\nval)}(h)|^2 & = \sum_{n=1}^{\nval} (\onennweights_n)^2 \EE[\ell(\Yval_n, h^{\spatialfield}(\Sval_n)) -\averageloss(\Sval_n)] \\
    &\geq V \sum_{n=1}^{\nval} (\onennweights_n)^2. \label{eqn:lowerbdd-variance}
\end{align}
where $0 < V < \lossbdd^2/4$ is the assumed lower bound on the variance of $\ell(\Yval_n, h^{\spatialfield}(\Sval_n)) -\averageloss(\Sval_n)$ and we have left implicit the dependence of the weights on $\nval$. Also, $|\{1 \leq n \leq \nval: \spatialmetric(S', \Sval_n) = r| \leq 1$, implies $S'$ has exactly one $1$-nearest neighbor in $\Sval_{1:\nval}$, call the index of this neighbor $n(S')$. Then
\begin{align}
   \sum_{n=1}^{\nval} (\onennweights_n)^2  \geq (\onennweights_{n(S')})^2 \geq  \Qtest(\{S'\})^2. \label{eqn:lowerbdd-sq-weights}
\end{align}
Combining \cref{eqn:lowerbdd-variance} and \cref{eqn:lowerbdd-sq-weights}
\begin{align}
    \EE[|\conditionalrisk_{\Qtest}(h) - (\onennestimator)^{(\nval)}(h)|^2]  \geq V\Qtest(\{S'\})^2.
\end{align}
We now apply \cref{prop:2nd-moment-bdd} with $U = |\conditionalrisk_{\Qtest}(h) - (\onennestimator)^{(\nval)}(h)|$ to conclude that for $\delta \in (1-\frac{V\Qtest(\{S'\})^2}{\lossbdd^2}, 1)$ with probability at least $1-\delta$
\begin{align}
    |\conditionalrisk_{\Qtest}(h) - (\onennestimator)^{(\nval)}(h)| \geq  \lossbdd\sqrt{1 - \frac{1-\frac{V\Qtest(\{S'\})^2}{\lossbdd^2}}{\delta}} > 0. \label{eqn:high-prob-lower-bdd}
\end{align}
As neither $\delta$ nor the right hand side of \cref{eqn:high-prob-lower-bdd} depend on $\nval$, $1$-nearest neighbor is not consistent under infill asymptotics.
\end{proof}

\begin{proposition}\label{prop:1nn-classification}
    Consider $\responsedomain = \{0,1\}$ and $\ell(a,b) = \begin{cases}
    0 & a= b \\
    1 & \text{otherwise}
    \end{cases}$, $\spatialdomain = [0,1]^d$ and $\Sval_m \iidsim \mu$, where $\mu$ is any measure with density with respect to Lebesgue measure. Fix any $S\in \spatialdomain$ and any predictive method $h$ and let $\Qtest = \delta_S$. Then,
    \begin{align}
        |\onennestimator(h) - \conditionalrisk_{\Qtest}(h) | \geq \min(\mathbb{E}[Y^{\textrm{test}}], 1 - \mathbb{E}[Y^{\textrm{test}}]).
    \end{align}
    In particular, if $\mathbb{E}[Y^{\textrm{test}}] = 1/2$, then one-nearest neighbor risk estimation has error $1/2$.
\end{proposition}
\begin{proof}
    Because $\Sval_m \iidsim \mu$, and $\mu$ has Lebesgue density, the nearest neighbor to $S$ is almost surely unique. This implies that $\onennestimator(h) \in \{0,1\}$. Therefore,
    \begin{align}
        |\onennestimator(h) - \conditionalrisk_{\Qtest}(h) | \geq \min_{a \in \{0,1\}} |a - \mathbb{E}[Y^{\textrm{test}}] | = \min(\mathbb{E}[Y^{\textrm{test}}], 1 - \mathbb{E}[Y^{\textrm{test}}]).
    \end{align}
\end{proof}

\subsubsection{Inconsistency of Nearest Neighbors with Number of Neighbors Depending On Number of Validation Points}\label{app:inconsistency-k-sequence}
We now restate and prove \cref{prop:counterexample-k-sequence}, which states that nearest-neighbor risk estimation with the number of neighbors depending (only) on the number of validation points is inconsistent under infill asymptotics, regardless of type of dependence.

\counterexamplekenn*

\begin{proof}
    The idea is that either 1.) $(k_n)_{n=1}^{\infty}$ has a bounded sub-sequence, in which case along this sub-sequence, the  $\kseq(h)$ can have a variance bounded below by $0$, and so by \cref{prop:2nd-moment-bdd} these estimators are bounded away from the test risk with fixed probability. Or 2.) the number of neighbors used tends to infinity, in which case we can find a sequence of data that accumulates more slowly around the test site, leading to many neighbors far from the point being used in the estimator, and therefore non-negligible bias. 
    
    We split into these two cases, and give an example showing in either case $\kseq$ can be inconsistent.
    \paragraph{Case 1: $\lim \inf_{n \to \infty} k_n < \infty$. \\} 
    Consider a data-generating process with no covariates, $\Sval \sim U(0,1)$, $\Stest = \{\frac{1}{2}\}$, $Y_n|S_n = \epsilon_n \sim U(-1/2,1/2)$, $h=0$ and $\ell(y,y') = |y-y'|$. Because $\lim \inf_{n \to \infty} k_n < \infty$, there exists a $C>0$ such that $(k_n)_{n=1}^{\infty}$ contains a bounded sub-sequence $(\tilde{k}_n)_{n=1}^{\infty}$ with $\tilde{k}_n \leq C$ for all $C>0$. Since the limit superior of a sub-sequence cannot be larger than of the full sequence
    \begin{align}
        \lim \sup_{\nval \to \infty} \mathrm{Var}(\kseq(h)) \geq \lim \sup_{n \to \infty} \mathrm{Var}(\ksubseq(h)).
    \end{align}
    where $\ksubseq$ denotes the sub-sequence of $\kseq$ where the number of validation points runs along the sub-sequence corresponding to $(\tilde{k}_n)_{n=1}^\infty$. 
    
    Almost surely, for any $\nval$ the test point $1/2$ has exactly $k_{\nval}$-nearest neighbors, because the probability that two validation points are equidistant from $1/2$ is $0$. Since a countable union of almost sure events is also an almost sure event, with probability $1$ for all $\nval$ the test point at $1/2$ has exactly $k_{\nval}$ neighbors. Therefore, with probability $1$, for all $\nval$ the vector of weights $\nnweights$ has exactly $k_{\nval}$ non-zero entries, each with value $1/k_{\nval}$. We condition on this probability $1$ event moving forward. In this case,
    \begin{align}
        \mathrm{Var}(\kseq) &= \mathrm{Var}\Big(\sum_{n=1}^{\nval} \nnweights |\epsilon_n|\Big) \\
        &= \frac{1}{k_{\nval}^2}\sum_{i=1}^{k_{\nval}}\mathrm{Var}(|\tilde{\epsilon}_i|) \\
        &= \frac{1}{48 k_{\nval}}
    \end{align}
    where $(\tilde{\epsilon})_{i=1}^{k_{\nval}}$ are the subset of $(\epsilon_n)_{n=1}^{\nval}$ corresponding to the $k_{\nval}$ nearest neighbors to $1/2$. The factor of 48 comes from the variance of a uniform random variable on $[0, 1/2]$. 
    
    For all $\nval$ corresponding to the bounded sub-sequence $(\tilde{k}_n)_{n=1}^\infty$, we conclude
    \begin{align}
        \mathrm{Var}(\ksubseq(h)) \geq \frac{1}{48 C}. 
    \end{align}
    Therefore, 
    \begin{align}
               \lim \sup_{\nval \to \infty} \mathrm{Var}(\kseq(h)) \geq \frac{1}{48 C}.
    \end{align}
    Applying \cref{prop:2nd-moment-bdd} to the random variable $|\kseq(h) - \conditionalrisk(h)|$ we conclude there exists an $\epsilon, \delta$ such that with probability at least $1-\delta$,
    \begin{align}
        \lim\sup_{\nval \to \infty} |\kseq(h) - \conditionalrisk(h)| \geq \epsilon.
    \end{align}
    \paragraph{Case 2: $\lim \inf_{n \to \infty} k_n = \infty$. \\}

    Consider the data generating process, $Y_n|S_n = S_n$ on $[0,1]$ with $\Stest = \{0\}$ and $\ell(y,y') = |y-y'|$ and $h=0$. We then have $\conditionalrisk(h) = 0$. We will construct a sequence of validation sites $\Sval$ such that for each $\nval$ less than $k_{\nval}/2$ of the validation sites fall in the interval $[0, 1/4)$. For any such sequence (supposing such a sequence exists for the moment),
    \begin{align}
        |\kseq(h) - \conditionalrisk(h)| =|\kseq(h)|
         \geq \frac{1}{k_{\nval}} \cdot \frac{k_{\nval}}{2} \frac{1}{4} \geq \frac{1}{8}.
    \end{align}
 All that remains is to construct such a sequence that also satisfies infill asymptotics. Define the function $\psi: \mathbb{N} \to (0, 1)$ by 
    \begin{align}
        \psi(i) = \frac{i - 2^{\lfloor \log_2 i\rfloor + 1}}{2^{\lfloor \log_2 i\rfloor + 1}}.
    \end{align}
    This corresponds to the dyadic sequence $(1/2, 1/4, 3/4, 1/8, 3/8, 5/8, 7/8, 1/16, \dotsc)$. The essential properties of this function for our application is that the image of the function is dense on $(0,1)$ and for any $i$ at least $i/2$ of $(\phi(j))_{j=1}^{i}$ are in the interval $(0, 1/2]$. 
    
    \begin{algorithm}
    \caption{Algorithm defining $(\Sval_n)_{n=1}^\infty$}\label{alg:defn-sval}
    \begin{algorithmic}
    \WHILE{True:} 
        \IF{$j < k_n$ and $n - j > j$}
            \STATE $\Sval_n = \frac{1}{2} - \frac{1}{2}\psi(j)$
            \STATE $j \leftarrow j + 1$.
        \ELSE
            \STATE $\Sval_n = \frac{1}{2} + \frac{1}{2}\phi(n-j)$
            \STATE $n \leftarrow n + 1$.
        \ENDIF
    \ENDWHILE
    \end{algorithmic}
    \end{algorithm}
    The validation points are defined algorithmically via \cref{alg:defn-sval}. Because $k_n$ is unbounded, the first condition must be called infinitely often, and so $j$ eventually takes on all natural numbers in this loop. 
    
    Because $n - j > j$ each time the first condition is called, $n-j$ is incremented if and only if $j$ is not incremented $n-j$ also takes on all natural numbers in this loop. Therefore, $(\Sval_n)_{n=1}^{\infty}$ is a dense set in $[0,1]$ because $\frac{1}{2} - \frac{1}{2}\psi(\mathbb{N})$ is dense on $[0, 1/2]$ and $\frac{1}{2} + \frac{1}{2}\psi(\mathbb{N})$ is dense on $[1/2, 1]$. We conclude this sequence satisfies the infill asymptotics.
    
    For any $\nval$, the number of validation points less than $1/2$ is not more than $k_{\nval}$ by induction and using the first condition in the if statement. Of the points placed in $(0, 1/2)$ at most half of them are in $(0, 1/4]$, by our earlier observation that for any $i$ at least $i/2$ of $(\psi(j))_{j=1}^{i}$ are in the interval $(0, 1/2]$. Therefore, not more than $k_{\nval}/2$ points are in $[0, 1/4)$ for any $\nval$. 
\end{proof}
\subsubsection{Inconsistency of Blocked Spatial Validation}\label{app:spatial-blocking}

In what follows, we first review blocked spatial validation; then we prove it is inconsistent under infill asymptotics; and finally we provide illustrative examples that support its inconsistency (and the spatial consistency of our proposed method). The intuition is the same as in our previous \holdout{} counterexamples: blocked spatial validation is inconsistent because the set of locations used for validation could be substantially different from the set of locations where we want to apply the model (the test locations).

\paragraph{Blocked Spatial Validation.}
In blocked spatial validation, we have access to a collection of data, with spatial locations, that may be used for training and validation. All observations within a contiguous spatial area, typically a square (i.e., block), are reserved for validation. As an example, see the left plot in \cref{fig:spatial-blocking-realistic}, where we treat the horizontal and vertical dimensions as two spatial coordinates (e.g., latitude and longitude). We imagine that we originally have access to all of the training (pink circles) and validation (orange diamonds) data. We choose a particular spatial block (here, one sixteenth of the area of the interest) and reserve all of the data within it for validation. The estimate of test risk is the empirical average of the loss over the validation data. This estimate can be seen as a particular example of the \holdout{} estimator of test risk. While we have already seen that the \holdout{} need not be spatially consistent, our counterexample  above (\cref{prop:holdout-inconsistent}) was not a spatially blocked example, so it still remains to show blocked spatial validation is inconsistent, as we do below.

\paragraph{Existing Work on Blocked Spatial Validation.}
Blocked spatial validation is typically used as part of a cross-validation scheme, \emph{blocked spatial cross-validation} (also sometimes referred to as spatial block cross-validation or simply block cross-validation) \citep{burman1994cross,roberts2017cross}. In blocked spatial cross-validation, the region of interest is partitioned into disjoint regions. In a particular fold of spatial cross-validation, data within a single region are reserved for validation data, and the data in the remaining regions serve as the training data for the fold. For instance, we can imagine partitioning the square of spatial locations into sixteen blocks. \Cref{fig:spatial-blocking-realistic} can then be interpreted as a single fold of blocked spatial cross-validation. As in \cref{sec:inconsistent,app:val-vs-cv}, we emphasize that, while cross-validation can be useful for data efficiency, even in i.i.d.\ settings it essentially estimates a different notion of risk than validation does \citep{bates2023cv}. So we leave a careful study of blocked spatial cross-validation to future work and focus on evaluating blocked spatial validation presently.

Spatial (cross-)validation is commonly motivated by a desire to test the extrapolation quality of a prediction method \citep[Table 1]{roberts2017cross}. While the present paper does not examine the quality of validation methods for extrapolation tasks, we can still ask whether spatial validation passes the minimal check we propose in the present paper; that is, we ask whether spatial validation satisfies spatial consistency. In fact, we will shortly show it is spatially inconsistent.

\citet[Figure 1]{wadoux2021spatial} and \citet[Figure 9]{debruin2022dealing} have previously used simulations to show empirically that spatial cross-validation can give biased estimates of \emph{map accuracy} of a prediction method. \emph{Map accuracy} is the loss between observed and predicted values of the response evaluated over a grid of prediction points in a particular region. This is essentially the same as the test risk we consider when the test point locations form a grid. A difference between our test risk with test point locations on a regular grid and map accuracy is that we treat the response as a random variable and average over this randomness as well, whereas map accuracy treats the test response values as fixed. However, no work has previously established the spatial consistency properties of spatial validation.

There is an intuitive understanding in some of the literature on estimating map accuracy that there may be a mismatch between the validation and test locations. Several recent works have tried to address the mismatch in the cross-validation setting, where validation locations in each fold can be chosen \citep{mila2022nearest,Linnenbrink2023kNNDM}. (Conversely our only assumption in the present work is that validation locations arrive in a way that satisfies infill asymptotics.) These authors have suggested matching the distribution of pairwise distances between validation locations and the closest training locations to the distribution of pairwise distances between test locations and the closest training locations within each fold. We still expect these methods to exhibit bias in their estimates of test risk unless one is willing to make additional assumptions beyond the smoothness assumption in \cref{assumption:Lipschitz}. 

\paragraph{We Show Blocked Spatial Validation is Inconsistent under Infill Asymptotics.}
We will focus on cases where infill asymptotics is satisfied, and show that even in these particularly nice cases where validation data is available near the test point, the blocked spatial validation estimator need not give consistent estimates of the test risk. As in our other \holdout{} counterexamples, the intuition is that the validation locations could be meaningfully different from the test locations, and getting more validation locations need not alleviate this mismatch.

\counterexampleblocked*

We now construct such an example. In \cref{prop:blocked-example}, we show that for this example (which has random validation locations), the fill distance tends to zero in probability and the holdout converges in probability to a different value from the test risk. The predictive method described is linear, and is therefore Lipschitz. And the response is a deterministic, Lipschitz function. Therefore, \cref{assumption:Lipschitz} is also satisfied. Therefore, infill asymptotics is satisfied, and so \cref{prop:blocking-inconsistent} follows from \cref{prop:blocked-example}.

We prove \cref{prop:blocked-example} after describing the data-generating process, test distribution, and model that we will use.

\emph{Training Locations, Test Locations and Validation Locations.}
We generate the train locations on a $2$-dimensional regular grid on $[-0.5, 0] \times [-0.5,5]$ with $5$ points in the first dimension and $11$ points in the second dimension,
$S^{\text{train}} = \left\{-0.5, -0.4, -0.3, -0.2, -0.1 \right\} \times \left\{-0.5, -0.4, -0.3, -0.2, -0.1, 0, 0.1, 0.2, 0.3, 0.4, 0.5\right\}$. 
We generate the validation locations uniformly at random on $[0, 0.5] \times [-0.5,5]$. The test locations consist of a single point at $(0.15, 0)$.

\emph{Covariates and Response.}
We consider a data-generating process without additional covariates. We assume the response is generated as $Y_i^{j} = |(S_i^j)^{(1)}|_1 $, for $j \in \{\text{train}, \text{val}, \text{test}\}$. In our analysis, we assume that there is no noise for simplicity.

\emph{Predictive Model.}
We consider ordinary least squares fit on the spatial location using the training data. 

\emph{Blocked Spatial Validation Estimator.}
We use the holdout estimator in  \cref{eqn:sample-avg-definition}. The training and validation data are generated to be contained in disjoint regions in space. The estimator is therefore an example of blocked spatial validation.

\begin{proposition}\label{prop:blocked-example}
    For the data-generating process described above, the fill distance tends to $0$ in probability, and the holdout estimator tends to $2\int_{s=0}^{0.5} s \; ds = 0.5$ almost surely. On the other hand, the test risk is $0.3$.
\end{proposition}
\begin{proof}
    Because the training points all have first coordinate less than zero,  $Y^{\text{train}}_j = (S^{\text{train}}_j)^{(i)}$ for all $1 \leq j \leq 55$. Therefore, the predictive model is $\hat{Y}(S) = -S^{(1)}$, as this has zero squared error on the training data, and is the unique linear function with this property. We can then compute the test loss as $|\Ytest_1 - \hat{Y}((0.15, 0))| = 0.3$, proving the second part of the claim. By the strong law of large numbers, and because the holdout points are assumed to be independent and identically distributed, the holdout estimator converges almost surely to its expected value. We can calculate this expected value as,
    \begin{align}
        \int_{[0,0.5] \times [0,1]} (S^{(1)} - \hat{Y}(S))dS = 2 \int_{[0,0.5]} s \; ds = 0.5.
    \end{align}
    Finally, the fill distance tends to $0$ in probability by \cref{prop:iid-implies-infill}, as the uniform distribution has non-zero density in a neighborhood containing the test point.
\end{proof}

\paragraph{Illustrative Simulation of \Cref{prop:blocked-example}.}
We simulate the process described above for number of validation points $\nval \in \{250, 500, 1000, 2000, 4000, 8000\}$. For each number of validation points, we resample the validation locations, the only random part of the process described above, 100 times; since the response at the validation points varies deterministically with the location in this simulation, the validation responses will update when the locations update. We compare our method, the blocking estimator described above, and the standard holdout (using all available validation data). We expect our method to be consistent, because the second half of the data is uniformly distributed on a region containing the test data, and so the fill distance tends to $0$ with probability $1$ (\cref{prop:iid-implies-infill}). 
\begin{figure}
\includegraphics[width=0.5\textwidth]{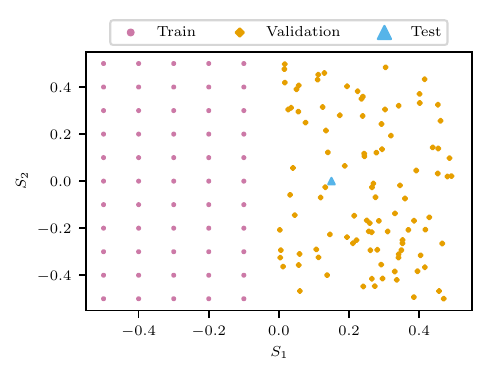}
\includegraphics[width=0.5\textwidth]{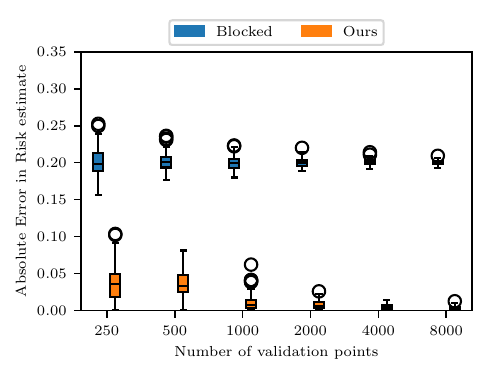}
\caption{On the left: the training (pink circles), validation (orange diamonds), and test (blue triangle) locations for the blocked spatial validation counterexample described above. On the right: A comparison of our method and the holdout estimator using the spatial block on the dataset described above.}\label{fig:blocked-validation-no-noise}
\end{figure}
\Cref{fig:blocked-validation-no-noise} shows the train, validation and test locations. The plot on the right in \cref{fig:blocked-validation-no-noise} supports the claim that the absolute error in risk estimation using the blocked holdout converges to $0.5 - 0.3 = 0.2$. And the error of our method is close to $0$ as infill asymptotics holds. 

\paragraph{Illustrative Simulation with More Complex Block Structure.}
The example above is particularly simple with a single validation block and a single training block. 
We also consider a second simulation that more closely resembles the use of blocked spatial validation in practice while still satisfying an infill assumption. We generate all data on the unit square. For each seed, we select a block with side length $0.25 \times 0.25$ that will only contain validation and test data. The remaining part of the unit square only contains test data so the data is spatially blocked into training and validation sets. The validation data are uniform on this sub-block. We sample the training data uniform on the remaining part of the unit square. We draw the test locations as a product of beta-distributions with parameter $0.5, 0.5$ on the sub-square containing the validation data. As a result, the test data is clustered close to the edges of the square containing validation data. We use 1000 train points and 500 test points. We vary the number of the validation points in $\{250, 500, 1000, 2000\}$. We generate the response as a Gaussian process on these locations with Mat\'ern 5/2 kernel with lengthscale 0.1, and Gaussian noise with standard deviation $0.1$. We use truncated mean absolute error loss, truncated at $2$. We fit kernel ridge regression as the predictive method, using the kernel used to generate the data and the variance ($0.01$) as the regularization parameter. The training, validation and test locations are shown in \cref{fig:spatial-blocking-realistic}, left. 

\begin{figure}
    \centering
    \includegraphics[width=0.4\linewidth]{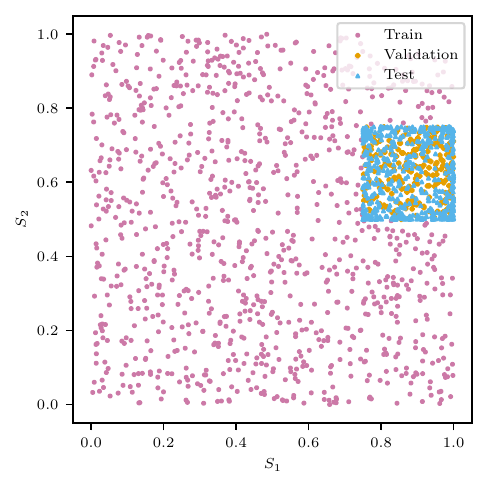}
     \includegraphics[width=0.55\linewidth]{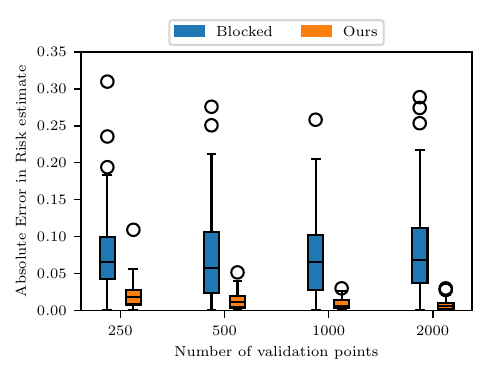}
    \caption{On the left: the training (pink circles), validation (orange diamonds) and test (blue triangle) locations for the blocked spatial validation counterexample described above. On the right: A comparison of our method and the holdout estimator using the spatial block on the dataset described above.}
    \label{fig:spatial-blocking-realistic}
\end{figure}

We resample the spatial locations, as well as the Gaussian process and noise at these spatial locations, as described above 100 times. In \cref{fig:spatial-blocking-realistic}, we compare the quality of estimation of blocked spatial validation and spatial nearest neighbors (ours) on these 100 seeds. We see that the holdout using blocking has substantially larger errors than our method (SNN). Also, while our method improves as the amount of validation data increases, as our theory suggests should happen since infill asymptotics is satisfied, the blocked spatial validation estimator using the holdout does not improve substantially even as we increase the amount of validation data available by an order of magnitude.

\subsection{General Nearest Neighbor Bound and Selecting the Number of Neighbors}\label{app:proof-nn-bound}
In this section, we prove two results giving bounds on the performance of nearest neighbor risk estimation that will be the basis of later results. The first, already stated in the main text, is a general bound for any $k$. 

\nnbound*

The second result we will show is specific to SNN and relates the error incurred by using $\kstar$ to the error of the minimizer of the bound from \cref{thm:nn-fill-bdd}:

\begin{restatable}[Minimization over Powers of Two]{proposition}{nnminimizationpowtwo}\label{prop:bound-minimization-powers-2}
      Let $T = \{1, \dots, \nval\}$ and $T_2 = \{2^{i}\}_{i=1}^{\lfloor \log_2 {\nval} \rfloor}$. Fix $\delta \in (0,1)$. Define $\kstar \in \arg \min_{k \in T_2} \rho_k + \seconstant{\delta}\|\nnweights\|_2 $ with $\seconstant{\delta}=\lossbdd\sqrt{\frac{1}{2}\log\frac{2}{\delta}}$. Define $C_{\lipschitzbdd} = \max(1, \lipschitzbdd)$. Under \cref{assumption:DGP,assumption:BoundedLoss,assumption:Lipschitz} with probability at least $1-\delta$
    \begin{align}
        | & \conditionalrisk_{\Qtest}(h) - \kstarnearestneighborrisk(h)|  \leq \sqrt{2}C_{\lipschitzbdd} \left(\min_{k \in T} \rho_k + \seconstant{\delta}\sqrt{\max_{1\leq n \leq \nval}\frac{\Qtest(\ball{\Sval_n}{\rho_k})}{k}}\right),
    \end{align}
    where $\ball{S}{r}$ denotes the ball of radius $r$ centered at $S$.
\end{restatable}

\subsubsection{Preliminary Result: Hoeffding's Inequality}
We begin by recalling Hoeffding's inequality, which will be used to control tail probabilities of the sum of the weighted losses being far from its expectation.

\begin{lemma}[{Hoeffding's Inequality, \citep[Theorem 2]{hoeffding1963probability}}]\label{lem:hoeffding}
    Let $(Z_i)_{i=1}^{\ell}$ be independent random variables and $(a_i)_{i=1}^{\ell}$ and $(b_i)_{i=1}^{\ell}$ be sequences of real numbers such that $a_i \leq Z_i \leq b_i$ almost surely. Then for all $t > 0$
    \begin{align}
        \mathrm{Pr}\left( \left\lvert \sum_{i=1}^{\ell} Z_i - \sum_{i=1}^{\ell}\EE[Z_i]\right\rvert \geq t \right) \leq 2 \exp\left(-\frac{2t^2}{\sum_{i=1}^{\ell} (b_i - a_i)^2}\right).
    \end{align}
    Equivalently, for any $\delta \in (0,1)$ with probability at least $1-\delta$
    \begin{align}
        \left\lvert \sum_{i=1}^{\ell} Z_i - \sum_{i=1}^{\ell}\EE[Z_i]\right\rvert \leq \|b-a\|_2 \sqrt{\frac{1}{2}\log\frac{2}{\delta}}.
    \end{align}
    where $a, b \in \RR^{\ell}$ have entries $a_i$ and $ b_i$ respectively.
\end{lemma}

\subsubsection{Proof of General Nearest Neighbor Risk Estimation Bound}
We again recall and prove \cref{thm:nn-fill-bdd}. The idea is to apply triangle inequality to split the error into a bias term and a sampling error term. The sampling error term is then controlled with \cref{lem:hoeffding} since the loss is bounded. The bias term is controlled using \cref{assumption:Lipschitz}.
\nnbound*
\begin{proof}
    By the triangle inequality,
    \begin{align}
        |\conditionalrisk_{\Qtest}(h) - \nearestneighborrisk(h)| & \leq \underbrace{\big|\conditionalrisk_{\Qtest}(h) - \sum_{n=1}^{\nval} \nnweights_n\averageloss(\Sval_n)\big|}_{\tau_1} \\ & + \underbrace{\Big\vert\sum_{n=1}^{\nval} \nnweights_n(\ell(f(\Sval_n, \spatialfield(\Sval_n), \epsilon^{\textup{val}}_n), h^{\spatialfield}(\Sval_n)) -  \averageloss(\Sval_n)\Big\vert}_{\tau_2}. \label{eqn:initial-triangle}
    \end{align}
    The first term, $\tau_1$ is a bias term, while the second term, $\tau_2$ is a sum of $\nval$ independent variables with expectation zero. Using \cref{assumption:BoundedLoss}, we apply Hoeffding's inequality (\cref{lem:hoeffding}) to bound $\tau_2$: for any $\delta \in (0,1)$ with probability at least $1-\delta$
    \begin{align}
        \tau_2 \leq \lossbdd \|\nnweights\|_2\sqrt{\frac{1}{2}\log\frac{2}{\delta}}. \label{eqn:t2-bdd}
    \end{align}
    By H\"older's inequality and because the weights are non-negative and sum to one,
    \begin{align}
        \|\nnweights\|_2 \leq \sqrt{\|\nnweights\|_1 \|\nnweights\|_\infty} = \sqrt{\max_{1\leq n \leq \nval} \nnweights_n}. \label{eqn:holder-ineq}
    \end{align}
    Recalling the definition of $\nnweights$ (\cref{def:nn-estimator}) and that each $A^k(s)$ contains at least $k$ points by construction,
    \begin{align}
        \nnweights_n = \frac{1}{\ntest}\sum_{m=1}^{\ntest}\frac{1}{|A^k(\Stest_m)|} \mathbf{1}\{\Sval_n \in A^k(\Stest_m)\}\leq\frac{1}{k} \frac{1}{\ntest}\sum_{m=1}^{\ntest}\mathbf{1}\{\Sval_n \in A^k(\Stest_m)\}. \label{eqn:max-weight-bdd}
    \end{align}
    By the definition of $\rho_k$,
    \begin{align}
        \Sval_n \in A^k(\Stest_m) \Rightarrow \Stest_m \in \ball{\Sval_n}{\rho_k}
    \end{align}
    and so
    \begin{align}
        \mathbf{1}\{\Sval_n \in A^k(\Stest_m)\} \leq \mathbf{1}\{\Stest_m \in \ball{\Sval_n}{\rho_k}\}.
    \end{align}
    Substituting this in \cref{eqn:max-weight-bdd} and using \cref{eqn:holder-ineq}
    \begin{align}
        \|\nnweights\|_2 \leq \sqrt{\frac{1}{k}\max_{1\leq n \leq \nval} \frac{1}{\ntest}\sum_{m=1}^{\ntest}\mathbf{1}\{\Stest_m \in \ball{\Sval_n}{\rho_k}\}} = \sqrt{\max_{1 \leq n \leq \nval}\frac{\Qtest(\ball{\Sval_n}{\rho_k})}{k}}. \label{eqn:bdd-on-2norm}
    \end{align}

    It remains to bound the bias term, $\tau_1$. Define $\alpha_{nm}^k =  \frac{1}{|A^k(\Stest_m)|}\mathbf{1}\{\Sval_n \in A^k(\Stest_m)\}$. Recalling the definition of $\nnweights$ and rearranging the order of summation
    \begin{align}
        \tau_1 &= \left\lvert\frac{1}{\ntest}\sum_{m=1}^{\ntest} (\averageloss(\Stest_m) - \sum_{n=1}^{\nval} \alpha_{nm}^k\averageloss(\Sval_n) ) \right\rvert \\ &\leq \max_{1\leq m \leq \ntest} \left\lvert\averageloss(\Stest_m) - \sum_{n=1}^{\nval} \alpha_{nm}^k\averageloss(\Sval_n)\right\rvert.
    \end{align}
    Because for any $1 \leq m \leq \ntest$, $\sum_{n=1}^{\nval} \alpha_{nm}^k = 1$
    \begin{align}
        \left\lvert\averageloss(\Stest_m) - \sum_{n=1}^{\nval} \alpha_{mn}^k\averageloss(\Sval_n)\right\rvert & = \left\lvert\sum_{n=1}^{\nval} \alpha_{mn}^k(\averageloss(\Stest_m) - \averageloss(\Sval_n)) \right\rvert \\
        &\leq \max_{n: \alpha_{mn}^k >0} |\averageloss(\Stest_m) - \averageloss(\Sval_n)|.
    \end{align}
    Applying \cref{assumption:Lipschitz} and taking the maximum over $m$ as well,
    \begin{align}
        \tau_1 \leq \max_{n, m: \alpha_{nm}^k >0} \lipschitzbdd \spatialmetric(\Stest_m, \Sval_n).
    \end{align}
    The constraint $\alpha_{nm}^k > 0$ implies that $\spatialmetric(\Stest_m, \Sval_n) \leq \rho_k$ and so
    \begin{align}
        \tau_1 \leq \lipschitzbdd  \rho_k. \label{eqn:t1-bdd}
    \end{align}
    The result follows from combining \cref{eqn:initial-triangle,eqn:t2-bdd,eqn:t1-bdd,eqn:bdd-on-2norm}.
\end{proof}

\subsubsection{Proofs Related to Selecting the Number of Neighbors}

We now restate and prove a bound that upper bounds the error of risk estimation with SNN (i.e.~using $\kstar$ neighbors) to the minimum of the upper bound from \cref{thm:nn-fill-bdd} over all $k$. A key observation is that because \cref{thm:nn-fill-bdd} is conditional on the test locations, it can be minimized without the need to take a union bound over all $k$ in the set we minimize over. We first need a preliminary result, which holds for minimization over any subset of $\{1, \dots, \nval\}$.

\begin{restatable}[Minimization of Upper Bound]{proposition}{nnminimization}\label{prop:bound-minimization}
    Let $T \subset \{1, \dots, \nval\}$. Fix $\delta \in (0,1)$. Define $\kstar[T] \in \arg \min_{k \in T} \rho_k + \seconstant{\delta}\|\nnweights\|_2 $ with $\seconstant{\delta}=\lossbdd\sqrt{\frac{1}{2}\log\frac{2}{\delta}}$. Define $C_{\lipschitzbdd} = \max(1, \lipschitzbdd)$. Under \cref{assumption:DGP,assumption:BoundedLoss,assumption:Lipschitz} with probability at least $1-\delta$
    \begin{align}
        |  \conditionalrisk_{\Qtest}(h) - \kstarnearestneighborrisk[T](h)| &\leq C_{\lipschitzbdd} \left(\min_{k \in T} \rho_k + \seconstant{\delta}\|\nnweights\|_2 \right) \label{eqn:first-ub-min} \\
          & \leq C_{\lipschitzbdd} \left(\min_{k \in T} \rho_k + \seconstant{\delta}\sqrt{\max_{1 \leq n \leq \nval}\frac{\Qtest(\ball{\Sval_n}{\rho_k})}{k}}\right),
    \end{align}
    where $\ball{S}{r}$ denotes the ball of radius $r$ centered at $S$.
\end{restatable}

\begin{proof}[Proof of \cref{prop:bound-minimization}]
    Because the minimization problem does not depend on a quantity that is treated as random in \cref{thm:nn-fill-bdd}, we may directly apply \cref{thm:nn-fill-bdd} to $\kstar[T]$ to conclude with probability at least $1-\delta$
    \begin{align}
        |\conditionalrisk(h) - \kstarnearestneighborrisk[T](h)|\leq \lipschitzbdd \rho_{\misspecifiedkminimizer[T]} + \lossbdd \|\kstarweights[T]\|_2 \sqrt{\frac{1}{2}\log\frac{2}{\delta}}\label{eqn:recall-bdd}
    \end{align}
    We split into cases.
    \paragraph{Case 1: $\lipschitzbdd \leq 1$}
    Because $\lipschitzbdd \leq 1$ and by the minimality of $\misspecifiedkminimizer[T]$,
    \begin{align}
        \lipschitzbdd \rho_{\misspecifiedkminimizer[T]} + \lossbdd\|\kstarweights[T]\|_2 \sqrt{\frac{1}{2}\log\frac{2}{\delta}} & \leq  \rho_{\misspecifiedkminimizer[T]} + \lossbdd\|\kstarweights[T]\|_2 \sqrt{\frac{1}{2}\log\frac{2}{\delta}} \\
        & = \min_{k \in T} \rho_{k} + \lossbdd\|\nnweights\|_2 \sqrt{\frac{1}{2}\log\frac{2}{\delta}}.\label{eqn:caseLlt1}
    \end{align}
    \paragraph{Case 2: $\lipschitzbdd > 1$}
    Because $\lipschitzbdd  > 1$ and the second term is non-negative,
    \begin{align}
        \lipschitzbdd\rho_{\misspecifiedkminimizer[T]} + \lossbdd\|\kstarweights[T]\|_2 \sqrt{\frac{1}{2}\log\frac{2}{\delta}} & \leq \lipschitzbdd \left(\rho_{\misspecifiedkminimizer[T]} + \lossbdd\|\kstarweights[T]\|_2 \sqrt{\frac{1}{2}\log\frac{2}{\delta}}\right) \\
         & =\lipschitzbdd  \min_{k \in T} \rho_{k} + \lossbdd\|\nnweights\|_2 \sqrt{\frac{1}{2}\log\frac{2}{\delta}}. \label{eqn:caseLgt1}
    \end{align}
    Combining \cref{eqn:caseLlt1,eqn:caseLgt1,eqn:recall-bdd} gives the result.
\end{proof}

\nnminimizationpowtwo*

\begin{proof}
    Let $\kstar[T]$ denote a minimizer of the bound on the right hand side, which exists since the minimization is over a finite set. If $\kstar[T] = 1$, we are done since $1 \in T_2$. Otherwise, there exists a $\tilde{k} \in T_2$ such that,
    \begin{align}
       \kstar[T]/ 2 \leq \tilde{k} \leq \kstar[T].
    \end{align}
    By monotonicity of the $k$\textsuperscript{th} order fill distance in $k$, 
    \begin{align}
        \rho_{\tilde{k}} \leq \rho_{\kstar[T]}.
    \end{align}
    This also implies,
    \begin{align}
        \Qtest(\ball{\Sval_n}{\rho_{\tilde{k}}}) \leq \Qtest(\ball{\Sval_n}{\rho_{\kstar[T]}}),  
    \end{align}
    since the measure of a subset is never larger than the measure of a set that contains it. Therefore,
    \begin{align}
        \min_{k \in T_2} \Bigg(\rho_k + \seconstant{\delta}&\sqrt{\max_{1\leq n \leq \nval} \frac{\Qtest(\ball{\Sval_n}{\rho_k})}{k}} \Bigg)  \leq  \rho_{\tilde{k}} + \seconstant{\delta}\sqrt{\max_{1\leq n \leq \nval}\frac{\Qtest(\ball{\Sval_n}{\rho_{\tilde{k}}}))}{\tilde{k}}} \\
        & \phantom{\qquad\qquad\qquad\qquad} \leq \rho_{\kstar[T]} + \seconstant{\delta}\sqrt{2}\sqrt{\max_{1\leq n \leq \nval}\frac{\Qtest(\ball{\Sval_n}{\rho_{\kstar[T]}})}{\kstar[T]}} \\
        & \phantom{\qquad\qquad\qquad\qquad} \leq \sqrt{2}(\rho_{\kstar[T]} + \seconstant{\delta}\sqrt{\max_{1\leq n \leq \nval}\frac{\Qtest(\ball{\Sval_n}{\rho_{\kstar[T]}})}{\kstar[T]}}) \\
        & \phantom{\qquad\qquad\qquad\qquad} = \sqrt{2}\left(\min_{k \in T} \rho_k + \seconstant{\delta}\sqrt{\max_{1\leq n \leq \nval}\frac{\Qtest(\ball{\Sval_n}{\rho_k})}{k}}\right).
    \end{align}
    The result now follows from \cref{prop:bound-minimization}.
\end{proof}

\subsection{Consistency of our Nearest Neighbor Method under Infill Asymptotics}\label{app:consistency-nn}
In this section we prove \cref{cor:prediction-consistent}, which establishes the consistency of SNN under infill asymptotics. The idea of the proof is to upper bound the $k$\textsuperscript{th} order fill distance of the validation points in the test points to the fill distance of the validation points in $[0,1]^d$. This allows together with \cref{thm:nn-fill-bdd} and \cref{prop:bound-minimization-powers-2} allows us to derive an upper bound on the error of our method in terms of the fill distance, from which the result follows. 

\subsubsection{Preliminary Lemma: Relating Fill Distances}
\begin{proposition}\label{prop:k-fill-bdd}
    Let $A$ be an $\epsilon$-net for $[0,1]^d$. Then for $k \leq \left(\frac{1}{2\epsilon}\right)^d$
    \begin{align}
        \zeta^k(A, [0,1]^d) \leq 2k^{1/d}\epsilon + \epsilon.
    \end{align}
\end{proposition}
\begin{proof}
    Let $S \in [0,1]^d$ and $\tau \in (0,1]$. Then,
    \begin{align}
        |A \cap \ball{S}{\tau + \epsilon}| \geq \coveringnumber{\epsilon}{\ball{S}{\tau}}
    \end{align}
    because for any $S' \in \ball{S}{\tau} \cap [0,1]^d$, there is a point $a \in A$ such that $d(S', a) \leq \epsilon$ and for any such point, this $a$ must also be in $\ball{s}{\tau + \epsilon}$ by the triangle inequality. Let $C$ be an $\epsilon$-net of $\ball{s}{\tau} \cap [0,1]^d$. Then by the definition of a net and by subadditivity
    \begin{align}
        \text{vol}(\ball{S}{\tau} \cap [0,1]^d) \leq \text{vol}(\cup_{c \in C} \ball{c}{\epsilon}) \leq |C| \spherevolume \frac{1}{\epsilon^{d}}
    \end{align}
    Since $S \in [0,1]^d$ and $\tau \leq 1$, at least one orthant of $\ball{s}{\tau}$ is contained in $[0,1]^d$, so
    \begin{align}
        \text{vol}(\ball{S}{\tau} \cap [0,1]^d) \geq 2^{-d}\spherevolume\frac{1}{\tau^d}.
    \end{align}
    Combining the previous estimates,
    \begin{align}
        \coveringnumber{\epsilon}{\ball{S}{\tau}\cap [0,1]^d} \geq \left(\frac{\tau}{2\epsilon}\right)^{d}
    \end{align}
    Choose $\tau = 2\epsilon k^{1/d} \in (0,1)$. Then
    \begin{align}
        |A \cap \ball{S}{\tau + \epsilon}| \geq k.
    \end{align}
    As $S$ was arbitrary, this holds for all $S \in [0,1]^d$, and so for all $S \in [0,1]^d$, there are $k$ points in $A$ in $\ball{S}{2\epsilon k^{1/d} + \epsilon}$.
\end{proof}

\subsubsection{Bound on Loss Depending on Fill Distance}
We now present and prove an upper bound on the error of SNN that depends on the fill distance of the validation set in $[0,1]^d$. We will derive consistency of the estimator under infill asymptotics as a corollary of this bound. The bound will also be relevant in later discussion of model selection, where we are also interested in rates of convergence of estimators.

\begin{restatable}[Bound for Dense Validation Data and General Test Data]{corollary}{generaldensebdd}\label{cor:general-dense-bdd}
    Suppose that $\spatialdomain = [0,1]^d$ and \cref{assumption:DGP,assumption:BoundedLoss,assumption:Lipschitz}. Let $\kstar \in \arg \min_{k \in T_2} \rho_k + \seconstant{\delta}\|\nnweights\|_2$ with $\seconstant{\delta} = \lossbdd\sqrt{\frac{1}{2}\log\frac{2}{\delta}}$. Then there exists a constant $K_{d, \delta, \lossbdd, \lipschitzbdd}$ such that with probability at least $1-\delta$
    \begin{align}
        |\conditionalrisk_{\Qtest}(h) - \kstarnearestneighborrisk(h)| \leq K_{d, \delta, \lossbdd, \lipschitzbdd} \tilde{\rho}^{\frac{d}{d+2}},
    \end{align}
    where $\tilde{\rho} =\zeta(\Sval_{1:\nval}, [0,1]^d)$.
\end{restatable}
\begin{proof}
    For all, $\tilde{\rho}\geq 1$, the stated bound holds with $K = \lipschitzbdd\sqrt{d} + \seconstant{\delta}$. Therefore, moving forward, we assume $\tilde{\rho} < 1$.

    Choose $k = \min (\lceil(\gamma \tilde{\rho})^{-\frac{2d}{d+2}}\rceil, \nval)$ for some $\gamma \in (0, 1/2) $ to be specified later. Because $\frac{2d}{d+2} < d$ and $\gamma \leq \frac{1}{2}$, this $k$ satisfies the conditions of \cref{prop:k-fill-bdd} so,
    \begin{align}
        \rho_k \leq \zeta^k(A, [0,1]^d) & \leq 2 \min (\lceil(\gamma \tilde{\rho})^{-\frac{2d}{d+2}}\rceil^{1/d}, (\nval)^{1/d})\tilde{\rho} + \tilde{\rho} \\
        &\leq 2 (1+\gamma \tilde{\rho})^{-\frac{2}{d+2}})\tilde{\rho} + \tilde{\rho} \\
        & \leq 4\gamma^{-2}\tilde{\rho}^{\frac{d}{d+2}}.
    \end{align}
    We now apply \cref{thm:nn-fill-bdd} and \cref{prop:bound-minimization-powers-2} together with this bound to conclude with probability $1-\delta$
    \begin{align}
    |\conditionalrisk_{\Qtest}(h) -\kstarnearestneighborrisk(h)| &\leq \sqrt{2}C_L(4\gamma^{-2}\tilde{\rho}^{\frac{d}{d+2}} + \frac{1}{\sqrt{k}}\seconstant{\delta})\\
    & \leq  \sqrt{2}C_L\left(4\gamma^{-2}\tilde{\rho}^{\frac{d}{d+2}} + \min\left(\gamma^{\frac{d}{d+2}}\tilde{\rho}^{\frac{d}{d+2}},\frac{1}{\sqrt{\nval}}\right)\seconstant{\delta}\right). \label{eqn:general-data-explicit-bdd}
    \end{align}
    Choosing $\gamma=\frac{1}{4}$ (for example) completes the proof.
\end{proof}

\subsubsection{Consistency of Spatial Nearest Neighbors}\label{app:consistency-of-snn}
We now restate and prove that the spatial nearest neighbor procedure we describe is consistent under infill asymptotics. This follows as a corollary of \cref{cor:general-dense-bdd} since the infill assumption means that for any fixed $\delta$, the upper bound in \cref{cor:general-dense-bdd} tends to zero with the fill distance.

\predictionconsistent*

\begin{proof}
Take $\delta = \min(1, r)$. Under infill asymptotics, this tends to zero because $r \leq C \tilde{\rho}$ and $\tilde{\rho}$ tends to $0$ by assumption. \Cref{cor:general-dense-bdd} (or more precisely \cref{eqn:general-data-explicit-bdd} which makes the depend of the bound on $\delta$ explicit), with probability at least $1-\delta$
\begin{align}
   |\conditionalrisk_{\Qtest}(h) -\kstarnearestneighborrisk(h)| &\leq K_{d, \lossbdd, \lipschitzbdd} \tilde{\rho}^{d/(d+2)} \sqrt{\log \frac{1}{r}} \\
   & \leq K_{d, \lossbdd, \lipschitzbdd} \tilde{\rho}^{d/(d+2)} \sqrt{\log \frac{1}{c\tilde{\rho}}}.
\end{align}
The left hand side of this bound tends to zero with $\tilde{\rho}$, and so $\kstarnearestneighborrisk(h)$ converges in probability to $\conditionalrisk_{\Qtest}(h)$.
\end{proof}
\subsubsection{Computation of An Approximate Fill Distance}\label{app:computation-approx-fill}
The fill distance can be computed exactly by computing the vertices of a $d$-dimensional Voronoi diagram with the validation points, then computing the maximum distance from each of these vertices to a point in the validation set. For problems in $1$ or $2$ dimensions, this is feasible even with a large number of validation points, because algorithms for computation of the Voronoi diagram can be done in nearly linear, $O(\nval \log \nval)$ time in $1$ and $2$ dimensions \citep[Chapter 4]{Okabe2000Spatial}. In higher dimensions, the worst case complexity of algorithms for computing Voronoi diagrams can be $O((\nval)^{\lfloor d/2\rfloor} + \nval \log \nval)$, and so for large number of points in more the computational cost can become quite high.

To avoid this computational cost, we instead use a simple space partitioning algorithm when $\spatialdomain= [0,1]^d$ that is guaranteed to give an approximation to the fill distance $r$ satisfying $ \frac{\tilde{\rho}}{2\sqrt{d}} \leq r \leq 2 \tilde{\rho}$. The idea is to split the domain into $2^d$ quadrants and check if each quadrant contains a validation point. If it does, recurse, otherwise stop and keep track of the side length of each quadrant, call it $r$. When the algorithm terminates it must be the case that there exists a partitioning of $[0,1]^d$ into cubes of side length $2r$, with each cube containing at least one validation point. Because this is a partition, each point in the spatial domain must also be within a cube, and so the fill distance is upper bounded by the maximum distance between two points in a cube of side length $2 r$, i.e.~$\tilde{\rho} \leq 2r\sqrt{d}$. Also, when the algorithm stops, a cube of side length $r$ has been found that does not contain any points. Therefore the fill distance is lower bounded by the distance from the center of this cube to the closest validation point, which must be at least $r/2$. That is, $\tilde{\rho} \geq r/2$. Rearranging, we see that 
\begin{align}
   \frac{\tilde{\rho}}{2\sqrt{d}} \leq r \leq 2 \tilde{\rho}
\end{align}
as claimed. Finally, we address the computational complexity of this approach. The number of times we recurse is $\log_2 r$, which is $O(\log 1/ \tilde{\rho})$, which is in turn $O(\log \nval)$, because the fill distance cannot decrease faster than the inverse of the covering number of $[0, 1]^d$, which certainly does not decrease faster than $1/\nval$.

It remains to consider the complexity of deciding which orthant all of the validation points lie in, and partioning the points by orthant. This can be done by looping over each of the dimensions, partitioning the points in the cube based on whether or not that coordinate is on the left or right hand side of the current cube, which is $O(\nval d)$. Therefore, the total computational complexity of this algorithm is not more than $O(\nval d\log \nval)$, which is nearly linear in $\nval$.

\subsection{Model Selection: Rates of Convergence of Spatial Nearest Neighbors}\label{app:model-selection-theory}
We now present an extended version of earlier discussion in \cref{app:model-selection}, as well as results on rates of convergence of spatial nearest neighbors that provides some support to claims regarding model selection. As a special case, we consider grid prediction (for example for assessing the global performance of a map constructed using a predictive method). We then discuss rates of convergence in the more general setting earlier addressed in \cref{cor:general-dense-bdd}.

\subsubsection{A Heuristic Discussion of Model Selection with Increasing Amounts of Data}
Suppose we have a fixed test task, and two data-driven algorithms for making predictions. We also suppose we have allocated a fixed percentage of data for training, and the remainder for validation. Can we use spatial nearest neighbors to select between the two method? As the amount of training data increases we would hope that both data-driven algorithms produce better predictive methods. We therefore do not expect consistency to be sufficient to select between the two sequences of predictive methods: if both are converging to the optimal estimator at (possibly different) rates as the amount of training data increases, we want our error in estimating the risk of the two methods to converge to zero faster at a rate faster in the amount of validation data than the rate that the slower converging method converges to the optimal predictor. This would suggest we should be able to reliably identify the better sequence of predictive methods as the amount of data increases. 

We will assume we are in the additive, homoskedastic error setting so that $Y = f(S, \chi(S)) + \epsilon$. We will assume that the noise is bounded, and that squared loss is used. Finally, we will assume that training data is also generated following this process. We will also assume that $f$ is Lipschitz continuous. In this setting, minimax pointwise regression rates are $\theta((\ntrain)^{-\frac{1}{d+2}})$ (cf.~\citet[Theorem 2.3, Corollary 2.2]{Tsybakov2008Introduction} in one-dimension under a fixed grid design and \citet[Examples 3.1, 3.2]{tibshirani2023minimax} for a multi-dimensional version with both fixed grid and random design. The latter assumes Gaussian noise instead of bounded noise). Ideally, we would like a method for performing model selection to be able to distinguish between a sequence of predictive methods converging at slower than the minimax optimal rate and a sequence of predictive methods converging at the minimax rate and to be able to reliably select between two sequences of predictive methods both of which are converging at the minimax rate, but with different constants. 

We will now present some finite sample and asymptotic bounds on the convergence of spatial nearest neighbors for a fixed hypothesis then return to the question of model selection in light of these results.

\subsubsection{Rates of Convergence for Grid Prediction}
We first consider the grid prediction task specifically, as this is a common problem in spatial analyses. For example, one might want to reconstruct air temperature across the continental United States on a dense grid (map) based on remotely sensed covariates observed on this grid and sparsely observed weather station data. We consider a test task grid prediction if the data falls on a regular $d$-dimensional grid.
\begin{assumption}[Grid prediction]\label{assumption:MapPrediction}
    We say that a task is \emph{grid prediction} if $\spatialdomain = [0,1]^d$ and $\Qtest = \frac{1}{\grid^d}\sum_{S \in \{i/\grid : 1 \leq i \leq \grid\}^d} \delta_S$ for some $\grid \in \NN$. 
\end{assumption}

As long as the resolution of the map is high, both $1$-nearest neighbor and spatial nearest neighbors provide reliable estimates of the error. 

\begin{restatable}[Bound on Estimation Error for Grid Prediction]{corollary}{mapprediction}\label{cor:map-fill-bdd}
    With the same assumptions as \cref{thm:nn-fill-bdd} and additionally \cref{assumption:MapPrediction}, with probability at least $1-\delta$
    \begin{align}
        \!\!\!|\conditionalrisk_{\Qtest}(h) \!-\!\onennestimator(h)| \!\leq\!   L\rho \!+\! \seconstant{\delta}\sqrt{\max(\tfrac{2^d}{\ntest}, (8\rho)^{d})}.
    \end{align}
    Also, with probability at least $1-\delta$
    \begin{align}
        \!\!|\conditionalrisk_{\Qtest}(h)\!-\! \!\kstarnearestneighborrisk(h)|\leq  \!C_L\!\!\left(\!\!\rho \!+\seconstant{\delta}\sqrt{\max(\tfrac{2^d}{\ntest}, (8\rho)^{d})}\right). \nonumber
    \end{align}
    with $\rho=\rho_1, \seconstant{\delta} = \lossbdd\sqrt{\frac{1}{2}\log\frac{2}{\delta}}$ and $C_{\lipschitzbdd} = \max(1, \lipschitzbdd)$.
\end{restatable}
See \cref{app:map-prediction-proofs} for a proof. The right-hand side of the bound is small as long as there is a validation point near each test point and the resolution of the map is high. If the available data for validation is generated i.i.d.~as in \cref{prop:iid-implies-infill}, then right hand side becomes, up to logarithmic factors in $\nval$, $\frac{1}{\sqrt{\ntest}} + (\nval)^{\min(-1/2, -1/d)}$ 

\subsubsection{Statement and Discussion of Result for IID validation data and Grid Prediction}\label{app:asymptotic-iid-map}
\begin{restatable}[Convergence of Grid Prediction with Independent and Identically Distributed Validation Data]{corollary}{mapiid}\label{cor:map-iid-bdd}
    Suppose that $\spatialdomain = [0,1]^d$, $\Sval_n \iidsim P$ for $1 \leq n \leq \nval$ with $\nval > 1$ and $P$ has Lebesgue density lower bounded by $c >0$. Additionally, take the assumptions of \cref{cor:map-fill-bdd}. Fix $\delta \in (0,1)$ and $k \in \{1, \kstar\}$. Then there exists a constant $K_{d, \delta, \lipschitzbdd, \lossbdd, c}$ that depends only on $d, \delta, \lipschitzbdd, c$ and $\lossbdd$ such that with probability at least $1-\delta$
    \begin{align}
        |\nearestneighborrisk(h) - \conditionalrisk_{\Qtest}(h)| \leq K_{d, \delta, \lipschitzbdd, \lossbdd, c}\left( (\tfrac{\log\nval}{\nval})^{\min(\frac{1}{2}, \frac{1}{d})}+\tfrac{1}{\sqrt{\ntest}}\right). \nonumber
    \end{align}
\end{restatable}

\begin{proof}
    From \cref{prop:iid-implies-infill}, there exists a constant $\gamma$ depending on $d, \lipschitzbdd, \delta, c, \lossbdd$ such that with probability at least $1-\delta/2$
    \begin{align}
        \rho \leq \gamma \left(\frac{\log\nval}{\nval}\right)^{1/d}. \label{eqn:fill-bdd-map}
    \end{align}
    An upper bound on the bound in \cref{cor:map-fill-bdd} shows that for $k \in \{1, \kstar\}$ with probability $1-\delta/2$,
    \begin{align}
        |\nearestneighborrisk(h) - \conditionalrisk_{\Qtest}(h)| \leq \gamma C_L \max(\seconstant{\delta/2} 8^{d/2}, 1)\left(\rho + \frac{1}{\sqrt{\ntest}} + \rho^{d/2} \right). \label{eqn:risk-bdd-const}
    \end{align}
    Combining \cref{eqn:fill-bdd-map} and \cref{eqn:risk-bdd-const} via a union bound and using that $a+b \leq 2\max(a,b)$ completes the proof.
\end{proof}

This matches the bound proven in \citet[Proposition 4]{portier2023scalable} which assumed that the test data was independent and identically distributed instead of on a regular grid. This rate of convergence is reasonably fast, particularly in low-dimensions. In particular ignoring the dependence of the bound on $\ntest$, which is reasonable as the number of test points in map prediction is often far larger than the number of available points for training and validation, it is faster than the minimax optimal rate of convergence for Lipschitz functions of $\theta((\ntrain)^{-1/(d+2)})$ discussed earlier. We therefore would expect both spatial nearest neighbors and $1$-nearest neighbors to perform well for model selection for grid prediction tasks if a fixed percentage of the data is used for training, and the remainder for validation. We emphasize we do not give a formal proof of this, just a heuristic argument suggesting why this should be the case. To give a formal proof would involve at least ensuring estimates of the risk estimation procedure hold uniformly over both sequences of predictive methods, and therefore involve additional assumptions.

\subsubsection{General Prediction Tasks}

For general $\Qtest$ we can combine \Cref{cor:general-dense-bdd} together with \cref{prop:iid-implies-infill} to get some sense of the rate of convergence of the spatial nearest neighbor method if the validation data is independent and identically distributed form a measure with density lower bounded on $[0,1]^d$ and the test task is fixed. 

\begin{restatable}[Convergence of Spatial Nearest Neighbor with Independent and Identically Distributed Validation Data]{corollary}{nonmapiid}\label{cor:dense-bdd-iid}
    Suppose that $\spatialdomain = [0,1]^d$, $\Sval_n \iidsim P$ for $1 \leq n \leq \nval$ with $\nval > 1$ and $P$ has Lebesgue density lower bounded by $c >0$. Additionally, take the assumptions of \cref{cor:general-dense-bdd}. Fix $\delta \in (0,1)$. Then there exists a constant $K_{d, \delta, \lipschitzbdd, \lossbdd, c}$ that depends only on $d, \delta, \lipschitzbdd, c$ and $\lossbdd$ such that with probability at least $1-\delta$
    \begin{align}
        |\nearestneighborrisk(h) - \kstarnearestneighborrisk(h)| \leq K_{d, \delta, \lipschitzbdd, \lossbdd, c}\left( (\tfrac{\log\nval}{\nval})\right)^{\frac{1}{d+2}}. 
    \end{align}
\end{restatable}
\begin{proof}
        From \cref{prop:iid-implies-infill}, there exists a constant $\gamma$ depending on $d, \lipschitzbdd, \delta, c, \lossbdd$ such that with probability at least $1-\delta/2$
    \begin{align}
        \tilde{\rho} \leq \gamma \left(\frac{\log\nval}{\nval}\right)^{1/d}. \label{eqn:fill-bdd-map-2}
    \end{align}
    \Cref{cor:general-dense-bdd} implies that there exists a constant $K \geq 0$ such that with probability at least $1-\delta/2$
    \begin{align}
         |\nearestneighborrisk(h) - \kstarnearestneighborrisk(h)|  \leq K \tilde{\rho}^{\frac{d}{d+2}} \label{eqn:general-nn-bound}.
    \end{align}
    Combining \cref{eqn:fill-bdd-map-2} and \cref{eqn:general-nn-bound} completes the proof.
\end{proof}

In this case, again up to logarithmic factors, \cref{cor:dense-bdd-iid} means that SNN converges at the optimal rate of convergence for Lipschitz functions. This means we do not necessarily expect to be able to distinguish between two sequences of predictive methods that converge at the minimax rate, but we might expect to distinguish between two sequences of predictive methods if one converges to the optimal predictor at much slower than the minimax rate. We again emphasize that we do not formally show this, and to do so would would involve at least ensuring estimates of the risk estimation procedure hold uniformly over both sequences of predictive methods, and therefore involve additional assumptions.

In contrast, both the \holdout{} and $1$-nearest neighbor methods are not even always consistent for risk estimation in this setting, and therefore cannot be expected to reliably perform model selection.

\subsubsection{Grid Prediction Proofs}\label{app:map-prediction-proofs}
In order to prove the claimed upper bound on grid prediction \cref{cor:map-fill-bdd} we use \cref{thm:nn-fill-bdd} together with an upper bound on the number of test points that lie within a ball of radius equal to the fill distance around any validation point. In order to do this, we will use that all the points in a grid are well-separated. We therefore begin by recalling the definition of a packing of a set, as well as a relationship between covering number and packing number.
\begin{definition}[Packing, Packing Number]
    Let $A \subset \RR^d$ a compact set. A (finite) set $B \subset A$ is called an $\epsilon$-packing of $A$ if for all $b, b' \in B$, $\|b-b'\| > \epsilon$. The $\epsilon$-packing number of a set $A$, $\packingnumber{\epsilon}{A}$ is the largest cardinality of an $\epsilon$-packing of $A$. 
\end{definition}
\begin{proposition}[{Packing and Covering Numbers \citealt[Lemma 5.5]{Wainwright2019High}}]\label{prop:packing-covering}
    For any $A\subset \RR^d$ and $\epsilon >0$,
    \begin{align}
        \packingnumber{2\epsilon}{A} \leq \coveringnumber{\epsilon}{A} \leq \packingnumber{\epsilon}{A}.
    \end{align}

\end{proposition}

We now restate and prove \cref{cor:map-fill-bdd}. 
\mapprediction*

\begin{proof}
    The second inequality follows from the first by \cref{prop:bound-minimization} and because $1 \in T_2$. We therefore focus on proving the case $k=1$.

    In light of \cref{thm:nn-fill-bdd}, it suffices to show that for these test location,
    \begin{align}
        \ntest \Qtest(\ball{\Sval_n}{\rho}) = |\ball{\Sval_n}{\rho} \cap \{a/\grid : 1 \leq a \leq \grid\}^d| \leq \max \left(2^d, 8^d\rho^d \ntest \right).
    \end{align}
    The set $\{a/\grid : 1 \leq a \leq \grid\}^d$ is a $\frac{1}{\grid+\epsilon}$-packing for any $\epsilon >0$, and so by the first inequality in \cref{prop:packing-covering} and \cref{prop:scale-translate-cov-num}
    \begin{align}
        |\ball{\Sval_n}{\rho} \cap \{a/\grid: 1 \leq a \leq \grid\}| \leq \packingnumber{\tfrac{1}{(\grid+\epsilon)}}{ \ball{\Sval_n}{\rho}} \leq \coveringnumber{\tfrac{1}{2\rho (\grid+\epsilon)}}{\ball{0}{1}}.
    \end{align}
    Applying \citet[Lemma 5.7, Equation 5.9]{Wainwright2019High} and taking the limit as $\epsilon \to 0^{+}$
    \begin{align}
        \lim_{\epsilon \to 0^{+}} \coveringnumber{\tfrac{1}{2(\grid+\epsilon)\rho}}{\ball{0}{1}} \leq  \lim_{\epsilon \to 0^{+}}(1+4(\grid+\epsilon)\rho)^d = (1+4\grid\rho)^d.
    \end{align}
    By the binomial theorem and bounding the sum by the number of terms times the largest term
    \begin{align}
        (1+4\grid \rho)^d \leq 2^d \max(1, 4^d\grid^d\rho^d) = \max(2^d, 8^d\ntest\rho^d).
    \end{align}
\end{proof}

\section{ADDITIONAL EXPERIMENTAL DETAILS}\label{app:experimental-details}
In this section, we provide additional details about the data, fitting procedures and validation procedures used in \cref{sec:experiments}. Code used in experiments is available anonymously at: \coderepository. Code is almost all implemented in Python3 \citep{vanRossum2009python} (with a small amount of r). Numpy is also heavily used for data generation and array manipulation \citep{harris2020array}.

In \cref{app:mc-estimate-test-risk} we give an overview of our method for estimating ground truth test risk in all experiments. In \cref{app:risk-estimation-synthetic} we describe details of the synthetic experiment described in \cref{sec:risk-estimation-synthetic}. In \cref{app:airTemp} we provide additional details on the air temperature data and tasks
in \cref{sec:air-temp-bootstrapped,sec:airtemp}. In \cref{app:UK-housing} we provide additional details on the UK flat price prediction experiment described in \cref{sec:uk-ppd}, while in \cref{app:windSpeed} we provide additional details for the wind speed prediction task presented in \cref{sec:wind_speed}.

\subsection{Monte Carlo Estimation of Ground Truth Test Risk}\label{app:mc-estimate-test-risk}

We would like to compute the exact test risk across the $\ntest$ test points:
\begin{align}
    \conditionalrisk_{\Qtest}(h) := (1/\ntest) \sum_{m=1}^{\ntest} \mathbb{E}[\ell(\Ytest_m, h^{\spatialfield}(\Stest_m)) \big| \Stest_m, \chi].
\end{align}

In all our examples where we report test risk, we have access to some sample $(\Ytest_m)_{m=1}^{\ntest}$ that we will use to construct an estimator. Our plan is to instead use the empirical test risk $\empconditionalrisk_{\Qtest}(h)$ as ground truth:
\begin{align}
    \hat{R}_{\Qtest}(h) := (1/\ntest) \sum_{m=1}^{\ntest} \ell(\Ytest_m, h^{\spatialfield}(\Stest_m)). \label{eqn:empirical-risk}
\end{align}

We would like to know how far off the empirical test risk is from the exact test risk. To that end, we observe that
\begin{align}
    \hat{R}_{\Qtest}(h) - R_{\Qtest}(h)
        & = (1/\ntest) \sum_{m=1}^{\ntest} Z_m, \quad \textrm{ where}                                                                \\
    Z_m & := \ell(\Ytest_m, h^{\spatialfield}(\Stest_m)) - \mathbb{E}[\ell(\Ytest_m, h^{\spatialfield}(\Stest_m)) | \Stest_m, \chi].
\end{align}

By construction, if we assume the expectations exist, each random variable $Z_m$ has mean zero. We make two additional assumptions. (1) We assume that the $Z_m$ are independent. (2) We assume that (almost surely) $\forall m, Z_m \in (a,b)$ for finite $a,b \in \mathbb{R}$. If the loss is bounded by $\Delta$, then such an $a, b$ exist satisfying $b-a \leq \Delta$; following \cref{assumption:BoundedLoss}, we use this bound moving forward.

Under these assumptions, we can apply Hoeffding's inequality to conclude that for any $\delta \in (0,1)$, with probability at least $1-\delta$,
\begin{align}
    |\hat{R}_{\Qtest}(h) - R_{\Qtest}(h)| \leq \Delta \sqrt{\frac{1}{2\ntest}\log \frac{2}{\delta}}. \label{eqn:test-risk-hoeffding}
\end{align}

If we are willing to make the two assumptions above, we next show that we can reach (high probability) conclusions about the (true) relative quality of different estimators on a particular task if they pass a check: namely, we check if, for a small $\delta$ (e.g.~$\delta=0.05$), the right-hand side of \cref{eqn:test-risk-hoeffding} is smaller than twice the difference between how much closer the ``good'' estimator is to the estimate of ground truth than the ``bad'' estimate. To see why this check is sufficient, first observe the following two applications of the triangle inequality:
\begin{align}
    |\text{good}-\text{true}|
     & \le |\text{good} - \widehat{\text{true}}| + |\text{true} - \widehat{\text{true}}| \\
    |\text{bad}-\widehat{\text{true}}|
     & \le |\text{bad}-\text{true}| + |\text{true}-\widehat{\text{true}}|.
\end{align}
Using these two inequalities, we can write
\begin{align}
    |\text{bad} - \text{true}| - |\text{good} - \text{true}|
     & \geq |\text{bad} - \widehat{\text{true}}| - |\text{good} - \widehat{\text{true}}| - 2|\text{true} - \widehat{\text{true}}|.
\end{align}
Therefore, to conclude
\begin{align}
    |\text{bad} - \text{true}| - |\text{good} - \text{true}| \geq 0,
\end{align}
it suffices for
\begin{align}
    \label{eq:estimators_vs_true_hat}
    |\text{bad} - \widehat{\text{true}}| - |\text{good} - \widehat{\text{true}}| \geq 2|\text{true} - \widehat{\text{true}}|.
\end{align}
Under the earlier assumptions, we see that \cref{eq:estimators_vs_true_hat} is implied (with high probability) by
\begin{align}
    |\text{bad} - \widehat{\text{true}}| - |\text{good} - \widehat{\text{true}}| \geq 2\Delta \sqrt{\frac{1}{2\ntest}\log \frac{2}{\delta}}.
\end{align}
When discussing each experiment, we discuss the plausibility of the assumptions needed to make this argument when justifying our estimated ground truth, as well as specific values for $\Delta$ and $\ntest$ and the resulting bound.

\subsection{Computational Considerations}\label{app:computational-considerations}

\subsubsection{Computational Complexity of our Method}\label{app:computational-complexity-1nn}
We focus on the case $\spatialdomain = [0,1]^d$ with nearest-neighbors implemented using a $k$-d tree. 

\begin{proposition}[Computational Complexity of Spatial Nearest Neighbors]\label{prop:comp-complexity-snn}
Suppose $\spatialdomain = [0,1]^d$ then the spatial nearest-neighbor estimator can be computed in $O(d\ntest\nval+ \ntest\nval\log \nval)$ algebraic operations.
\end{proposition}
\begin{proof}
We have already shown that computation of the approximate fill distance is  $O(d\nval\log \nval)$ to be in \cref{app:computation-approx-fill}. While we use $k$-d trees to implement nearest neighbors in practice, we consider the complexity of a direct implementation of nearest neighbors. In this case, computing all pairwise distances between test and validation points has complexity $O(d\ntest\nval)$.  Once these distances are computed, the distances for each validation point can be sorted, and the original indices can then be used to recover the nearest neighbors. The sorting operation has complexity $O(\ntest\nval\log \nval)$. Given the sorted list of distances and corresponding indices computing the weights needed for computing the bound for each $k$ points involves counting the number of times each training index appears in a sub-array consisting of the first $k$ columns. This can be done by scanning the array from left to right and keeping track of partial sums in $O(\nval\ntest)$. And given the weights, we can compute the bound for all of the $\log \nval$ many values of $k$ each in time $\nval$, giving a complexity of not more than $O(\nval \log \nval)$. Taking all of these steps together, we see the overall complexity is,
$
O(d\ntest\nval+ \ntest\nval\log \nval).
$
\end{proof}

\begin{figure}
    \centering
    \includegraphics[width=0.48\linewidth]{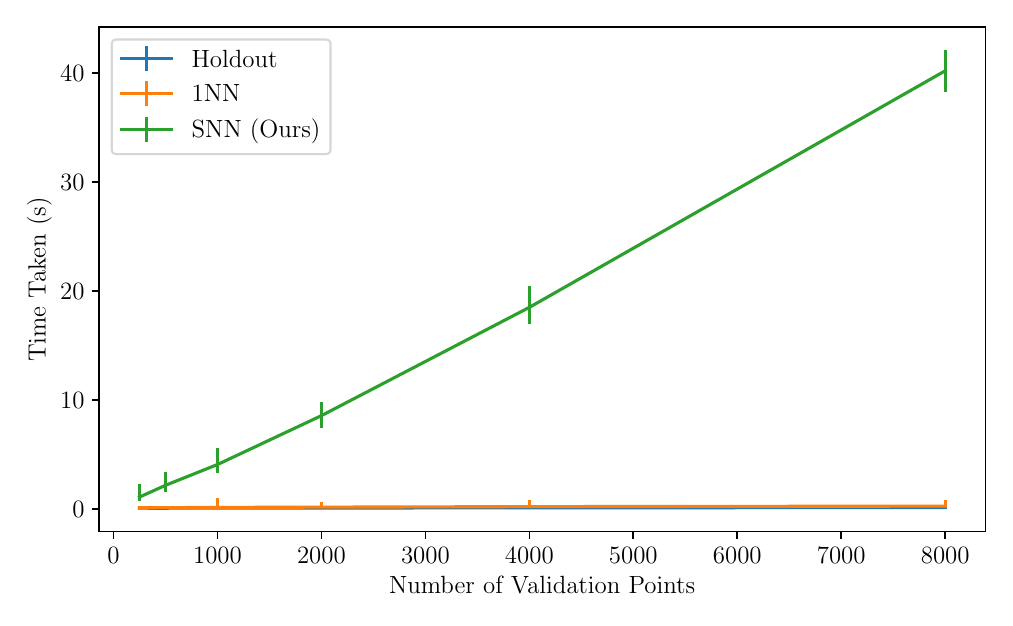}
    \includegraphics[width=0.48\linewidth]{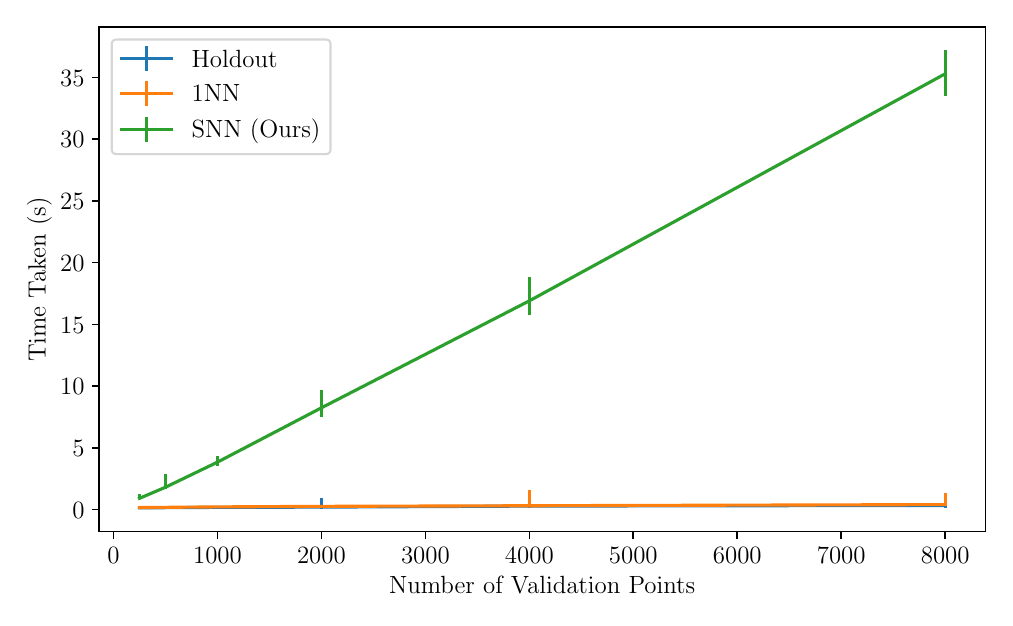}
    \caption{Computational time (in seconds) for computation of the three estimates of validation error in the synthetic experiment with data at a single test point (left) and on a regular grid (right). The mean is shown with the main line, while the maximum and minimum are indicated by the vertical bars. For both datasets, our method takes significantly longer to compute than the baselines But even with the largest number of validation points used, our method takes under a minute to compute.}
    \label{fig:validation-time-synthetic}
\end{figure}
\subsubsection{Computational Setup Used}\label{app:compuational-setup}
All experiments were run on a CPU cluster with 36 Intel(R) Xeon(R) W-2295 CPU @ 3.00GHz CPUs and a total of 251 GB of system RAM. In all experiments linear algebra operations operations were allowed to be multithreaded, and so at times all 36 CPUs were used, even if fewer than 36 parallel jobs were run.
\subsubsection{Computational Cost of Synthetic Experiment}\label{sec:computation-synthetic}

\paragraph{Data Generation.}
Generating the point prediction synthetic data takes around 25 minutes using 10 parallel jobs and has a peak memory usage of around 46GB.

Generating the grid prediction synthetic data takes around 33 minutes using 10 parallel jobs and has a peak memory usage of around 71GB.

\paragraph{Running Experiment.}
Running the point prediction task takes around 25 minutes using 10 parallel jobs and has a peak memory usage of around 32GB.

Running the grid prediction task takes around 27 minutes using 10 parallel jobs and has a peak memory usage of around 39GB.

\paragraph{Computational Cost of Validation Methods.}
\Cref{fig:validation-time-synthetic} shows the time taken to compute each estimator of the validation risk as a function of the number of validation points. Our method takes longer than baselines, but is still fast to compute in practice.

\subsubsection{Computational Cost of Bootstrapped Air Temperature Experiment}\label{sec:computation-bootstrapped-airtemp}

\paragraph{Data Generation.}

Running the make file to download air station data takes on the order of 30 seconds and not more than 6GB of RAM after some data has been installed manually as described in the README file in the released code. Fitting the model and computing residuals in order to generate bootstrapped datasets takes around 10.5 minutes and has peak memory usage around 95GB.

\paragraph{Running Experiment.}
Running the bootstrapped air temperature metro prediction task takes around 66 minutes with 20 parallel jobs and has peak memory usage around 225GB.

Running the bootstrapped air temperature grid prediction task takes around 8.5 hours with 3 parallel jobs and has peak memory usage around 150GB.

\paragraph{Computational Cost of Validation Methods.}
\emph{Metro Prediction:}
For geographically weighted regression, the holdout took a mean time of 8.66 seconds with standard deviation of 3.93 seconds. 1NN took a mean time of 7.98 with standard deviation 3.53 seconds. SNN took a mean time of 29.89 seconds with a standard deviation of 11.48 seconds. For kernel ridge regression the holdout took a mean time of 3.06 seconds with standard deviation 0.93 seconds. 1NN took a mean time of 3.35 seconds with a standard deviation of 0.94 seconds. SNN took a mean time of 19.87 seconds and standard deviation 3.73 seconds.

\emph{Grid Prediction:} For geographically weighted regression, the holdout took a mean time of 3.88 seconds with standard deviation of 0.60 seconds. 1NN took a mean time of 5.83 with standard deviation 0.59 seconds. SNN took a mean time of 171.49 seconds with a standard deviation of 2.75 seconds. For kernel ridge regression the holdout took a mean time of 0.72 seconds with standard deviation 0.02. 1NN took a mean time of 2.57 seconds and standard deviation of 0.026 seconds. SNN took a mean time of 169.91 seconds and standard deviation 3.77 seconds.

\subsubsection{Computational Cost of UK House Price Experiment}\label{sec:computation-ukppd}
\paragraph{Data Processing.}
Downloading and processing the data takes around 7 seconds and has a peak memory usage of under 3GB.

\paragraph{Running Experiment.}
Running the UK House price prediction task takes around 4.6 hours with 5 parallel jobs and has peak memory usage around 70GB.

\paragraph{Computational Cost of Validation Methods.}
The holdout took a mean time of 44.97 seconds with standard deviation 5.77. 1NN took a mean time of 43.37 seconds and standard deviation of 5.02 seconds. SNN took a mean time of 66.18 seconds and standard deviation 8.49 seconds.

\subsubsection{Computational Cost of Wind Speed}\label{sec:computation-windspeed}
\paragraph{Data Processing.}
Downloading and processing the data takes around 3.5 minutes and has a peak memory usage of under 12GB.

\paragraph{Running Experiment.}
Running the wind speed prediction task takes around 8.5 hours with 15 parallel jobs and has peak memory usage around 15GB.

\paragraph{Computational Cost of Validation Methods.}
The holdout took a mean time of 0.14 seconds with standard deviation 0.09. 1NN took a mean time of 0.50 seconds and standard deviation of 0.33 seconds. SNN took a mean time of 51.71 seconds and standard deviation 15.71 seconds.

\subsubsection{Computational Cost of Real Data Air Temperature Experiment}\label{sec:computation-airtemp}
Time to process the data has been previously described in \cref{sec:computation-bootstrapped-airtemp}.

\paragraph{Running Experiment.}
Running all the air temperature tasks takes around 10 minutes and has peak memory usage around 9GB.

\paragraph{Computational Cost of Validation Methods.}
\emph{Metro Prediction:} For geographically weighted regression, the holdout took 2.31 seconds. 1NN took 2.31. SNN took 2.32 seconds. For kernel ridge regression the holdout took 1.75 seconds. 1NN took 1.75 seconds. SNN took 1.75 seconds.

\emph{Grid Prediction:} For geographically weighted regression, the holdout took 2.30 seconds. 1NN took 3.54. SNN took 83.1 seconds. For kernel ridge regression the holdout took 1.70 seconds. 1NN took 3.00 seconds. SNN took 82.2 seconds.

\subsubsection{Computational Cost of Model Selection Experiment}\label{sec:computation-model-selection}

Generating data for and running the synthetic model selection experiment takes under a minute of time and under 4GB of RAM.

\subsection{Risk Estimation on Synthetic Data}\label{app:risk-estimation-synthetic}

In this section, we provide additional details and figures for our risk estimation experiment on synthetic data presented in \cref{sec:risk-estimation-synthetic}. \Cref{app:simulation-dgp} describes the process by which we generate both datasets considered and shows an example of the covariate and response spatial fields for each problem (\cref{fig:grid-prediction-data-0,fig:point-prediction-data-0}). \Cref{app:risk-estimation-model} describes the procedure used to fit the predictive methods to each dataset. \Cref{app:implementation-validation} describes the implementation of the risk estimation procedures we compare. \Cref{app:risk-estimation-metric} describes the metric reported, and \cref{fig:grid-rel-risk,fig:point-rel-risk} show the (signed) relative error of each estimator as we vary the amount of validation data available.

\subsubsection{Simulation Data Generating Process Details}\label{app:simulation-dgp}
For both tasks, 100 datasets are generated following the process outlined below.

\paragraph{Grid data}

\begin{figure}[t]    \centering\includegraphics[width=0.6\columnwidth]{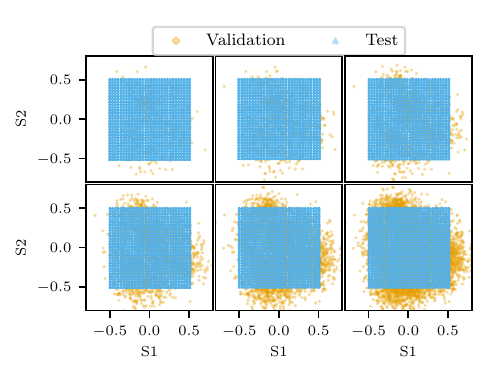}
    \caption{The validation sites (orange diamonds, clustered) for the first seed of the synthetic grid task. Panels from left to right and top to bottom represent $\nval$ in the sequence $(250, 500, 1000, 2000, 4000, 8000)$. Test sites (blue triangles, gridded) are constant across panels.}
    \label{fig:synthetic-sites}
\end{figure}

\emph{Generation of training and validation sites:} The first training point is selected via generating a point uniformly in $[-0.5, 0.5]^2$, and making this the mean of a Gaussian mixture component, with standard deviation randomly sampled between 0.05 and 0.15. This mixture is initially given weight $1$, the first training point is then sampled from a Gaussian with this mean and standard deviation, the and weight of this mixture is increased to $2$. Subsequent points are sampled sequentially. For each $i$ between $2$ and the total number of training and validation points, a weight of $1$ is assigned to adding a mixture component. The new point is then sampled from either one of the existing mixture components, or the new mixture component, with probability proportional to the current weights assigned to each mixture component. The weight of the mixture from which the points, $w^{(t)}_{i(t)}$, is then increased as
\begin{align}
    w^{(t+1)}_{i(t)}= w^{(t)}_{i(t)} + \frac{1}{w^{(t)}_{i(t)}}.
\end{align}
This is reminiscent of a Chinese restaurant process \citep[Section 3.1]{pitman2006combinatorial}, but the weights are increased more slowly, leading to more clusters being formed and less large clusters typically.

If a new mixture component is generated, a mean for the mixture component is generated on $[-0.5, 0.5]^2$, and a standard deviation is selected uniformly on $[0.05, 0.15]$. Conditional on the mixture component, the new point is sampled from a Gaussian distribution with the components mean and standard deviation.

The first $1000$ points generated this way are taken to be the training data, and the remaining $\nval$ points generated this way are the validation data. An example of the training and validation data generated through this process are shown in the top left of \cref{fig:grid-prediction-data-0}.

\emph{Generation of test data}
The test data is $\{(-0.5 + a/29,-0.5 + b/29) : 0 \leq a,b \leq 49\}$. That is, it is a regular grid on $[-0.5, 0.5]^2$. We generate 50 values of each response variable on each. grid point.

\emph{Generation of Covariates}

The covariates are generated as a zero-mean
Gaussian process with an isotropic Mat\'ern 3/2 covariance function with lengthscale $0.3$ and scale parameter $1$. That is, the covariance function is,
\begin{align}
    k_{\spatialfield}(S,S') = \left(1 + \frac{ \sqrt{3}\|S-S'\|_2}{0.3}\right)\exp\left(-\frac{\sqrt{3}\|S-S'\|_2}{0.3}\right).
\end{align}
A small diagonal term (1e-12) is added to the diagonal of the covariance matrix to avoid numerical linear algebra errors. Sampling is performed using Tensorflow probability \citep{dillon2017tensorflow}. We generate two covariates spatial processes via this process $\spatialfield=(\spatialfield^{(1)}, \spatialfield^{(2)})$.

\emph{Generation of Response}
Once the sites and covariates have been generated, the response variable is sampled from a Gaussian process with zero mean. The covariance function of the Gaussian is a sum of two, $2$ dimensional isotropic Mat\'ern 3/2 kernels:
\begin{align}
    k(S,S') & =  0.5\left(1 + \frac{ \sqrt{3}\|S-S'\|_2}{0.5}\right)\exp\!\left(-\frac{\sqrt{3}\|S-S'\|_2}{0.5}\right) \\ &+ \left(1+\sqrt{3}\|\spatialfield(S)-\spatialfield(S')\|_2\right)\exp\left(-\sqrt{3}\|\spatialfield(S)-\spatialfield(S')\|_2\right). \label{eqn:response-generation-kernel}
\end{align}
Independent, identically distributed Gaussian noise is added to the function values with variance $0.1$.

\begin{figure}
    \centering
    \includegraphics[width=\textwidth]{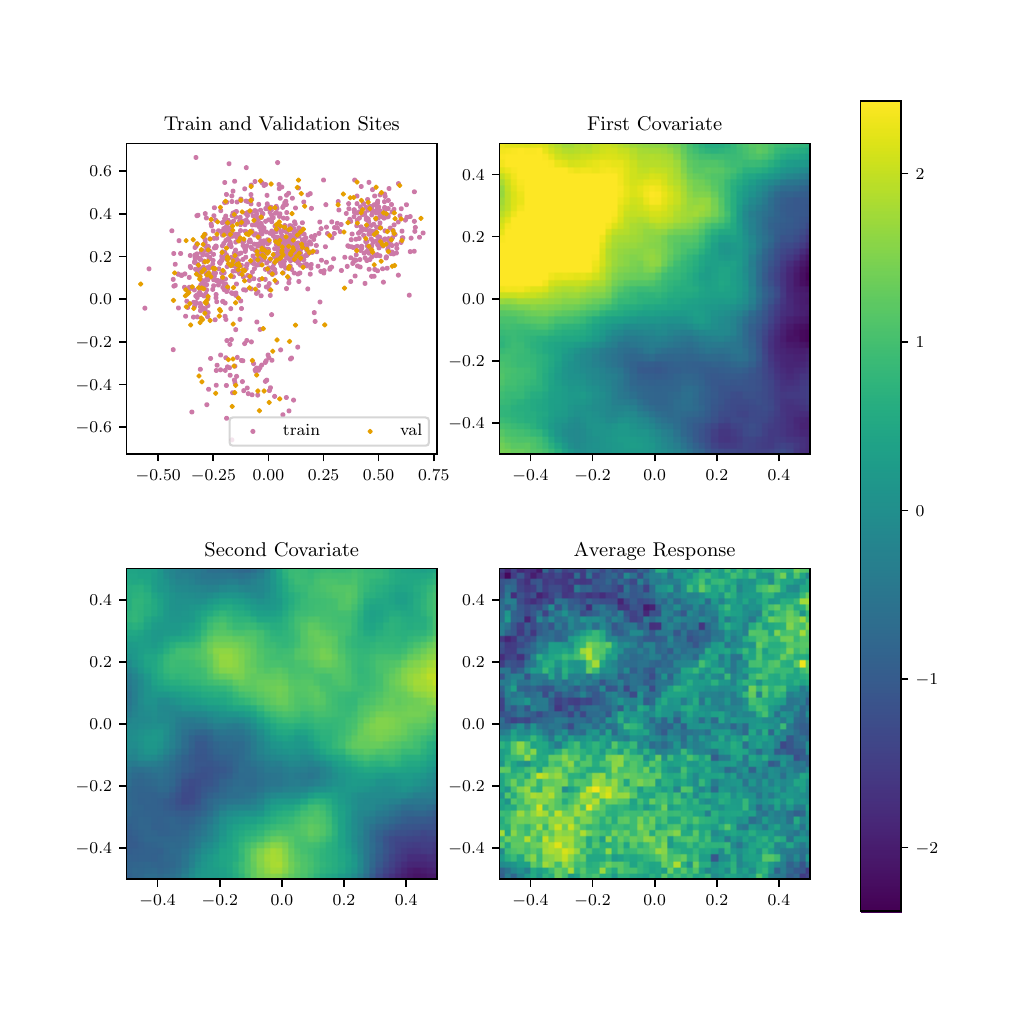}
    \caption{Data for the grid prediction task. An example of a single sample of the training and validation points (with 1000 training points and 250 validation points) is shown in the top left. The top right and bottom left show the covariates as a function of space, while the bottom right shows the mean of the response variable as a function of space.}
    \label{fig:grid-prediction-data-0}
\end{figure}

\paragraph{Point Prediction Task} \emph{Generation of training and validation sites:}
The training and validation points are sampled independently and identically from a uniform distribution supported on $[-0.5, 0.5]^2$. $1000$ training points are used in all experiments. The number of validation points is varied in $\{250 \times 2^{\ell}\}_{\ell=0}^{5}$.

\emph{Generation of test site} The test site is fixed to be the origin. We generate 45000 response values at the origin so that when we compute the empirical risk we expect it to accurately reflect that actual risk.

\emph{Generation of covariates:}
The covariates are generated as a zero-mean
Gaussian process with an isotropic squared exponential covariance function with lengthscale $0.3$ and scale parameter $1$. That is, the covariance function is,
\begin{align}
    k_{\spatialfield}(S,S') = \exp\left(-\frac{\|S-S'\|_2^2}{2\cdot 0.3^2}\right).
\end{align}
A small diagonal term (1e-12) is added to the diagonal of the covariance matrix to avoid numerical linear algebra errors. Sampling is performed using tensorflow probability \citep{dillon2017tensorflow}. We generate two covariates via this process $X=(X^{(1)}, X^{(2)})$.

\emph{Generation of response}
Once the sites and covariates have been generated, the response variable is sampled from a Gaussian process with zero mean. The covariance function of the Gaussian is a sum of two, $2$ dimensional isotropic squared exponential kernels:
\begin{align}
    k(S,S') =  0.5\exp\left(-\frac{\|S-S'\|^2_2}{2 \cdot 0.5^2}\right) + \exp\left(-\frac{\|\spatialfield(S)-\spatialfield(S')\|_2}{2}\right). \label{eqn:response-generation-kernel-2}
\end{align}
Independent, identically distributed Gaussian noise is added to the function values with variance $0.1$.

\subsubsection{Loss Function}
We use truncated, squared loss,
\begin{align}
    \ell(a,b) = \min(1.0,(a-b)^2),
\end{align}
which is bounded by $1.0$. The empirical risk is calculated as in \cref{eqn:empirical-risk}.

\subsubsection{Model Fitting}\label{app:risk-estimation-model}
We fit a Gaussian process regression model to using only the first covariate $\spatialfield^{(1)}$ to the training data. The prior is taken to be the same as the data generating process, but with (only) $\spatialfield^{(1)}$ in place of $(\spatialfield^{(1)}, \spatialfield^{(2)})$ in the second kernel in \cref{eqn:response-generation-kernel} and \cref{eqn:response-generation-kernel-2} for the two datasets respectively. The mean of the posterior process is used for predictions, and is calculated using GPFlow \citep{GPflow2017}.

\subsubsection{Implementation of Risk Estimation}\label{app:implementation-validation}
The \holdout{} is implemented by taking an (unweighted) average of the loss on each validation point. Both nearest neighbor methods are implemented using \texttt{scikit-learn} \citep{scikit-learn} with $kd$-trees and Euclidean distance. For $\kstar$ nearest neighbors, nearest neighbors is performed for all $k$ that are powers of $2$ less than $\nval$, and the value of $k$ with the smallest bound is used for risk estimation. This is done with $\delta=r$, with $r$ calculated as in \cref{app:computation-approx-fill}, and $\Delta=1$. \Cref{tab:grid-k-chosen} and \cref{tab:point-k-chosen} show the values of $k$ chosen for grid and point prediction respectively. For the grid prediction task, $\kstar$ tends to be small, even as the size of the validation set becomes large. This supported by \cref{cor:map-fill-bdd}, since even $1$-nearest neighbor reliably estimates risk in this setting. For the point prediction task, $\kstar$ grows with $\nval$.

\begin{table}
    \centering
    \caption{Value of $\kstar$ chosen by minimizing the bound in the grid prediction task. Because the test data is well separated, the variance of the estimator is small even when $k$ is small. As a result the value of $k$ that minimizes the upper bound is generally small, even as the number of validation points increases.}
    \begin{tabular}{c cccccc} 
  \multicolumn{6}{c@{}}{Number of Validation Points} \\ \cmidrule(l){2-7}  $\mathbf{\kstar}$  & 250 & 500 & 1000 & 2000 & 4000 & 8000 \\ \hline 
1 &82 & 85 & 81 & 82 & 66 & 50 \\ 
2 &15 & 15 & 19 & 18 & 34 & 41 \\ 
4 &3 & 0 & 0 & 0 & 0 & 9 \\ 
\end{tabular}
    \label{tab:grid-k-chosen}
\end{table}
\begin{table}
    \centering
    \caption{Value of $\kstar$ chosen by minimizing the bound in the point prediction task. In this task, there is generally a bias-variance trade-off that must be balanced. As a result the value of $k$ that minimizes the upper bound increases as the amount of available validation data increases.}
    \begin{tabular}{c cccccc} 
  \multicolumn{6}{c@{}}{Number of Validation Points} \\ \cmidrule(l){2-7}  $\mathbf{\kstar}$  & 250 & 500 & 1000 & 2000 & 4000 & 8000 \\ \hline 
16 &7 & 0 & 0 & 0 & 0 & 0 \\ 
32 &93 & 33 & 0 & 0 & 0 & 0 \\ 
64 &0 & 67 & 99 & 14 & 0 & 0 \\ 
128 &0 & 0 & 1 & 86 & 98 & 0 \\ 
256 &0 & 0 & 0 & 0 & 2 & 100 \\ 
\end{tabular}
    \label{tab:point-k-chosen}
\end{table}
\subsubsection{Monte Carlo Estimation of Test Risk}\label{app:mc-estimate-test-risk-synthetic}

Following the argument in \cref{app:mc-estimate-test-risk}, we use the empirical test risk \cref{eqn:empirical-risk} in place of the test risk as ground truth in synthetic experiments. The assumption that $(Z_m)_{m=1}^{\ntest}$ are independent holds by the description of the data generating process, because the $(Y_m)_{m=1}^{\ntest}$ are conditionally independent and $Z_m$ is a function of $Y_m$. The assumption that $Z_m$ is almost surely bounded holds with $\Delta=1$ by our choice of truncated squared loss. Further, in both synthetic experiments, we take $\ntest = 45000$.
Combining these gives that with probability at least $0.95$
\begin{align}
    |\hat{R}_{\Qtest}(h) - R_{\Qtest}(h)| \leq \sqrt{\frac{1}{2\times 45000}\log \frac{2}{0.05}} \leq 0.0065. \label{eqn:test-risk-hoeffding-synthetic}
\end{align}

\cref{fig:consistency-grid-plot} shows the absolute difference between each estimator and the empirical test risk across 100 seeds. We see that the difference between the estimators is generally larger than twice \cref{eqn:test-risk-hoeffding-synthetic}, and so by the argument in \cref{app:mc-estimate-test-risk}, we expect our estimate of ground truth to be accurate enough that the difference in performance of the methods is not simply due to error in estimating the ground truth.

\subsubsection{Metrics Reported and Additional Figures}\label{app:risk-estimation-metric}

\Cref{fig:grid-rel-risk,fig:point-rel-risk} show the relative errors of each estimation, calculated as
\begin{align}
    \frac{\empconditionalrisk_{\Qtest}(h) - \hat{R}(h)}{\empconditionalrisk_{\Qtest}(h)}.
\end{align}
From this, we can see that the \holdout{} method has a bias in both cases that does not appear to go away as the number of validation points increases. In contrast, the $1$-nearest neighbor method primarily suffers due to a variance issue when it fails to converge. We again see in both instances the $\kstar$-nearest neighbor approach appears to concentrate around zero error as the number of validation points increases.
\begin{figure}
    \centering
    \includegraphics[width=\textwidth]{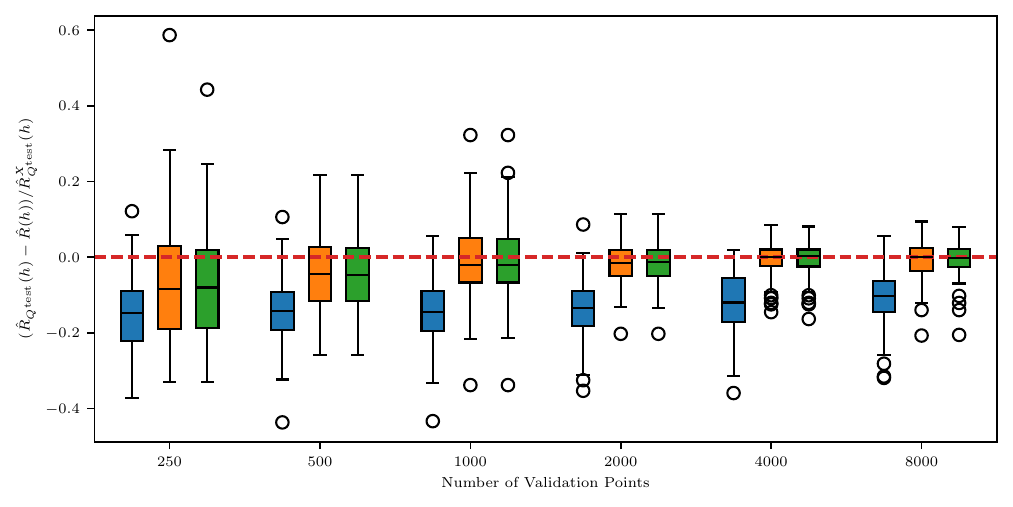}
    \caption{The relative error in estimating the empirical risk for each method is plotted against the number of validation points used for the grid prediction task. The \holdout{} is biased, even for large $\nval$. The $1$-nearest neighbor and $\kstar$-nearest neighbor estimates both have small relative error for large $\nval$.}
    \label{fig:grid-rel-risk}
\end{figure}

\begin{figure}
    \centering
    \includegraphics[width=\textwidth]{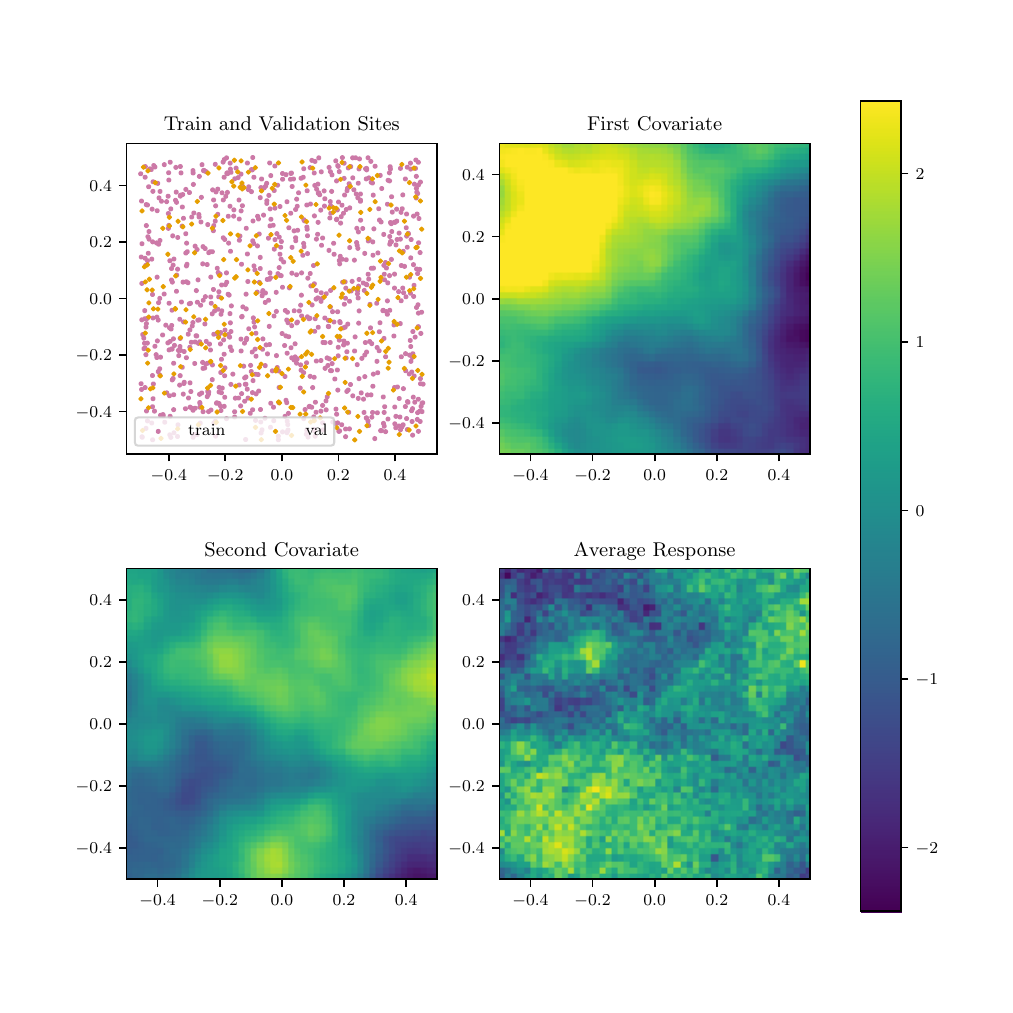}
    \caption{Data for the point prediction task. An example of a single sample of the training and validation points (with 1000 training points and 250 validation points) is shown in the top left. The top right and bottom left show the covariates as a function of space, while the bottom right shows the mean of the response variable as a function of space.}
    \label{fig:point-prediction-data-0}
\end{figure}

\begin{figure}
    \centering
    \includegraphics[width=\textwidth]{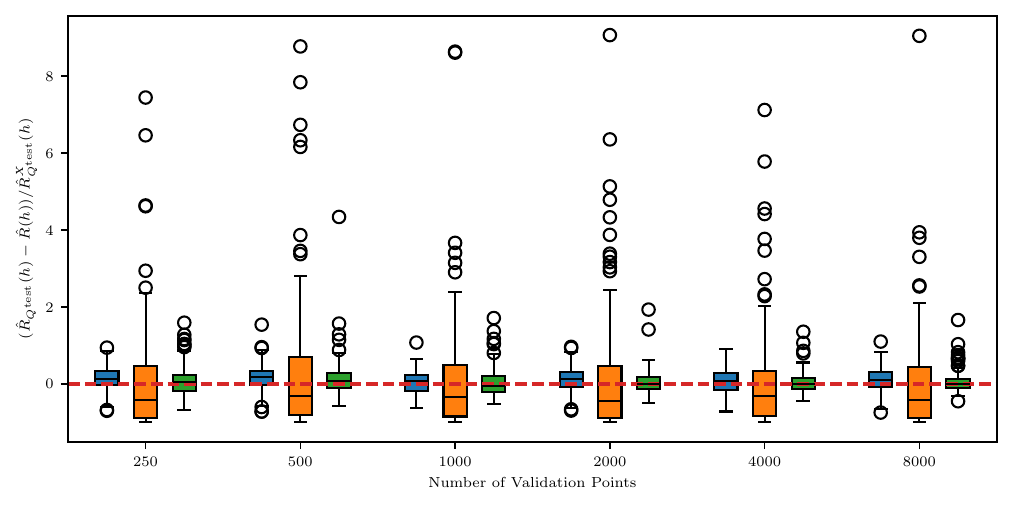}
    \caption{The relative error in estimating the empirical risk for each method is plotted against the number of validation points used for the point prediction task. The \holdout{} has a small but non-negligible bias. The $1$-nearest neighbor method has a small bias, but large variance. The nearest neighbor procedure with $\kstar$ neighbors has small relative error for large $\nval$.}
    \label{fig:point-rel-risk}
\end{figure}

\subsection{Air Temperature Tasks}\label{app:airTemp}

We now provide additional details about data source, pre-processing, model fitting and risk estimation for the air temperature dataset. These are identical between the real response experiment and the partially synthetic experiment, except for the bootstrapping procedure described in \cref{app:bootstrapping-details}. We also provide additional experimental results on a grid prediction task for both the bootstrapped and original datasets.

\subsubsection{Data Sources}\label{app:airTemp-datasources}

The land surface temperature used is from MODIS Aqua \citep{modisLST} and is monthly average land surface temperature on a $0.05$ degree grid. We download monthly average weather station data from the Global Historical Climatology Network \citep{ghcnm}. Latitude and longitude of major United States (US) urban areas are from the 2023 US census gazetteer \citep{census}. All of these datasets are produced in large part by US government agencies (NASA, NOAA and the Census). While we could not find specific license information, we understand these datasets to be public domain following section 105 of the Copyright Act of 1976.

\subsubsection{Data Pre-processing}
We assign the land surface temperature at the nearest point (using a spherical approximation to distance between points) to each weather station. The nearest point is found using the \texttt{scikit-learn} implementation of nearest neighbor algorithm using the `ball-tree' \citep{Omohundro2009FiveBC} and `Haversine' metric. Temperatures are converted to Celsius from Kelvin. We remove all rows where the Land Surface temperature is not available. We use the weather station data uploaded to GHCNM as of January 15, 2024. We filter out weather stations outside of the United States (based on the station ID). We also filter out stations with a non-empty quality control flag or no temperature recorded for January 2018 (the month we consider). Finally, we remove stations in Hawaii or Alaska to focus on the continental United States. In total, after this processing, there are 6422 weather stations. We use 70\% of the stations for fitting the models, and \holdout{} the remaining 30\% estimates. When building the test sites, we remove points outside the United States based on a reverse geocoding lookup with \citet{Thampi2015reversegeo} to the nearest city. This does not create an exact boundary (since it is based on the nearest city or town and not the country in which the latitude and longitude is based in) but is a good proxy for whether or not a point is in the United States. \Cref{fig:weather-station-sites} of the available weather stations for model fitting and validation, colored by monthly average temperature in January 2018.
\begin{figure}
    \centering
    \includegraphics[width=0.8\textwidth]{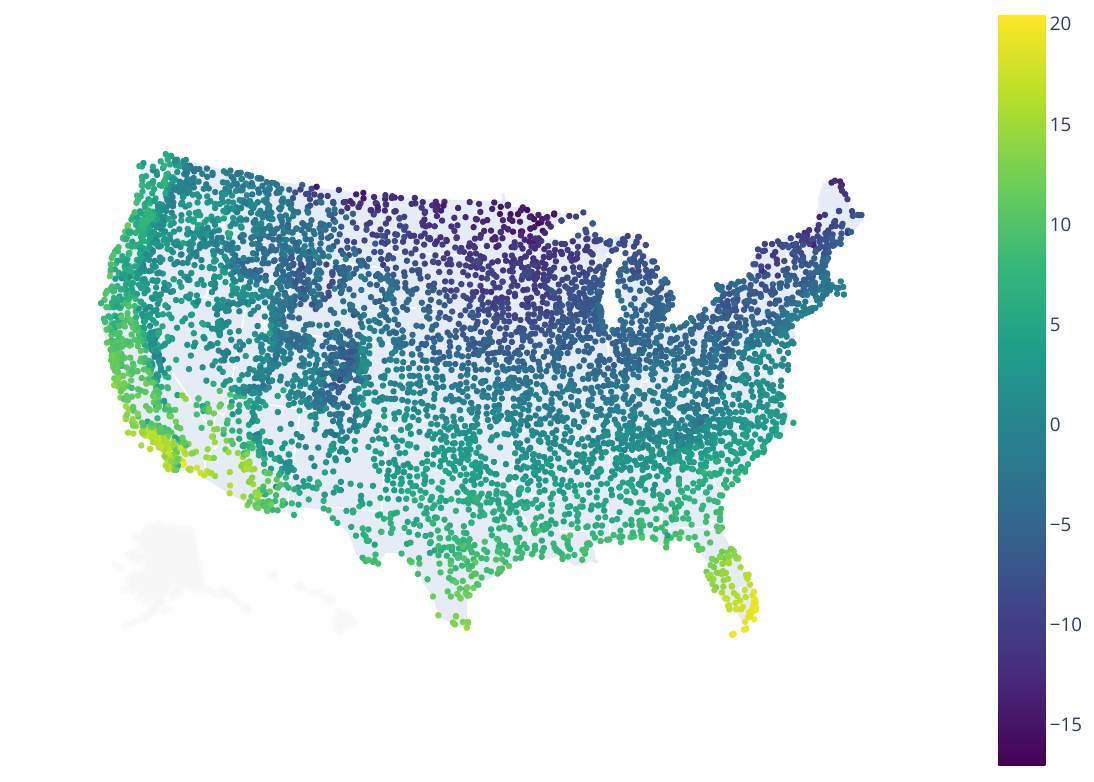}
    \caption{Weather stations used in the air temperature prediction task we considered, colored by average temperature in January 2018 in degrees Celsius.}
    \label{fig:weather-station-sites}
\end{figure}

\subsubsection{Loss Function}
We consider a truncated absolute value as the loss function,
$
    \ell(a,b) = \min(5.0, |a-b|).
$
This means we are primarily interested in the quality of the model predictions when it is relatively close to the actual response, and do not consider differences in predictions that are, for example, 8 degrees Celsius versus 10 degrees Celsius wrong meaningfully different. While the choice of 5 degrees is arbitrary, this is motivated by applications in which we might have some allowable tolerance for the quality of a prediction beyond which the prediction is no longer useful (and so it doesn't matter how bad it is).

\subsubsection{Model Fitting}
Inspired by \citet{hooker2018global}, we fit a geographically weighted least squares regression model using the land surface temperature at day and night. In particular, we fit an affine model, with the coefficients, $\beta(S)$ depending on the location that will be predicted at. $\beta(S)$ is selected by  solving the weighted least squares problem,
\begin{align}
    \hat{\beta}(S) \in \arg \min_{(b_0, b_1, b_2) \in \RR^{3}} \sum_{i=1}^{n^{\text{train}}} w_i(S) (Y^{\text{train}}_i - (b_0 + b_1 X^{\text{train}, 1}_i + b_2X^{\text{train},2}_i))^2,
\end{align}
with $Y^{\text{train}}_i$ the temperature at station $i$, $X^{\text{train}, 1}_i$ the daytime land surface temperature $X^{\text{train}, 2}_i$ the nighttime land surface temperature both at the closest satellite point to station $i$ and $w_i(S) = \exp(-\frac{d_{\text{haversine}}(S, S^{\text{train}}_i)^2}{2 \ell^2})$ and $d_{\text{haversine}}$ the Haversine (great circle) distance between the points. $\ell \geq 0$ is a parameter, and we select it from $\{25.0, 50.0, 75.0, 100.0, 150.0, 200.0, 300.0, 400.0, 500.0, 750.0, 1000.0\}\mathrm{km}$ via leave-one-out cross-validation on the training data with mean squared error. We perform leave-one-out cross-validation (without additional weighting).

We also consider a simple baseline model fit using only the weather station data. We fit a Gaussian process with zero prior mean and Mat\'ern 3/2 kernel to the weather stations with covariate the spatial locations in latitude and longitude converted to radians, and a Gaussian likelihood model. We fit the parameters of the kernel using L-BFGS to attempt to maximize the marginal likelihood of the parameters. The parameters fit are two lengthscale parameters (one for each spatial dimension), a kernel scale parameter, and a likelihood variance parameter. The mean is removed from the training data prior to fitting, the kernel lengthscales are set to to standard deviation of each covariate and the kernel variance parameter is set to equal the variance of the training response data, and the likelihood variance parameter is set to equal $0.1$-times the variance of the training response variable. A maximum of $15$ iterations of L-BFGS are run.

\subsubsection{Risk Estimation Details}\label{app:airtemp-risk-estimation}
The \holdout{} is implemented as in previous experiments (\cref{app:implementation-validation}). We estimate the standard error of the \holdout{} empirically  by computing the sample standard deviation of the sum of the losses,
\begin{align}
    \hat{\sigma}^2 = \left(\frac{1}{\nval(\nval-1)} \sum_{j=1}^{\nval} \ell(\Yval_j, h^{\spatialfield}(\Sval_j))^2\right)^{1/2}.
\end{align}
\Cref{table:airtemp-metro} reports the \holdout{} estimate $\pm$ two standard deviation.

The nearest neighbor methods are implemented using the \texttt{scikit-learn} implementation with Haversine distance and the ball-tree algorithm. $\kstar$ is selected with $\delta=0.1$ and $\Delta=5^\circ$C and a Lipschitz constant of $1^\circ$C/100 km. We use a fixed $\delta$ since we only discuss a method applicable to estimating fill distance on the unit cube, not on a subset of the sphere.

\subsubsection{Bootstrapping of Residuals}\label{app:bootstrapping-details}
In order to generate many datasets with a realistic synthetic response variable where we have access to ground truth we:
\begin{enumerate}
    \item Fit a Gaussian process regression model to \emph{all} the available weather station data. We use a Mat\'ern 3/2 kernel with zero prior mean on the weather station data with the (spatial) mean temperature removed. Parameters of the kernel are selected via maximum likelihood.
    \item Compute the empirical distribution of the residuals of the mean of these predictions.
    \item For each seed we then use the same spatial locations and covariates, and generate the response surface at any point in space by computing the mean of the Gaussian process regression model fit and adding a sample from the empirical distribution of the residuals of the actual data.
\end{enumerate}

We can then directly estimate the test risk via generating many $\Ytest$ at each spatial location (we use 1000 realizations for each city in the 5-metros task) and 1 for each grid point in the grid task (since the error is averaged over test sites this still results in an estimator that is concentrated) in this manner and forming a Monte Carlo estimate as in \cref{app:mc-estimate-test-risk}.

\subsubsection{Estimation of Ground Truth in Bootstrapped Experiment}\label{sec:estimation-ground-truth-airtemp}

Following the argument in \cref{app:mc-estimate-test-risk}, we use the empirical test risk \cref{eqn:empirical-risk} in place of the test risk as ground truth in synthetic experiments. For the assumption that $(Z_m)_{m=1}^{\ntest}$ are independent to hold it is sufficient for $\Ytest_m$ to be independent, conditioned on the spatial location at which it is observed. This holds based on the data generating process used to construct the synthetic responses: since $\Ytest_m$ is a noisy observation of the smooth function we fit to the weather stations, plus noise sampled independently from the distribution of residuals. Because the loss is truncated MAE, the $Z_m$ are surely bounded by $5$. We use $10000$ samples at each of the 5 test location in estimating the risk. Using these numbers in \cref{eqn:test-risk-hoeffding}, we arrive at
\begin{align}
    |\hat{R}_{\Qtest}(h) - R_{\Qtest}(h)| \leq 5 \sqrt{\frac{1}{2\times 50000}\log \frac{2}{0.05}} < 0.031. \label{eqn:test-risk-hoeffding-airtemp}
\end{align}

Following the argument in \cref{app:mc-estimate-test-risk}, we expect our estimate of ground truth to be good enough to distinguish between the quality of models whose absolute error from the estimated ground truth differs by more than $2 \times 0.031 = 0.062$. \Cref{fig:bootstrapped-airtemp-synthetic} shows these absolute errors. We see that for many of the seeds, the difference between the performance of 1NN (orange) and SNN and the \holdout{} is greater than $0.174$. Given that the earlier argument is quite conservative (in the sense that Hoeffding's inequality is likely to be loose), we therefore can attribute the difference in performance of the methods to indicate that SNN and the \holdout{} are giving better estimates of the ground truth test risk, and the observed difference is not due to error in our estimation of the test risk.

\begin{figure}
    \centering
    \includegraphics[width=0.25\textwidth]{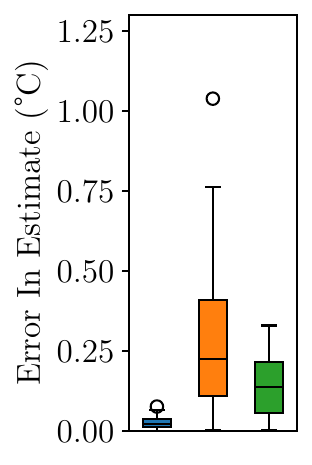}
    \includegraphics[width=0.25\textwidth]{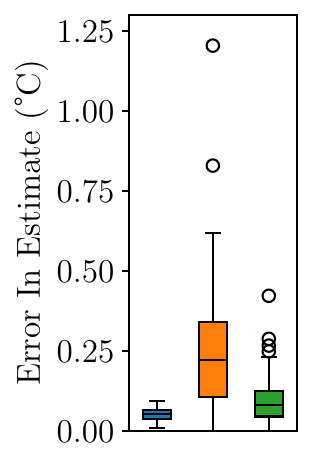}
    \caption{Absolute error in estimating the truncated mean absolute error in the air temperature dataset with bootstrapping.}
    \label{fig:bootstrapped-airtemp-synthetic}
\end{figure}

For the grid prediction task the assumptions are similar, but $\ntest = 341{,}628$ as this is the number of points on the map. Because the assumptions are satisfied by construction of the synthetic data (as in the 5-metro prediction task), with probability $1-\delta$
\begin{align}
    |\hat{R}_{\Qtest}(h) - R_{\Qtest}(h)| \leq 5 \sqrt{\frac{1}{2\times 341628}\log \frac{2}{0.05}} < 0.012. \label{eqn:test-risk-hoeffding-airtemp-grid}
\end{align}
We therefore expect our estimate of ground truth to be very accurate, although we see in \cref{fig:airtemp-bootstrapped-grid-prediction} that the methods all perform well in estimating the test risk on this task, and so there is likely not a meaningful difference in which approach is used to perform validation.

\subsubsection{Results for Grid Prediction with Bootstrapped Data}

\cref{fig:airtemp-bootstrapped-grid-prediction} shows the results for \holdout{}, $1$-nearest neighbor and SNN with the test set each grid point in the map that is located in the continental United States. All 3 methods lead to reasonably accurate estimates of the mean absolute error on this prediction task (within 0.1 degrees of the ground truth error). Based on our theory, we generally expect 1NN and SNN to have small error in grid prediction tasks (at least with sufficient data and the infill assumption being satisfied), while for the \holdout{} it depends on the particular predictive method and distribution of test and validation sites. In this case, it appears for both prediction methods the bias introduced by the use of the \holdout{} is relatively small.

\begin{figure}
    \centering
    \includegraphics[width=0.25\textwidth]{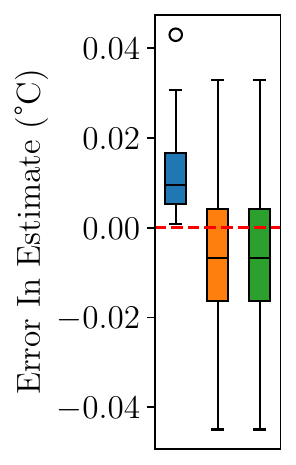}
    \includegraphics[width=0.25\textwidth]{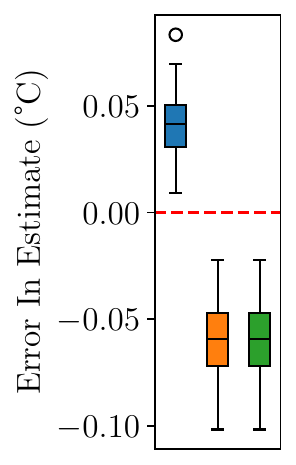}
    \caption{Error in estimate of mean absolute error of the Gaussian process regression predictions (left) and the geographically weighted regression predictions (right) for the \holdout{} (blue), 1NN (oragne) and SNN (green) on a grid prediction task. We see that all 3 methods result in accurate estimates (within 0.1 degrees Celsius) of the mean absolute error on this task for both prediction methods, and suspect it is unlikely meaningfully different conclusions would be drawn from use of any of the methods in this application.}
    \label{fig:airtemp-bootstrapped-grid-prediction}
\end{figure}

\subsubsection{Results for Grid Prediction with Real Data}
\Cref{tab:airtemp-geo} shows the results for \holdout{}, $1$-nearest neighbor and SNN with the test set each grid point in the map that is located in the continental United States. We see good agreement between all three method. This is expected for $1$-nearest neighbor and SNN based on earlier theory (\cref{app:map-prediction-proofs}).

\begin{table}[t]
    \caption{Estimates of risk given by each method. All three methods agree reasonably well (within $\pm 2$ standard deviation of the estimate given by \holdout{} for both geographically weighted regression and spatial regression in this task. In particular, all three methods suggest that the geographically weighted regression method has lower risk on this task.}\label{tab:airtemp-geo}
    \centering

\begin{tabular}{c|cc}
 & GWR & Spatial GP \\ \hline
Holdout & $\mathbf{0.83 \pm 0.03}$ & $0.90 \pm 0.04$ \\ 
1NN & $\mathbf{0.80}$ & $0.88$ \\ 
SNN & $\mathbf{0.80}$ & $0.88$ 
\end{tabular}
\end{table}

\subsection{UK Housing Experiment}\label{app:UK-housing}

We provide additional details for the UK flat price prediction task presented in \cref{sec:uk-ppd}.

\subsubsection{Data Sources and Pre-processing}\label{app:uk-data}
We download 2023 price paid data for England and Wales from \citet{ppduk}. This data is subject to a UK Open Government License (\url{https://www.nationalarchives.gov.uk/doc/open-government-licence/version/3/}), which requires citation of the data, but allows both commercial and non-commercial uses. These records contain postal codes for each property sold, the type of property sold, town or city, price paid for the property. We use the type of property variable to filter out all properties that are not flats, and only consider additions (not replacements or deletions) to the dataset and ``standard'' price paid data (not repossessions, buy-to-lets or other sales labelled as non-standard). Noting that a postal code in the UK corresponds to a very small geographic area, we obtain latitude and longitude data for each sale by looking up the postal code coordinates using the National Statistics Postcode Lookup \citep{nspl}, which we understand to be a product of the UK Census and therefore also subject to an Open Government License. We convert from northing and easting to latitude and longitude using $\texttt{R}$. We log transform the price variable prior to model fitting as we expect price paid to be non-negative and highly skewed, and so a Gaussian (process) prior would otherwise be almost certainly inappropriate.

\subsubsection{Model Fitting}\label{app:uk-model-fitting}

We fit hyperparameters of the variational Gaussian process regression by evidence lower bound maximization.
\paragraph{Model Specification}
We fit a Gaussian process regression model with prior covariance specified by a sum of two Mat\'ern 3/2 kernels and a zero prior mean on the mean centered log price paid data. We use a sum of Mat\'ern 3/2 kernel in place of the sum of RBF kernel used in \citet{hensman2013gaussian} as we expect there to be places where (log) property prices vary quickly in space, and so the smoothness properties implicitly assumed in using an RBF kernel may be inappropriate. We use 2000 inducing points for the variational approximation. The locations of these points are optimized jointly with model parameters when maximizing the evidence lower bound. We use the closed form for the optimal variational posterior (given a set of inducing points) provided in \citet{titsias2009variational}, and perform maximization of the evidence lower bound using L-BFGS.

\paragraph{Initialization}

The locations of the inducing points are initialized by the greedy procedure suggested in \citet{burt2023convergence}, which is essentially equivalent to a partially pivoted Cholesky decomposition recommended earlier in the Gaussian process approximation literature \citep{foster2009stable}. The initial prior variance of both kernels is set to be equal to the variance in the training data; the lengthscales of one kernel (intended to model regional price trends) are initialized to twice the standard deviation in the location data (in radians), while the scale of the other kernel (intended to model local trends) is initialized to half the standard deviation in the location data. The likelihood standard deviation is initialized to be 0.1$\times$ the standard deviation in the log price paid in the training data.

\subsubsection{Estimation of Ground Truth}\label{app:uk-estimation-truth}

We next describe why we might expect empirical test risk to provide a reasonable estimate of ground truth in this problem, and in particular we discuss the assumptions that justify the use of empirical test risk in relation to \cref{app:mc-estimate-test-risk}. We consider a truncated loss $\ell(a,b) = \min(0, 10^6)$, and so $Z_m$ is almost surely bounded by $10^6$. The assumption that the $(Z_m)_{m=1}^{\ntest}$ are independent would be implied by independence of $(Y_m)_{m=1}^{\ntest}$ that is: that given the location at which a flat is sold, any remaining randomness in the observed sale prices is independent. Concretely, we might think of the randomness in sales price, $\epsilon^{\text{test}}_m$ in our model, as coming from aspects of the sale process of the house, such as who happens to see the advertisement for a house, and we \emph{assume} these are independent for each house when constructing our estimate of the ground truth.

If the conditional independence assumption proposed above holds, then following \cref{app:mc-estimate-test-risk} we have that
\begin{align}
    |\hat{R}_{\Qtest}(h) - R_{\Qtest}(h)| \leq \pounds 1000000 \sqrt{\frac{1}{2\times 1000}\log \frac{2}{0.05}} < \pounds 43000. \label{eqn:test-risk-hoeffding-ukhousing}
\end{align}

Therefore, under this assumption, we might would expect our estimate of ground truth to be at least good enough to tell the difference (in the sense of which is closer to the actual ground truth) between predictors that differ in absolute error from the ground truth estimate by more than $2 \times \pounds43,000 = \pounds86,000$. \Cref{fig:uk-absolute-error} shows the absolute error of the three methods. We see that the error in the estimate provided by the holdout is on the order of $\pounds150,000$, while the error in 1NN and SNN are closer to $\pounds25000$ in most seeds. Given this large difference and earlier discussion, we do not expect that this error arises from difficulties in estimating the ground truth, but instead arises from actual differences in the qualities of the estimator.

\begin{figure}
    \centering
    \includegraphics[width=0.25\textwidth]{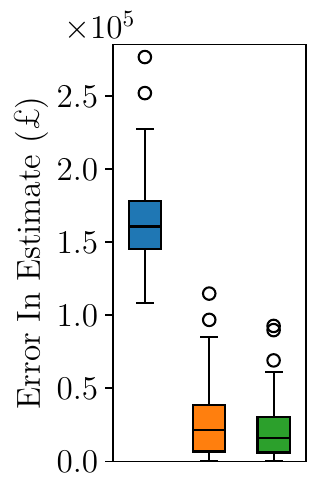}
    \caption{Absolute error of estimates (relative to Monte Carlo estimate of ground truth truncated mean absolute error) for the holdout (blue), 1NN (green) and our SNN (orange). We see that the holdout has significantly higher error in estimating the test risk in this task, which is caused by bias in the estimate provided.}
    \label{fig:uk-absolute-error}
\end{figure}

\subsubsection{Risk Estimation Details}\label{app:uk-estimation}
Holdout, 1NN and SNN are run as in the air temperature experiments (\cref{app:airtemp-risk-estimation}). In particular, we use a failure probability of $\delta=0.1$ for SNN and nearest neighbor calculations are done with respect to Haversine distance to account for Earth's curvature. We use a fixed $\delta$ since we only discuss a method applicable to estimating fill distance on the unit cube, not on a subset of the sphere. $\Delta=\pounds 1,000,000$ is used in selecting the number of neighbors as this is an upper bound on the truncated loss. We use a Lipschitz constant of £1,000/km as £1/km seems implausibly small.

\subsection{Wind Speed Prediction Experiment}\label{app:windSpeed}
In this section, we provide additional details for the wind speed prediction experiment discussed in \cref{sec:wind_speed} of the main text.

\subsubsection{Data Sources and Pre-processing}\label{app:windSpeed-data}

We download daily wind speed readings from weather stations from the Global Historical Climate Network \citep{ghcnd}. As this dataset was constructed by NOAA employees, we understand it to be public domain following section 105 of the Copyright Act of 1976. We filter out weather stations outside the continental US, as well as any weather stations that do not contain daily average wind speed readings. We look only at wind speed data from January in the prediction task, and years 2000--2024. There is a weather station at Chicago O'Hare which we remove from the training and validation data and use as the test set.

For each replicate used to form \cref{fig:airtemp-bootstrapped-grid-prediction} we split off a training set containing 80\% of weather stations, and a validation set containing the remaining 20\%. The number of observations in the training and validation set varies (because different weather stations may be online for a different number of days in January in previous years), but this leads to on the order of 580000 training observation and 126000 validation observations. Each training and validation point is a triple containing latitude, longitude and average wind speed. We perturb the latitude and longitude (in degrees) of validation points by a Gaussian random variable with standard deviation $10^{-12}$, which is essentially equivalent to using random tie-breaking in the nearest neighbor algorithms. We expect this has a significant impact on 1NN compared to the version discussed in the paper, because there may be many observations from the nearest weather station to Chicago O'Hare. While this would unlikely be done in practice, random tie-breaking, or tie-breaking by selecting the first nearest neighbor according to some other ordering are common and would lead to similar outcomes as the results presented here (but higher variance than averaging over all neighbors that are equally close). The latitude and longitude are then converted to radians for the analysis.

\subsubsection{Loss Function}
We use truncated mean squared error as the loss function, $\ell(a,b) = \min(25, (a-b)^2)$.

\subsubsection{Estimation of Ground Truth}\label{app:estimation-ground-truth-windspeed}

Following the argument in \cref{app:mc-estimate-test-risk}, we use the empirical test risk \cref{eqn:empirical-risk} in place of the test risk as ground truth in synthetic experiments. The assumption that $(Z_m)_{m=1}^{\ntest}$ are independent. It is sufficient for $\Ytest_m$ (the wind daily wind speeds) to be independent, conditioned on the spatial location at which it is observed. This is likely not the case, as we would expect average wind speed in consecutive days exhibit at least some dependence. However, if the wind speed decorrelates reasonably rapidly over time, we would expect similar arguments to hold, possibly with fewer effective samples.

Because the loss is truncated mean squared error, the $Z_m$ are surely bounded by $25$m$^2$/s$^2$. We use $775$ samples in estimating the risk. Using these numbers, and under the assumption that wind speed at a location is independent of the wind speed on previous days, in \cref{eqn:test-risk-hoeffding}, we arrive at
\begin{align}
    |\hat{R}_{\Qtest}(h) - R_{\Qtest}(h)| \leq 25\text{m}^2/\text{s}^2 \sqrt{\frac{1}{2\times 775}\log \frac{2}{0.05}} < 1.22 \text{m}^2/\text{s}^2. \label{eqn:test-risk-hoeffding-windspeed}
\end{align}

\begin{figure}
    \centering
    \includegraphics{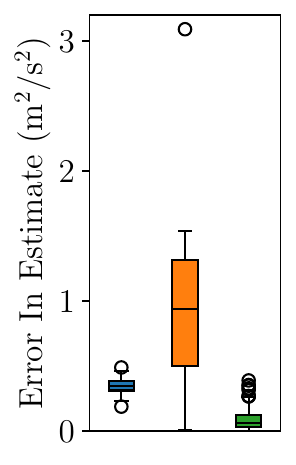}
    \caption{Absolute error in estimating (approximate) test risk in the wind speed experiment for the holdout (blue), 1NN (orange) and our SNN (green). It appears that SNN has the smallest error in estimating the ground truth, although the approximate ground truth we compute via Monte Carlo estimate is not entirely theoretically justified. However, due to the relatively large differences across most seeds, we still expect the difference is indicative of better performance of SNN.}
    \label{fig:windspeed-abs}
\end{figure}

Comparing to \cref{fig:windspeed-abs}, we see that this application of Hoeffding's inequality is not sufficient to justify that the estimate of ground truth is accurate enough to attributed the observed better performance of SNN to (actually) better estimation of the ground truth as opposed to inaccuracies of our Monte Carlo estimate of the test risk. However, we expect this is largely due to looseness is Hoeffding's inequality and, given that a large difference is observed in most seeds, it would be very surprising if this was only due to error in estimation of the ground truth which is independent across seeds.

\subsubsection{Model Fitting}

A gradient boosted machine is fit using LightGBM \citep{ke2017lightgbm} with default parameters expect for the number of leaves (set to 127) and the number of estimators (set to 100).

\subsubsection{Risk Estimation procedures}

Holdout, 1NN and SNN are run as in the air temperature experiments (\cref{app:airtemp-risk-estimation}). In particular, we use a failure probability of $\delta=0.1$ for SNN and nearest neighbor calculations are done with respect to Haversine distance to account for Earth's curvature. We use a fixed $\delta$ since we only discuss a method applicable to estimating fill distance on the unit cube, not on a subset of the sphere. A Lipschitz constant of $1(\text{m}^2/\text{s}^2) / \text{km}$ is used for selecting the number of neighbors.

\subsection{Model Selection on Synthetic Data}\label{app:model-selection-experiment}

\begin{figure}
    \centering
    \includegraphics[width=\textwidth]{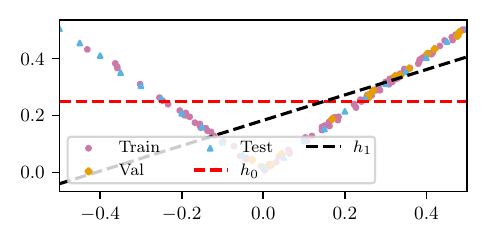}
    \caption{Training (pink circles), validation (orange diamonds), and test (blue triangles) data for a single seed of the model selection experiment. The dashed red line depicts predictive method  $h_0$, and dashed black shows $h_1$.}
    \label{fig:ms-data-plot}
\end{figure}
We next see that SNN and 1NN are able to select the model with lower test risk in a model selection task, but the \holdout{} systematically picks the wrong model.
We repeat the model selection problem 100 times. In each repetition, we have $N^{\text{train}}=100$
and a max $\nval=75$. In our analysis, we will consider validation subsets of size
$\nval \in \{5\ell\}_{\ell=1}^{15}$.
We generate independent test, validation, and training data as follows; see \cref{fig:ms-data-plot}.
\begin{align}
    U_i^{\smash{j}}           & \sim \mathcal{U}([-0.5,0.5]), \quad
    S_i^{\smash{j}}  = \sqrt{U_i^{\smash{j}} + 0.5}, \quad j \in \{\text{train}, \text{val}\}  \nonumber                                                         \\
    S_m^{\smash{\text{test}}} & = m/20 - 0.5, \quad 0 \leq m \leq 20, \quad \epsilon_{\smash{i}}^{\smash{j}}\sim \mathcal{U}([0,0.1]),      \nonumber            \\
    Y_{\smash{i}}^j           & = \vert S_{\smash{i}}^j \vert + \epsilon_{\smash{i}}^j \quad j \in \{\text{train}, \text{val}, \text{test}\}, \label{eqn:ms-dgp}
\end{align}

We compare two predictive methods: $h_0(S) = 0.25$ and $h_1(S) = \beta_1^\top S + \beta_0$, with $(\beta_1, \beta_0)$ fit by minimizing the mean absolute residual on the training data. \cref{fig:ms-data-plot} shows the data and predictions of both models (as functions of space). We use the loss function $\ell(a,b) = |a-b|$, which is bounded for this problem because both the hypotheses and the response variable are bounded on $[0,1]$.

Across all seeds, $h_0$ has the lower empirical test risk; $h_1$ makes large errors on the test points near $0$ because most of the training data is in $[0, 0.5]$. Since most of the validation data also clusters near 1, we expect the \holdout{} to struggle due to bias. Our arguments in \cref{app:model-selection} lead us to expect both SNN and 1NN should perform well on this task when given sufficient validation data.

We say an estimator of the risk, $\hat{R}$, selects $h_0$ if $\hat{R}(h_0) < \hat{R}(h_1)$. We plot the percentage of times each method (correctly) selects $h_0$ as a function of the number of validation points in \cref{fig:ms-results}. When the validation set is small, all estimators select the model with lowest test risk ($h_0$) less than half the time. For the nearest neighbor methods, we expect that when there are few or no spatial locations less than $1$, weighting cannot fix the estimate.
However, when the number of validation points is large, the nearest neighbor methods consistently (correctly) select $h_0$. By contrast, the \holdout{}  consistently (incorrectly) selects $h_1$, even though $h_1$ has higher test risk. See \cref{app:model-selection} for full experiment details.
\begin{figure}
    \centering\includegraphics[width=0.87\columnwidth]{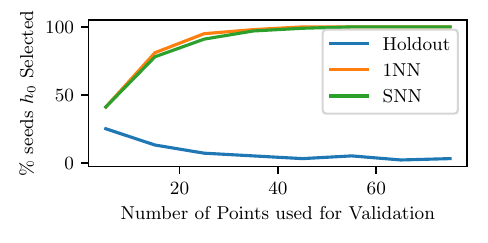}
    \caption{The percentage of times each estimator (correctly) selects the model with lower empirical test risk ($h_0$), out of 100 seeds as a function of $\nval$. Estimators include the \holdout{} (blue), 1NN (orange), and our SNN (green).}
    \label{fig:ms-results}
\end{figure}

\paragraph{Data Generation}
The data generation is fully described by \cref{eqn:ms-dgp}.

\paragraph{Model Fitting}
We consider two models. The first is a constant predictor that predicts $0.25$. The second is an affine model (a linear model with an intercept) fit by minimizing the mean absolute error from the line to the training points. This is fit using the \texttt{Scikit-learn} quantile regression with the (default) ``HiGHS'' solver \citep{Huangfu2015highs}.

\paragraph{Estimation of Risk}
The validation estimates used are calculated in the same as the synthetic experiments outlined previously in \cref{app:implementation-validation}. We use $\lossbdd = 1$ in the bound when selecting $\kstar$, even though the absolute value loss used can be larger than $1$. We don't expect this to have a significant impact on the results, as the upper bound we minimize is already misspecified in a similar way by not using the actual Lipschitz constant of the function. We again use $\delta = 0.1$ when selecting $\kstar$.

\end{document}